\documentclass[twoside,11pt]{article}

\usepackage[preprint]{jmlr2e}

\usepackage[T1]{fontenc}
\usepackage[utf8]{inputenc}
\usepackage{multirow}
\usepackage{tabularx}
\usepackage{bbding}
\usepackage{algorithm}
\usepackage{algorithmic}
\usepackage{amsmath,amssymb,amsfonts}
\usepackage{makecell}
\usepackage{colortbl}
\usepackage{subfig}
\usepackage{fancybox}
\usepackage{enumitem}
\usepackage{adjustbox}
\usepackage{threeparttable}
\usepackage{booktabs}
\usepackage{nicefrac}
\usepackage{microtype}
\usepackage{xcolor}
\hypersetup{hidelinks}

\definecolor{DownGreen}{RGB}{34,139,34}
\definecolor{UpOrange}{RGB}{210,105,30}
\definecolor{NeutralGray}{RGB}{120,120,120}
\definecolor{rowblue}{RGB}{235,242,250}
\definecolor{customgray}{gray}{0.9}

\newtheorem{assumption}[theorem]{Assumption}

\ShortHeadings{On the Convergence of Muon and Beyond}{Chang et al.}
\firstpageno{1}

\begin{document}

\title{On the Convergence of Muon and Beyond}

\author{\name Da Chang$^{1,3,4}$ \hfill \email{changda24@mails.ucas.ac.cn} \\
\name Yongxiang Liu$^{3}$ \hfill \email{liuyx@pcl.ac.cn} \\
\name Huihui Zhou$^{3}$ \hfill \email{zhouhh@pcl.ac.cn} \\
\name Ganzhao Yuan$^{2,*}$\hfill \email{yuanganzhao@foxmail.com} \\
\addr
$^1$Shenzhen Institute of Advanced Technology, Chinese Academy of Sciences, Shenzhen, China\\
$^2$Shenzhen University of Technology, Shenzhen, China\\
$^3$Pengcheng Laboratory, Shenzhen, China\\
$^4$University of Chinese Academy of Sciences, Beijing, China\\
$^*$ Corresponding author
}

\maketitle
\thispagestyle{plain}

\begin{abstract}
The Muon optimizer has demonstrated strong empirical performance in training neural networks with matrix-structured parameters. A central theoretical question is how momentum variance reduction interacts with practical orthogonalization and weight decay. We analyze two variants: the one-batch Muon-MVR1 and the two-batch Muon-MVR2. To the best of our knowledge, this is the first \textbf{horizon-free} convergence analysis of variance-reduced Muon that jointly covers constant mini-batches, finite practical fixed-coefficient Newton--Schulz iterations, and decoupled weight decay in unconstrained stochastic nonconvex optimization.
{Muon-MVR1 and Muon-MVR2 attain anytime stationarity rates of
$\widetilde{\mathcal{O}}(T^{-1/4})$ and
$\widetilde{\mathcal{O}}(T^{-1/3})$, respectively, over $T$ iterations.
Under mean-square smoothness, the Muon-MVR2 rate matches the
stochastic first-order lower bound up to logarithmic factors.
The analysis provides alignment and norm bounds uniform in the step count, together with a nuclear-norm alignment bound that improves with the number of steps.}
Under the Polyak--\L{}ojasiewicz (PL) condition and a uniform gradient bound along the iterates, Muon-MVR1 and Muon-MVR2 achieve last-iterate objective-gap rates of $\mathcal{O}(T^{-1/4})$ and $\mathcal{O}(T^{-1/3})$, respectively, with finite NS steps and decoupled weight decay under the prescribed horizon-free schedules.
Experiments on CIFAR-10 and C4 support the practical effectiveness of the proposed variance-reduced Muon variants. Code is available at the \href{https://github.com/MaeChd/MUON-MVR}{Muon-MVR repository}.
\end{abstract}

\begin{keywords}
Muon, stochastic nonconvex optimization, variance reduction,
Polyak--\L{}ojasiewicz condition, matrix orthogonalization
\end{keywords}

\section{Introduction}

The immense computational cost of pre-training Large Language Models (LLMs) has spurred a surge of research into novel optimization methods designed to enhance parameter efficiency and training stability \cite{hoffmann2022training,Liu2023SophiaAS,Chen2023SymbolicDO,vyas2025soap,pethick2025training,Yuan2024MARSUT}. Among these, methods based on matrix orthogonalization have recently garnered significant attention from both academia and industry \cite{jordan2024muon,Liu2025MuonIS}. In particular, Muon has shown promising empirical performance \cite{Liu2025MuonIS,An2025ASGOAS,liu2025cosmos,shah2025practical}. These advances motivate a convergence theory that accounts for momentum variance reduction, finite-step orthogonalization, and weight decay under learning-rate schedules that do not require a prescribed training horizon.

Existing analyses study standard and Nesterov momentum~\citep{Li2025ANO,shen2025convergence,nagashima2026improved}, spectral-norm-constrained formulations~\citep{chen2025muon}, and weight decay with explicit batch-size dependence~\citep{Sato2025ConvergenceBA}. Frank--Wolfe and non-Euclidean methods also attain variance-reduced rates~\citep{sfyraki2025lions,kovalev2025non,khirirat2025better}; \citet{qian2025muon} subsequently analyze decreasing-step-size MVR under layer-wise smoothness. Further work studies inexact orthogonalization~\citep{shulgin2025inexact,kim2026convergence,takezawa2025fedmuon,do2026muon}, spectral preconditioning~\citep{ma2026preconditioning}, and $\mu$P-compatible spectral control~\citep{zhao2026towards}. These developments leave a specific question: can variance-reduced Muon retain anytime stationarity rates with constant mini-batches when practical fixed-coefficient NS steps and decoupled weight decay are included together? We address this question for MVR1 and MVR2 in Algorithm~\ref{alg:muon}, using schedules independent of the terminal horizon. We also study last-iterate objective-gap guarantees for both practical variants under the PL condition.

\begin{table}[!htbp]
\centering
{
\begin{threeparttable}
\caption{Comparison of theoretical guarantees for Muon-type and related non-Euclidean methods.}
\label{tab:compare}
\setlength{\tabcolsep}{1.25pt}
\renewcommand{\arraystretch}{1.05}
\fontsize{6.4}{7.4}\selectfont
\begin{tabular}{lccccccc}
\toprule
     & Smooth$^a$ & \makecell{Stoc. Gradient \\Estimator$^b$} & \makecell{Batch\\Size} & \makecell{Nonconvex\\Rate}
     & \makecell{Weight\\Decay} & \makecell{Inexact\\Matrix-Sign$^e$}
     & \makecell{Last-Iterate\\Rate$^f$} \\
\midrule
\cite{Li2025ANO} & $L$ & MVR1($\gamma_t=0$) & $\mathcal{O}(1)$ & $\mathcal{O}(T^{-1/4})$ & \XSolidBrush & -- & \XSolidBrush \\
\cite{Sato2025ConvergenceBA} & $L$ & MVR1($\gamma_t=0$) & $\mathcal{O}(\epsilon^{-1})$ & $\mathcal{O}(T^{-1})+\mathcal{O}(1)^d$ & \Checkmark & -- & \XSolidBrush \\
\cite{Sato2025ConvergenceBA} & $L$ & MVR1 & $\mathcal{O}(\epsilon^{-1})$ & $\mathcal{O}(T^{-1})+\mathcal{O}(1)^d$ & \Checkmark & -- & \XSolidBrush \\
\cite{sfyraki2025lions} & $L_{+}$ & MVR2 & $\mathcal{O}(\epsilon^{-1})^c$ & $\mathcal{O}(T^{-1/3})$ & \Checkmark & -- & \XSolidBrush \\
\cite{shen2025convergence} & $L$ & MVR1($\gamma_t=0$) & $\mathcal{O}(1)$ & $\mathcal{O}(T^{-1/4})$ & \XSolidBrush & -- & \XSolidBrush \\
\cite{kim2026convergence} & $L$ & MVR1($\gamma_t=0$) & $\mathcal{O}(1)$ & $\mathcal{O}(T^{-1/4})$ & \XSolidBrush & Taylor & \XSolidBrush \\
\cite{shulgin2025inexact} & $L$ & MVR1($\gamma_t=0$) & $\mathcal{O}(1)$ & $\widetilde{\mathcal{O}}(T^{-1/4})$ & \XSolidBrush & Error & \XSolidBrush \\
\cite{kovalev2025non} & $L_{+}$ & MVR2 & $\mathcal{O}(1)$ & $\mathcal{O}(T^{-1/3})$ & \Checkmark & -- & \XSolidBrush \\
\cite{qian2025muon} & $(L_0,L_1)$ & MVR2 & $\mathcal{O}(1)$ & $\mathcal{O}(T^{-1/3})$ & \XSolidBrush & -- & \XSolidBrush \\
\cite{khirirat2025better} & $(L_0,L_1)_{+}$ & MVR2 & $\mathcal{O}(1)$ & $\mathcal{O}(T^{-1/3})$ & \XSolidBrush & -- & \XSolidBrush \\
\cite{huang2026limuon} & $L_{++}$ & MVR2 & $\mathcal{O}(1)$ & $\mathcal{O}(T^{-1/3})$ & \XSolidBrush & Taylor & \XSolidBrush \\
\cite{huang2026limuon} & $(L_0,L_1)_{+}$ & MVR2 & $\mathcal{O}(1)$ & $\mathcal{O}(T^{-1/3})$ & \XSolidBrush & Taylor & \XSolidBrush \\
\midrule
\rowcolor{customgray}Ours & $L$ & MVR1($\gamma_t=0$) & $\mathcal{O}(1)$ & $\widetilde{\mathcal{O}}(T^{-1/4})$ & \Checkmark & Actual & $\mathcal{O}(T^{-1/4})$ \\
\rowcolor{customgray}Ours & $L$ & MVR1 & $\mathcal{O}(1)$ & $\widetilde{\mathcal{O}}(T^{-1/4})$ & \Checkmark & Actual & $\mathcal{O}(T^{-1/4})$ \\
\rowcolor{customgray}\textbf{Ours} & $L_{+}$ & MVR2 & $\mathcal{O}(1)$ & $\boldsymbol{\widetilde{\mathcal{O}}(T^{-1/3})}$ & \Checkmark & Actual & $\boldsymbol{\mathcal{O}(T^{-1/3})}$ \\
\bottomrule
\end{tabular}

\begin{tablenotes}
    \fontsize{7}{8}\selectfont
    \item[a]
    $L$ denotes smoothness of the objective $f$.
    $L_{++}$ denotes sample-wise smoothness of the stochastic gradients.
    $L_{+}$ denotes mean-square smoothness of the stochastic gradients.
    $(L_0,L_1)$ denotes relaxed smoothness of $f$.
    $(L_0,L_1)_{+}$ denotes relaxed mean-square smoothness of the stochastic gradients. 
    \item[b] Option MVR1 ($\gamma_t=0$) is the standard momentum method, Option MVR1 is the one-batch variance-reduction momentum method, and Option MVR2 is the two-batch variance-reduction momentum method. We summarize them in Algorithm \ref{alg:muon}.
    \item[c] A large initial batch of size $\mathcal{O}(\epsilon^{-1})$ is required; subsequent batches can have size 1.
    \item[d] These bounds contain a non-vanishing noise floor.
    \item[e] ``Error'' denotes an assumed additive inexact-LMO error bound below one; ``Taylor'' denotes a Taylor-polynomial approximation; ``Actual'' denotes direct analysis of the original Muon Newton--Schulz coefficients~\citep{jordan2024muon} for every finite $K\ge1$. For LiMuon~\citep{huang2026limuon}, the NS error control follows the Taylor-polynomial analysis of \citet{kim2026convergence}.
    \item[f] Under the PL condition and a uniform gradient bound along the iterates, this column reports $\mathbb{E}[f(\mathbf{X}_T)]-f^*$. Our bounds include finite NS steps and decoupled weight decay under the prescribed horizon-free schedules.
\end{tablenotes}
\end{threeparttable}
}
\end{table}

We develop a \textbf{horizon-free} convergence analysis of Muon-MVR. Our contributions are fourfold:
\begin{itemize}[
    leftmargin=1.0em,
    labelsep=0.5em,    
    itemsep=2pt,
    topsep=2pt,
    parsep=0pt,
    partopsep=0pt]
    \item In general non-convex settings, we analyze both the standard Muon (Algorithm \ref{alg:muon}, Option MVR1 with $\gamma_t=0$) and Muon-MVR1, proving their baseline anytime rate $\widetilde{\mathcal{O}}(T^{-1/4})$.
    \item Muon-MVR2 (Algorithm~\ref{alg:muon}, Option MVR2) attains an anytime rate $\widetilde{\mathcal{O}}(T^{-1/3})$ with constant mini-batches in unconstrained stochastic nonconvex optimization. Under mean-square smoothness, this matches the stochastic first-order lower bound up to logarithmic factors.
    \item {To the best of our knowledge, this is the first horizon-free analysis of variance-reduced Muon to jointly cover constant mini-batches, finite NS steps with the original Muon coefficients~\citep{jordan2024muon}, and decoupled weight decay in the stochastic nonconvex setting. We establish bounds on momentum alignment and update norms that hold uniformly over the NS step count and preserve both MVR rates. A further nuclear-norm alignment bound improves with the number of NS steps.}
    \item Under the PL condition and a uniform gradient bound along the iterates, we establish anytime last-iterate objective-gap rates of $\mathcal{O}(T^{-1/4})$ for Muon-MVR1 and $\mathcal{O}(T^{-1/3})$ for Muon-MVR2 with finite NS steps and decoupled weight decay.
\end{itemize}
Table \ref{tab:compare} summarizes the main contributions of our work and compares them with existing methods.

$\bullet$~\textbf{Organization}. 
The rest of the paper is organized as follows. 
Section~\ref{section-rw} reviews related work and situates our contributions among existing Muon-type methods. 
Section~\ref{section-method} presents the algorithmic setup and the Muon variants analyzed in this paper. 
Section~\ref{section-conv} establishes nonconvex stationarity rates and PL-based last-iterate convergence guarantees. 
Section~\ref{section-exp} reports empirical results validating the effectiveness of our methods. 
Finally, Section~\ref{section-conclusion} concludes the paper.

$\bullet$~\textbf{Notations}. We denote scalars by non-bold letters (e.g., $a, A$), vectors in $\mathbb{R}^d$ by bold lowercase letters (e.g., $\mathbf{a}$), and matrices by bold uppercase letters (e.g., $\mathbf{A}$). The space $\mathbb{R}^d$ is endowed with the Euclidean inner product $\langle \mathbf{x}, \mathbf{y} \rangle_2 := \mathbf{x}^\top \mathbf{y}$ and norm $\|\mathbf{x}\|_2$. For matrices, we employ the Frobenius inner product $\langle \mathbf{A}, \mathbf{B} \rangle_F := \text{tr}(\mathbf{A}^\top \mathbf{B})$ and the corresponding norm $\|\mathbf{A}\|_F$. The nuclear norm, denoted by $\|\mathbf{A}\|_*$, is defined as the sum of the singular values of the matrix, $\|\mathbf{A}\|_* = \sum_i \sigma_i(\mathbf{A})$. Throughout the paper, $[m]$ denotes the set of integers $\{1, 2, \ldots, m\}$, and $\mathbb{N}$ denotes the set of non-negative integers. 
For a matrix $\mathbf{M}\in\mathbb{R}^{m\times n}$ with rank $r$, let $\mathbf{M}=\mathbf{U}_r\Sigma_r\mathbf{V}_r^\top$
be its compact SVD over the positive singular values. We define the matrix-sign normalization as
$\operatorname{msign}(\mathbf{M}):=\mathbf{U}_r\mathbf{V}_r^\top,
\operatorname{msign}(0):=0.$
Equivalently,
\(\operatorname{msign}(\mathbf{M})=\mathbf{M}(\mathbf{M}^\top \mathbf{M})^{\dagger/2}\).
where $A^{\dagger/2}$ denotes the Moore--Penrose pseudoinverse square root.
It satisfies
\(\langle \mathbf{M},\operatorname{msign}(\mathbf{M})\rangle_F=\|\mathbf{M}\|_*\) and
\(\|\operatorname{msign}(\mathbf{M})\|_F^2=\operatorname{rank}(\mathbf{M})\le n\).
The model is parameterized by a matrix $\mathbf{X} \in \mathbb{R}^{m \times n}$. Without loss of generality, we assume $m \ge n$, so the rank of the matrix is at most $n$. The model is optimized by minimizing the empirical loss function $f(\mathbf{X}) := \frac{1}{N} \sum_{i \in [N]} f_i(\mathbf{X})$, where $N$ is the number of training data points, and $f_i(\mathbf{X})$ is the loss function for $\mathbf{X} \in \mathbb{R}^{m \times n}$ with respect to the $i$-th training data point $\mathbf{z}_i$ (for $i \in [N]$). Let $\xi$ be a random variable that is independent of $\mathbf{X} \in \mathbb{R}^{m \times n}$, and let $\mathbb{E}_\xi[\mathbf{X}]$ denote the expectation of a random variable $\mathbf{X}$ with respect to $\xi$. 

\section{Related Work}
\label{section-rw}
\subsection{The Evolution of Optimization Algorithms}
Beyond Adam variants, research has explored other paradigms such as preconditioning methods that use parameter curvature.~\cite{Gupta2018ShampooPS} pioneered this direction with the Shampoo optimizer. Building on this work,~\cite{jordan2024muon} proposed Muon, which exploits matrix structure by orthogonalizing gradient momentum. Subsequent variants emerged, such as AdaMuon~\cite{si2025adamuon} which adds element-wise adaptivity, and COSMOS~\cite{liu2025cosmos} which integrates ideas from SOAP~\cite{vyas2025soap} for large model training. These methods emphasize practical optimizer design. On the theoretical side, Gluon~\cite{riabinin2025gluon} introduces a novel layer-wise smoothness assumption, providing convergence guarantees that align with the practical implementations of optimizers like Muon and Scion~\cite{pethick2025training}. In the literature, other related preconditioning methods include ASGO~\cite{An2025ASGOAS}, PolarGrad~\cite{lau2025polargrad}, and AdaGO~\cite{Zhang2025AdaGradMM}, which introduces the Adagrad-Norm step size~\cite{ward2020adagrad} into a simplified version of Muon. 
Meanwhile, other high-performing optimizers not belonging to the Shampoo family also warrant attention, such as Sophia~\cite{Liu2023SophiaAS}, which improves second-moment estimation through efficient diagonal Hessian approximation and coordinate clipping, and Lion~\cite{Chen2023SymbolicDO}, a lightweight optimizer that only tracks momentum and uses the sign function to normalize updates. These methods are closely related to normalized SGD with momentum variants, where the gradient is rescaled or truncated before the update~\cite{cutkosky2020momentum,cutkosky2021high}, in contrast to Muon and our Muon-MVR variants, which orthogonalize matrix-valued momentum through matrix-sign normalization.

\subsection{Analysis of Muon}
Steepest-descent and trust-region interpretations identify Muon's exact polar update with a spectral-norm LMO~\citep{bernstein2024old,kovalev2025understanding}. Constrained formulations connect weight-decayed Muon to stochastic Frank--Wolfe methods~\citep{sfyraki2025lions}, Lion-$\mathcal K$ and KKT convergence~\citep{chen2025muon}, and norm-constrained conditional-gradient theory~\citep{pethick2025training}. Complementary studies characterize its spectral max-margin bias~\citep{fan2025implicit} and distinguish curvature from gradient anisotropy through polar preconditioning~\citep{lau2025polargrad}. Practical work also examines weight decay and RMS-based learning-rate transfer from Adam~\citep{Liu2025MuonIS}.

Convergence analyses cover standard Muon~\citep{Li2025ANO,shen2025convergence}, Nesterov momentum and batch-size effects~\citep{Sato2025ConvergenceBA,nagashima2026improved}, and variance-reduced non-Euclidean or LMO methods~\citep{kovalev2025non,khirirat2025better}. Concurrent with our initial work,
\citet{huang2026limuon} independently developed
LiMuon, which combines STORM-type variance reduction
with Muon and introduces low-rank momentum compression. Subsequent to our initial version, \citet{qian2025muon} analyze MVR within Gluon under layer-wise smoothness, including a decreasing-step-size variant with a $\widetilde{\mathcal O}(T^{-1/3})$ rate. Appendix~\ref{appendix:theoretical-novelty} details the relation between our initial result, subsequent work, and the extensions in this manuscript.

For inexact orthogonalization, \citet{shulgin2025inexact} use an additive inexact-LMO model, while \citet{kim2026convergence} study Taylor-truncated inverse-square-root polynomials. Direct alignment analyses cover a different NS polynomial in federated optimization~\citep{takezawa2025fedmuon} and generalized NS maps, including the original coefficients, in strongly convex error bounds~\citep{do2026muon}. Section~\ref{section:practical-muon-mvr} compares these estimates with our step-uniform bounds. Other work studies condition-number-independent linear convergence for simplified Muon dynamics~\citep{ma2026preconditioning} and spectral conditions for $\mu$P-compatible training~\citep{zhao2026towards}.

\section{Revisiting the Muon Algorithms}
\label{section-method}
We consider the following optimization problem:
\begin{equation}
\label{problem}
\min_{\mathbf{X}\in \mathbb{R}^{m\times n}}f(\mathbf{X}),\mathrm{~where~}f(\mathbf{X})=\mathbb{E}_{\xi\sim\mathcal{D}}[f(\mathbf{X};\xi)],
\end{equation}
where $f: \mathbb{R}^{m\times n} \rightarrow \mathbb{R}$ is the loss function, $\mathbf{X}$ denotes the decision variable, and $\xi$ represents a random variable (e.g., a training data sample) drawn from an unknown distribution $\mathcal{D}$. We assume that $f$ is differentiable and possibly nonconvex.

{\subsection{Muon-MVR1 and Muon-MVR2}
\label{section:muon-mvr-estimators}}

The Muon optimizer begins by computing a momentum-based variance-reduced gradient update, similar in spirit to ADAM \cite{Kingma2014AdamAM}, STORM \cite{Cutkosky2019MomentumBasedVR}, and SGD with Nesterov momentum \cite{sutskever2013importance}.
{The update direction $\mathbf{O}_t$ is obtained from $\mathbf{M}_t$ using exact matrix-sign normalization or a finite number of Newton--Schulz steps. The parameter update also includes optional decoupled weight decay.}
The Muon algorithm is summarized in Algorithm \ref{alg:muon}.
\begin{algorithm*}
\caption{Muon-style Algorithm}
\label{alg:muon}
\begin{algorithmic}[1]
\STATE {\textbf{Input:} $\mathbf{X}_1\in\mathbb{R}^{m\times n}$, schedules $\eta_t>0$, $\beta_t\in[0,1)$, $\gamma_t\in[0,1]$, $\lambda_t\ge0$, Muon scale $a>0$, and orthogonalization parameter $K_t\in\mathbb{N}_0\cup\{\infty\}$ ($\infty$ selects exact normalization).}
\FOR{$t = 1$ {\bf to} $T$}
    \STATE Compute stochastic gradient: $\nabla f(\mathbf{X}_t;\xi_t)$
    \IF{$t=1$}
        \STATE $\mathbf{M}_1=\nabla f(\mathbf{X}_1;\xi_1)$
    \ELSE
        \STATE $\color{red}\text{Option MVR1: One-batch Momentum Variance-Reduction (MVR1)}$ \label{line:opt2}
        \STATE $\mathbf{M}_t = \beta_t \mathbf{M}_{t-1} + (1-\beta_t)\nabla f(\mathbf{X}_t;\xi_t) + \gamma_t\cdot \beta_t \cdot (\nabla f(\mathbf{X}_t;\xi_t)-\nabla f(\mathbf{X}_{t-1};\xi_{t-1}))$
        \STATE $\color{red}\text{Option MVR2: Two-batch Momentum Variance-Reduction (MVR2)}$ \label{line:opt3}
        \STATE $\mathbf{M}_t = \beta_t \mathbf{M}_{t-1} + (1-\beta_t)\nabla f(\mathbf{X}_t;\xi_t) + \gamma_t\cdot \beta_t \cdot (\nabla f(\mathbf{X}_t;\xi_t)-\nabla f(\mathbf{X}_{t-1};\xi_t))$
    \ENDIF
    \STATE {$\mathbf{O}_t=\mathcal{Q}_{K_t}(\mathbf{M}_t)$, where $\mathcal{Q}_{\infty}=\operatorname{msign}$ and $\mathcal{Q}_{K_t}=\operatorname{NS}_{K_t}$ for $K_t<\infty$}
    \STATE {$\mathbf{X}_{t+1}=(1-\eta_t\lambda_t)\mathbf{X}_t-a\eta_t\mathbf{O}_t$}
\ENDFOR

\STATE \textbf{Output:} Final parameters $\mathbf{X}_{T+1}$
\end{algorithmic}
\end{algorithm*}

{The exact branch of Algorithm~\ref{alg:muon} uses the matrix-sign normalization of $\mathbf{M}_t$.}
If $\mathbf{M}_t=\mathbf{U}_{t,r}\Sigma_{t,r}\mathbf{V}_{t,r}^\top$ is the compact SVD over its positive singular values, then
\(\mathbf{O}_t=\operatorname{msign}(\mathbf{M}_t)=\mathbf{U}_{t,r}\mathbf{V}_{t,r}^\top
=\mathbf{M}_t(\mathbf{M}_t^\top \mathbf{M}_t)^{\dagger/2}\).
Thus the exact map normalizes the nonzero singular values of $\mathbf{M}_t$ to one, while rank-deficient directions remain zero. Since explicitly forming and factorizing $\mathbf{M}_t^\top\mathbf{M}_t$ is costly, practical Muon uses a small, fixed number of quintic Newton--Schulz steps~\cite{jordan2024muon}. Section~\ref{section:practical-muon-mvr} analyzes this finite-step map for the practical coefficients. Recent theory also studies finite-step Newton--Schulz Muon~\cite{kim2026convergence}. {Setting $K_t=\infty$, $\lambda_t=0$, and $a=1$ gives the exact Muon update without weight decay.}

Algorithm \ref{alg:muon} incorporates two distinct strategies for momentum-based variance reduction, termed Muon-MVR1 and Muon-MVR2. These options present a fundamental trade-off between computational efficiency and theoretical rigor. MVR2 implements a principled variance reduction scheme at the cost of two gradient evaluations per step, while MVR1 serves as a computationally cheaper, single-gradient approximation. We detail both below.

$\blacktriangleright$ \textbf{Option 1: Muon-MVR1}. For $t\ge2$, MVR1 augments the classical momentum update with a variance-reducing term that reuses the gradient from the previous step:
\begin{equation}
\label{o2}
\begin{aligned}
\mathbf{M}_t &= \beta_t \mathbf{M}_{t-1} + (1-\beta_t)\nabla f(\mathbf{X}_{t};\xi_t) 
+ \gamma_t\cdot \beta_t \cdot (\nabla f(\mathbf{X}_{t};\xi_t)-\nabla f(\mathbf{X}_{t-1};\xi_{t-1})).
\end{aligned}
\end{equation}
The primary advantage of this formulation is its computational efficiency, as it requires only one stochastic gradient evaluation per iteration. This update rule is flexible:

\textit{(i)} When $\gamma_t\equiv0$, Rule (\ref{o2}) degenerates to the standard exponential moving average (EMA) of gradients (Rule (\ref{o1})), a stochastic gradient estimator widely used in optimizers like Adam \cite{Kingma2014AdamAM}.
\begin{equation}
\label{o1}
\mathbf{M}_t = \beta_t \mathbf{M}_{t-1} + (1-\beta_t)\nabla f(\mathbf{X}_t;\xi_t).
\end{equation}
\textit{(ii)} By setting $\beta_t = \mu$ and $\gamma_t = 1 - \mu$, the update rule~(\ref{o2}) yields a momentum term that, after rescaling by $1/(1 - \mu)$, satisfies the recurrence $\widetilde{\mathbf{M}}_t = \mu \widetilde{\mathbf{M}}_{t-1} + \nabla f(\mathbf{X}_t;\xi_t) + \mu(\nabla f(\mathbf{X}_t;\xi_t) - \nabla f(\mathbf{X}_{t-1};\xi_{t-1}))$. This form is algebraically equivalent to the standard Muon optimizer \cite{jordan2024muon,Liu2025MuonIS} derived from Eq.~(\ref{practical}) \cite{Yuan2024MARSUT}, and inherently implements Nesterov acceleration via a first-order Taylor approximation of the gradient \cite{xie2024adan}.
\begin{equation}
\label{practical}
\begin{aligned}
    \mathbf{C}_{t} &= \mu \mathbf{C}_{t-1} + \nabla f(\mathbf{X}_{t};\xi_t),\\
    \mathbf{M}_t &= \mu \mathbf{C}_{t} + \nabla f(\mathbf{X}_{t};\xi_t).\\
\end{aligned} 
\end{equation}
$\blacktriangleright$ \textbf{Option 2: Muon-MVR2}. For $t\ge2$, MVR2 uses a same-sample recursive variance-reduction mechanism related to SPIDER \cite{fang2018spider}, STORM \cite{Cutkosky2019MomentumBasedVR}, SVRG \cite{zhou2020stochastic}, SUPER-Adam \cite{Huang2021SUPERADAMFA}, and MARS \cite{Yuan2024MARSUT}:
\begin{equation}
\label{o3}
\begin{aligned}
\mathbf{M}_t &= \beta_t \mathbf{M}_{t-1} + (1-\beta_t)\nabla f(\mathbf{X}_{t};\xi_t) + \gamma_t \cdot \beta_t \cdot (\nabla f(\mathbf{X}_{t};\xi_t)-\nabla f(\mathbf{X}_{t-1};\xi_t)).
\end{aligned}
\end{equation}
The key distinction from MVR1 is the correction term. MVR2 subtracts the gradient computed on the previous parameters but with the current mini-batch, i.e., $\nabla f(\mathbf{X}_{t-1};\xi_t)$. This modification is crucial as it is designed to directly cancel the variance introduced by the mini-batch $\xi_t$ \cite{Cutkosky2019MomentumBasedVR,Huang2021SUPERADAMFA,Yuan2024MARSUT}. However, this theoretical benefit comes at the cost of requiring two gradient evaluations per step. MVR1 can be formally understood as a practical approximation of MVR2. The difference between their update rules is a single noise term, $\Delta^{\text{Noise}}_{t-1} = \nabla f(\mathbf{X}_{t-1};\xi_t)-\nabla f(\mathbf{X}_{t-1};\xi_{t-1})$. Under the standard assumption of bounded variance ($\mathbb{E}_{\xi}[\|\nabla f(\mathbf{X};\xi) - \nabla f(\mathbf{X})\|_F^2] \le \sigma^2$), the variance of this noise is well-controlled, satisfying $\mathbb{E}[\|\Delta^{\text{Noise}}_{t-1}\|_F^2] \le 2\sigma^2$. While MVR1 is often sufficient in practice, this structural difference leads to fundamentally different theoretical guarantees. The principled variance cancellation in MVR2 allows our algorithm to achieve a superior ergodic
convergence rate of $\widetilde{\mathcal{O}}(T^{-1/3})$, as we will formally establish in Theorem \ref{th-nonconvex-mvr2}.

{
\subsection{Practical Muon-MVR}
\label{section:practical-muon-mvr}

Practical Muon uses finitely many Newton--Schulz (NS) steps and decoupled
weight decay. Algorithm~\ref{alg:muon} incorporates both through
$\mathcal{Q}_{K_t}$ and $\lambda_t$, while retaining the MVR estimators in
Eqs.~\eqref{o2}--\eqref{o3}.

For $\mathbf A\ne0$, the NS iteration starts with Frobenius normalization:
\begin{equation}
\mathbf Y^{(0)}(\mathbf A)=\frac{\mathbf A}{\|\mathbf A\|_F},\qquad
\mathbf Y^{(k+1)}(\mathbf A)
=p_{\rm ns}\!\left(\mathbf Y^{(k)}(\mathbf A)
\mathbf Y^{(k)}(\mathbf A)^\top\right)\mathbf Y^{(k)}(\mathbf A),
\label{eq:practical-ns-iteration}
\end{equation}
where
\begin{equation}
p_{\rm ns}(\mathbf Z)
=b_1\mathbf I+b_2\mathbf Z+b_3\mathbf Z^2.
\label{eq:practical-ns-polynomial}
\end{equation}
The common choice $(b_1,b_2,b_3)=(3.4445,-4.7750,2.0315)$ comes from the
original Muon implementation~\citep{jordan2024muon}. These coefficients are
also used in large-scale LLM training by Moonshot~\citep{Liu2025MuonIS} and
in the first eight steps of DeepSeek-V4's hybrid NS scheme~\citep{deepseek2026v4}.
All NS bounds below use these fixed values.
For finite $K\ge0$, the resulting map is
$\operatorname{NS}_K(\mathbf A)=\mathbf Y^{(K)}(\mathbf A)$, with
$\operatorname{NS}_K(0)=0$.

For the theorems below, we fix $K_t\equiv K\ge1$ and write the update as
\begin{align}
    \mathbf{O}_t&=\operatorname{NS}_K(\mathbf{M}_t),
    \label{eq:practical-ns-direction}\\
    \mathbf{X}_{t+1}
    &=(1-\eta_t\lambda_t)\mathbf{X}_t-a\eta_t\mathbf{O}_t,
    \label{eq:practical-model-update}
\end{align}
where $a>0$ is fixed and $\lambda_t\ge0$ is the weight-decay coefficient.
Appendix~\ref{appendix:practical-ns-geometry} proves the geometric bounds
used below.

Existing analyses control descent either through the error relative
to the exact polar factor or through alignment with the momentum.
The first approach includes the Taylor-truncated inverse-square-root
analysis of \citet{kim2026convergence} and the additive inexact-LMO
model of \citet{shulgin2025inexact}. For LiMuon~\citep{huang2026limuon}, the NS error control
follows the Taylor-polynomial analysis of
\citet{kim2026convergence} and uses a uniform polar
approximation error bound below one.
With $\mathbf{P}_t=\operatorname{msign}(\mathbf{M}_t)$, the
nuclear--spectral H\"older inequality gives
\[
\langle\mathbf{M}_t,\mathbf{O}_t\rangle_F
\ge\bigl(1-\|\mathbf{O}_t-\mathbf{P}_t\|_2\bigr)\|\mathbf{M}_t\|_*.
\]
This estimate gives a uniform positive descent coefficient when
$\sup_t\|\mathbf{O}_t-\mathbf{P}_t\|_2\le\delta<1$.
For fixed finite $K$, small normalized positive singular values
can make this error arbitrarily close to one.

Direct alignment estimates provide another approach.
\citet{takezawa2025fedmuon} establish alignment bounds in a Schatten
norm for the coefficients $(15/8,-5/4,3/8)$, whose scalar iterates
increase to one. The practical coefficients produce non-monotone
scalar dynamics and map $1$ to $b_1+b_2+b_3=0.701$.
\citet{do2026muon} analyze the original coefficients under strong
convexity. Their Eq.~(3.200) contains the product
\[
\prod_{k=1}^{K}\min\left\{\inf_{x\in\mathbb R}\sum_{j=0}^{\infty}b_{k,j}x^{2j},1\right\}.
\]
Each factor lies in $(0,1]$, so the product is non-increasing in $K$.
For the practical coefficients, it decreases with each additional NS
step.\footnote{Here $b_{k,j}$ are their NS
polynomial coefficients, with $K$ replacing their $n$. For the
practical coefficients, each factor is $b_1-b_2^2/(4b_3)\approx0.6386$.}

We bound the scalar NS iterates using an invariant interval and
the lower envelope $s\mapsto\min\{3s,2/3\}$.
Summing the resulting estimates over the singular values gives,
for every matrix $\mathbf{M}$ and every finite $K\ge1$,
\begin{equation}
    \langle\mathbf{M},\operatorname{NS}_K(\mathbf{M})\rangle_F
    \ge \tfrac23\|\mathbf{M}\|_F,
    \qquad
    \|\operatorname{NS}_K(\mathbf{M})\|_F\le\tfrac54\sqrt n.
    \label{eq:practical-ns-structural-bounds}
\end{equation}
We now relate these geometric bounds to descent through
\[
\langle\nabla f(\mathbf{X}_t),\mathbf{O}_t\rangle_F
\ge\langle\mathbf{M}_t,\mathbf{O}_t\rangle_F
-\|\nabla f(\mathbf{X}_t)-\mathbf{M}_t\|_F\|\mathbf{O}_t\|_F.
\]
The alignment bound controls the first term, and the norm bound
controls the contribution of the gradient-tracking error.

The same scalar estimates also give
\begin{equation}
\langle\mathbf{M},\operatorname{NS}_K(\mathbf{M})\rangle_F
\ge\max\!\left\{\frac{2}{3\sqrt n},
\frac{4\cdot3^K}{n+6\cdot3^K}\right\}\|\mathbf{M}\|_*.
\label{eq:practical-ns-nuclear-bound}
\end{equation}
Its coefficient is non-decreasing in $K$, and the bound holds
for arbitrary singular spectra. Combined with the MVR2 tracking
estimate, it yields the nuclear-norm stationarity bound in
Appendix~\ref{proof:th-nonconvex-mvr2}.

For exact matrix-sign normalization, the corresponding bounds are
\[
\langle\mathbf{M},\operatorname{msign}(\mathbf{M})\rangle_F
=\|\mathbf{M}\|_*\ge\|\mathbf{M}\|_F,
\qquad
\|\operatorname{msign}(\mathbf{M})\|_F\le\sqrt n.
\]

With the learning-rate, momentum, and weight-decay schedules in
Theorems~\ref{th-nonconvex-mvr1} and~\ref{th-nonconvex-mvr2},
these geometric estimates yield the anytime rates
$\widetilde{\mathcal{O}}(T^{-1/4})$ for MVR1 and
$\widetilde{\mathcal{O}}(T^{-1/3})$ for MVR2.
}

\section{Convergence Analysis}
\label{section-conv}
We begin by placing our results in the standard framework of first-order stochastic optimization. The $\mathcal{O}(T^{-1/4})$ rate is a standard baseline for stochastic nonconvex optimization and has also been established for standard Muon~\cite{shen2025convergence}. Accordingly, we first analyze Muon-MVR1 and recover this baseline anytime convergence rate in Theorem \ref{th-nonconvex-mvr1}. While the rate itself is standard, our main contribution here is a unified proof framework that also underpins the sharper results for Muon-MVR2.

We focus on \textbf{horizon-free} schedules that yield anytime guarantees. Fixed-horizon constant-stepsize variants may remove the logarithmic factors, but are not pursued here.

To facilitate the analysis of convergence for Muon, we make the following assumptions:
\begin{assumption}
\label{ass:1}
The function $f$ is bounded from below. There exists $f^* > -\infty$ such that $f(\mathbf{X}) \geq f^*$, for all $\mathbf{X} \in \mathbb{R}^{m\times n}$.
\end{assumption}

\begin{assumption}
\label{ass:2}
The function $f$ is $L$-smooth: $\|\nabla f(\mathbf{Y}) - \nabla f(\mathbf{X})\|_F \leq L \|\mathbf{Y} - \mathbf{X}\|_F$.    
\end{assumption}

\begin{assumption}
\label{ass:2.2}
The stochastic gradients are $L$-smooth in mean square: for all $\mathbf{X},\mathbf{Y}\in\mathbb{R}^{m\times n}$,
$\mathbb{E}_{\xi}\|\nabla f(\mathbf{Y};\xi) - \nabla f(\mathbf{X};\xi)\|_F^2 \leq L^2 \|\mathbf{Y} - \mathbf{X}\|_F^2$.
\end{assumption}

\begin{remark}
The $L_{+}$ label in Table~\ref{tab:compare} also covers the operator-weighted mean-square condition of \cite{kovalev2025non}; both control same-sample gradient differences.
\end{remark}

\begin{assumption}
\label{ass:3}
The variance of unbiased stochastic gradient is finite. Specifically, there exists a constant $\sigma > 0$ such that for all $\mathbf{X} \in \mathbb{ R}^{m\times n}$, the following holds: $\mathbb{E}[\nabla f(\mathbf{X}; \xi)] = \nabla f(\mathbf{X})$ and $\mathbb{E}\|\nabla f(\mathbf{X};\xi) - \nabla f(\mathbf{X})\|_F^2 \leq \sigma^2$.
\end{assumption}

By unbiasedness and Jensen's inequality, Assumptions~\ref{ass:2.2} and~\ref{ass:3} imply Assumption~\ref{ass:2} with the same $L$.
These assumptions are quite common~\cite{Yuan2024MARSUT,Zhou2018OnTC,Chen2018OnTC,Huang2021SUPERADAMFA,Guo2021ANC,Li2023ConvergenceOA,Wang2023ClosingTG,xie2024adan,Chang2026MuonEqBB}. Global smoothness and uniformly bounded variance are nevertheless analytical idealizations that need not hold throughout neural-network training. Edge-of-stability observations further highlight limitations of conventional smoothness-based descriptions of practical training dynamics~\cite{Cohen2021GradientDO}.

\subsection{Ergodic Convergence of Muon}
\label{section-conv.1}
In this subsection, we establish the ergodic convergence of Muon-MVR1 and Muon-MVR2.
\subsubsection{Option MVR1}
We begin our analysis with Option MVR1, a straightforward implementation of one-batch momentum-based variance reduction. The following theorem establishes its ergodic convergence rate, demonstrating that the algorithm converges to a stationary point at a rate of $\widetilde{\mathcal{O}}(T^{-1/4})$ for specific choices of learning rate and momentum schedules.

{
\begin{theorem}
\label{th-nonconvex-mvr1}
Suppose Assumptions~\ref{ass:1}, \ref{ass:2}, and \ref{ass:3} hold.  Let
$\{\mathbf{X}_t\}_{t\ge1}$ be generated by Eq.~\eqref{o2} and the practical update
\eqref{eq:practical-ns-direction}--\eqref{eq:practical-model-update}, with
$\eta_t=t^{-3/4}$.  Fix $a>0$ and $0\le\rho<2a/3$, and suppose
$0\le\lambda_t\le \rho t^{-1/4}/(\|\mathbf{X}_1\|_F+5a\sqrt n)$.
Define $\kappa_\rho:=2a/3-\rho$.
Consider the following two MVR1 parameterizations:
\begin{enumerate}[leftmargin=*, itemsep=0pt, topsep=0pt]
    \item For MVR1 ($\gamma_t=0$), set $\beta_t=1-t^{-1/2}$, and define
    $\mathcal{A}_{1,\rho}^{(1)}:=2L^{-1}\sigma^2
    +\frac{L}{2}\bigl[(8\sqrt2+1)(\tfrac54a\sqrt n+\rho)^2+a^2(\tfrac49+\tfrac{25}{16}n)\bigr]$ and
    $\mathcal{A}_{1,\rho}^{(2)}:=4L^{-1}\sigma^2
    +\frac{L}{2}\bigl[(8\sqrt2+1)(\tfrac54a\sqrt n+\rho)^2+a^2(\tfrac49+\tfrac{25}{16}n)\bigr]$.
    \item For MVR1 ($\gamma_t=t^{-1/2}$), set
    $\beta_t=1-(t+1)^{-1/2}$, and define
    $\mathcal{A}_{2,\rho}^{(1)}:=4L^{-1}\sigma^2
    +\frac{L}{2}\bigl[(16\sqrt2+1)(\tfrac54a\sqrt n+\rho)^2+a^2(\tfrac49+\tfrac{25}{16}n)\bigr]$ and
    $\mathcal{A}_{2,\rho}^{(2)}:=12L^{-1}\sigma^2
    +\frac{L}{2}\bigl[(16\sqrt2+1)(\tfrac54a\sqrt n+\rho)^2+a^2(\tfrac49+\tfrac{25}{16}n)\bigr]$.
\end{enumerate}
Then, for either choice $i\in\{1,2\}$ and any integer $T\ge1$,
\begin{equation}
    \frac1T\sum_{t=1}^T
       \mathbb{E}\|\nabla f(\mathbf{X}_t)\|_F
    \le
    \frac{f(\mathbf{X}_1)-f^*
       +\mathcal{A}_{i,\rho}^{(1)}\ln T
       +\mathcal{A}_{i,\rho}^{(2)}}
       {\kappa_\rho T^{1/4}}.
    \label{eq:practical-mvr1-main-bound}
\end{equation}
Hence finite-step NS and decoupled weight decay preserve the anytime
$\widetilde{\mathcal{O}}(T^{-1/4})$ rate.  See
Appendix~\ref{appendix:practical-mvr1-proof} for the proof.
\end{theorem}

\begin{remark}
The same proof applies to exact matrix-sign normalization using
$\langle\mathbf{M}_t,\mathbf{O}_t\rangle_F\ge\|\mathbf{M}_t\|_F$
and $\|\mathbf{O}_t\|_F\le\sqrt n$.
Appendix~\ref{appendix:practical-mvr1-proof} derives the
corresponding exact-update bounds.
\end{remark}
}

\subsubsection{Option MVR2}
Prior work establishes ergodic rates of $\widetilde{\mathcal{O}}(T^{-1/3})$ for variance-reduced SGD-type methods~\citep{fang2018spider,Cutkosky2019MomentumBasedVR,zhou2020stochastic} and Adam-type methods~\citep{Huang2021SUPERADAMFA,Yuan2024MARSUT} in nonconvex settings. Recent analyses have sharpened the theory of standard Muon and practical Newton--Schulz Muon \citep{kim2026convergence,nagashima2026improved}, while our focus is the two-batch MVR2 estimator in Eq.~(\ref{o3}) with constant mini-batches, horizon-free schedules, finite NS steps, and decoupled weight decay. Theorem \ref{th-nonconvex-mvr2} complements these results by proving that Option MVR2 of Algorithm \ref{alg:muon} (Muon-MVR2) attains the same $\widetilde{\mathcal{O}}(T^{-1/3})$ ergodic convergence rate in the general non-convex setting.

{
\begin{theorem}
\label{th-nonconvex-mvr2}
Suppose Assumptions~\ref{ass:1}, \ref{ass:2.2}, and \ref{ass:3} hold.  Let
$\{\mathbf{X}_t\}_{t\ge1}$ be
generated by the MVR2 estimator in Eq.~\eqref{o3} and the practical update
\eqref{eq:practical-ns-direction}--\eqref{eq:practical-model-update}, with \(\eta_t=t^{-2/3},\beta_t=1-\eta_t,\gamma_t=1\). Fix $a>0$ and $0\le\rho<2a/3$, and suppose
\begin{equation}
    0\le\lambda_t\le\lambda_0t^{-1/3},
    \qquad
    \lambda_0
    :=\frac{\rho}{\|\mathbf{X}_1\|_F
       +\tfrac{15}{4}a\sqrt n}.
    \label{eq:practical-wd-schedule}
\end{equation}
Define \(\kappa_\rho:=2a/3-\rho,d_\rho:=5a\sqrt n/4+\rho,\mathcal{V}_\rho:=4L^2d_\rho^2+\sigma^2,\)
and \(\mathcal{G}_T:=f(\mathbf{X}_1)-f^*+\frac{2\sigma^2}{L}+2Ld_\rho^2+\left(\frac{4\mathcal{V}_\rho}{L}+\frac{Ld_\rho^2}{2}\right)(1+\ln T).\)
Then, for every integer $T\ge1$,
\begin{equation}
    \frac1T\sum_{t=1}^T\mathbb{E}\|\nabla f(\mathbf{X}_t)\|_F
    \le\frac1{T^{1/3}}
    \left[
      \frac{\mathcal{G}_T}{\kappa_\rho}
      +\sqrt{4\sigma^2+8\mathcal{V}_\rho(1+\ln T)}
    \right].
    \label{eq:practical-main-bound}
\end{equation}
Consequently, finite-step NS and decoupled weight decay preserve the anytime
$\widetilde{\mathcal{O}}(T^{-1/3})$ rate whenever $\rho$ is fixed
independently of $T$ and $\kappa_\rho>0$.

See Appendix~\ref{proof:th-nonconvex-mvr2} for the proof.
\end{theorem}

\begin{remark}
The quantity $\kappa_\rho$ is the coefficient of the descent term,
while $d_\rho$ bounds the update displacement through
$\|\mathbf{X}_{t+1}-\mathbf{X}_t\|_F\le d_\rho\eta_t$.
The condition $\rho<2a/3$ keeps the descent coefficient positive.
Both constants are independent of the finite NS step count $K$.
The exact matrix-sign case follows by using its alignment and
norm bounds in the same proof.
Table~\ref{tab:coefficient-nonconvex} in Appendix~\ref{appendix:coefficient-comparison}
summarizes how orthogonalization and weight decay enter the individual terms of both bounds.
\end{remark}
}

\begin{remark}
\label{rem:param_dependence}
For \textbf{MVR1}, the prefactor in the $\widetilde{\mathcal{O}}(T^{-1/4})$ bound depends linearly on the noise variance and the dimension, i.e., it is of order $L^{-1}\sigma^2 + Ln + L$, so there is no super-linear growth in $n$.
For \textbf{MVR2}, the leading stochastic term behaves like $\sigma + L\sqrt{n}$ (up to a $\sqrt{1+\ln T}$ factor), while the remaining constant term scales as $f(\mathbf{X}_1)-f^* + L^{-1}\sigma^2 + Ln$, which reveals a mixed $\sqrt{n}$- and $n$-dependence.
In contrast to standard parameter-agnostic complexity results for variance-reduced SGD and adaptive methods, which typically hide the dependence on $L$, $\sigma$, and $n$ inside the $\widetilde{\mathcal{O}}(\cdot)$ notation (see, e.g., \cite{fang2018spider,Cutkosky2019MomentumBasedVR,zhou2020stochastic,Huang2021SUPERADAMFA,Yuan2024MARSUT}), our analysis keeps this structure explicit and highlights how the Muon geometry interacts with variance reduction in the matrix-valued setting.
\end{remark}

\subsection{PL-Based Convergence of Muon}
\label{section-conv.2}
This subsection studies last-iterate objective-gap convergence of Muon-MVR1 and Muon-MVR2 with finite NS steps and decoupled weight decay under the Polyak--\L{}ojasiewicz condition.
Our analysis is based on the following additional assumption:
\begin{assumption}
\label{ass:6}
We assume the function $f$ is $\mu$-PL, i.e., $\|\nabla f(\mathbf{X})\|_F^2 \ge 2\mu(f(\mathbf{X}) - f^*)$.
\end{assumption}

\begin{remark}
The PL condition has been widely employed in the convergence analysis of various first-order algorithms \cite{karimi2016linear,xie2020linear,li2022simple}, though typically under restricted settings. 
Note that if $f(\cdot)$ is strongly convex, then it is necessarily convex and satisfies the PL condition. However, the PL condition alone does not imply convexity or strong convexity; a counterexample is given by $f(x)=x^2+3\sin^2(x)$.
\end{remark}
To obtain last-iterate bounds, we impose a uniform gradient bound along the iterates. Together with the PL condition, it yields a linear recursion for the expected objective gap.

\begin{assumption}
\label{ass:7}
There exists a constant $G>0$ such that, almost surely, $\|\nabla f(\mathbf{X}_t)\|_F \le G$ for all $t\ge 1$.
\end{assumption}

\begin{theorem}
\label{th-last-mvr1}
Suppose Assumptions \ref{ass:1}, \ref{ass:2}, \ref{ass:3}, \ref{ass:6}, and \ref{ass:7} hold. Fix $a>0$ and $0\le\rho<2a/3$. Let $\Delta_t := \mathbb{E}[f(\mathbf{X}_t)] - f^*$ and $\kappa := (2\mu/G)(2a/3-\rho)$. Let $\{\mathbf{X}_t\}_{t\ge 1}$ be generated by Algorithm \ref{alg:muon} with \textbf{MVR1}, a fixed finite NS step count $K_t\equiv K\ge1$, $0\le\lambda_t\le\rho t^{-1/4}/(\|\mathbf{X}_1\|_F+5a\sqrt n)$, and stepsize $\eta_t = \eta t^{-3/4}$, where $0< \eta \le \min\{1,\frac{1}{4\kappa}\}$. Define $c := 4\kappa\eta$ and $C_{1/3} := \frac{1}{\kappa}\left(2^{4/3} + \frac{4}{e\kappa\eta}\right)$. Consider the following two \textbf{MVR1} parameterizations:
\begin{enumerate}[leftmargin=*, itemsep=0pt, topsep=0pt]
    \item For \textbf{MVR1} ($\gamma_t = 0$), set $\beta_t = 1 - t^{-1/2}$; define $C_1^{(1)} = e^c\Delta_1 + 4\eta^{2/3}L^{-1}\sigma^2 e^{2^{1/4}c}$ and $C_1^{(2)} = 8\sqrt{2}L(\tfrac54a\sqrt n+\rho)^2\eta^{5/3} + 4L^{-1}\sigma^2\eta^{-1/3} + \frac L2\bigl[a^2(\tfrac49+\tfrac{25}{16}n)+(\tfrac54a\sqrt n+\rho)^2\bigr]\eta^{1/3}$.
    \item For \textbf{MVR1} ($\gamma_t = t^{-1/2}$), set $\beta_t = 1 - (t+1)^{-1/2}$; define $C_2^{(1)} = e^c\Delta_1 + 8\eta^{2/3}L^{-1}\sigma^2 e^{2^{1/4}c}$ and $C_2^{(2)} = 16\sqrt{2}L(\tfrac54a\sqrt n+\rho)^2\eta^{5/3} + 10L^{-1}\sigma^2\eta^{-1/3} + \frac L2\bigl[a^2(\tfrac49+\tfrac{25}{16}n)+(\tfrac54a\sqrt n+\rho)^2\bigr]\eta^{1/3}$.
\end{enumerate}
Then, for either choice $i \in \{1,2\}$ and any integer $T \ge 2$,
\[
\mathbb{E}[f(\mathbf{X}_T)] - f^*
\le
C_i^{(1)} \exp(-cT^{1/4}) + C_i^{(2)} C_{1/3} T^{-1/4}.
\]
\end{theorem}

\begin{theorem}
\label{th-last-mvr2}
Suppose Assumptions \ref{ass:1}, \ref{ass:2.2}, \ref{ass:3}, \ref{ass:6}, and \ref{ass:7} hold. Fix $a>0$ and $0\le\rho<2a/3$. Let $\Delta_t := \mathbb{E}[f(\mathbf{X}_t)] - f^*$ and $\kappa := (2\mu/G)(2a/3-\rho)$. Let $\{\mathbf{X}_t\}_{t\ge 1}$ be generated by Algorithm \ref{alg:muon} under \textbf{MVR2} with a fixed finite NS step count $K_t\equiv K\ge1$ and $0\le\lambda_t\le\rho t^{-1/3}/(\|\mathbf{X}_1\|_F+\tfrac{15}{4}a\sqrt n)$, corresponding to update rule (\ref{o3}). Set $\eta_t = \eta t^{-2/3}$, $\beta_t = 1 - t^{-2/3}$, and $\gamma_t = 1$, where $0<\eta \le \min\{1,\frac{1}{8\kappa}\}$. Define $c := 3\kappa\eta$, $C_{1/2} := \frac{1}{\kappa}\left(2^{3/2} + \frac{4}{e\kappa\eta}\right)$, $C_3^{(1)} = e^c\Delta_1 + 8\eta^{1/2}L^{-1}\sigma^2 e^{2^{1/3}c}$, and $C_3^{(2)} = 64L(\tfrac54a\sqrt n+\rho)^2\eta^{3/2} + 16L^{-1}\sigma^2\eta^{-1/2} + \frac L2\bigl[a^2(\tfrac49+\tfrac{25}{16}n)+(\tfrac54a\sqrt n+\rho)^2\bigr]\eta^{1/2}$. Then, for any integer $T \ge 2$,
\[
\mathbb{E}[f(\mathbf{X}_T)] - f^*
\le
C_3^{(1)} \exp(-cT^{1/3}) + C_3^{(2)} C_{1/2} T^{-1/3}.
\]
\end{theorem}

{
Theorems~\ref{th-last-mvr1} and~\ref{th-last-mvr2} establish both
last-iterate rates with finite NS steps and decoupled weight decay under
the stated schedules. Appendices~\ref{proof:th-last-mvr1}
and~\ref{proof:th-last-mvr2} give the detailed proofs and the
corresponding constants for exact matrix-sign normalization with
$a=1$ and $\lambda_t=0$.
Table~\ref{tab:coefficient-pl} in Appendix~\ref{appendix:coefficient-comparison}
compares the resulting parameter dependence across the four settings.
}

\begin{remark}
The PL condition is a standard geometric regularity assumption that captures a favorable gradient-dominance regime beyond strong convexity without requiring convexity \cite{karimi2016linear}. Our PL analysis complements the general nonconvex theory by identifying a structured regime in which Muon admits explicit last-iterate rates. This condition is related to the broader Kurdyka--\L{}ojasiewicz (KL) framework, which applies to real-analytic and semialgebraic objectives \cite{bolte2014proximal}. Many machine-learning objectives are built from such components, and KL-type convergence analyses have been established for structured deep-learning problems \cite{zeng2019global}. For smooth objectives, the PL inequality used here corresponds specifically to a square-root-type KL inequality with exponent $1/2$; other KL exponents lead to different gradient--gap relations and require their own recursions.

The linear-recursion structure is shared with stochastic conditional-gradient analysis: Eq.~(60) of \citet{mokhtari2018stochastic} combines contraction, gradient-estimation error, and a smoothness remainder. There, contraction follows from convexity and a bounded feasible set with an exact LMO. Here, PL and Assumption~\ref{ass:7} give $\mathbb E\|\nabla f(\mathbf X_t)\|_F\ge(2\mu/G)\Delta_t$ for unconstrained, possibly nonconvex objectives. Our analysis controls the MVR-dependent tracking terms and, for Theorems~\ref{th-last-mvr1} and~\ref{th-last-mvr2}, derives alignment and displacement bounds for finite NS steps with weight decay. These additional estimates determine the forcing terms and rates for Muon within the shared recursion framework.
\end{remark}

\begin{remark}
For completeness, Appendix~\ref{appendix:synthetic-rates} provides two minimal synthetic sanity checks whose sole purpose is to illustrate qualitative agreement with the predicted rates.
\end{remark}

\section{Experiments}
\label{section-exp}

We evaluate Muon-MVR on CIFAR-10 and C4 using practical recipes with cosine learning-rate decay, warmup, fixed momentum, weight decay, and tuned constant $\gamma$. Appendix~\ref{appendix:exp} gives the hardware and configurations for each evaluation. We separately examine horizon-free schedules in Appendix~\ref{appendix:synthetic-rates} and Table~\ref{tab:theory-schedule-cifar}, and finite NS steps with weight decay in Figure~\ref{fig:ns-wd-ablation}.

$\blacktriangleright$ \textbf{ResNet18 on CIFAR10 Dataset.} We train ResNet-18~\cite{he2016deep} on CIFAR-10 for 100 epochs (batch size 128), comparing Muon variants against SGD and Adam over five random seeds. For each optimizer, the learning rate is tuned via grid search over
\(\left\{10^{-4},5\times10^{-4},10^{-3},5\times10^{-3},10^{-2},5\times10^{-2},10^{-1}\right\}\).
As shown in Figures~\ref{fig:cifar_acc} and \ref{fig:cifar_loss}, Muon variants demonstrate faster initial convergence and lower final test error than the baselines, with Muon-MVR2 achieving the best overall performance.
\begin{figure}[!htbp]
    \centering
    \subfloat[CIFAR-10: accuracy]{\label{fig:cifar_acc}%
    \begin{minipage}[t]{0.31\textwidth}
        \centering
        \includegraphics[width=\linewidth]{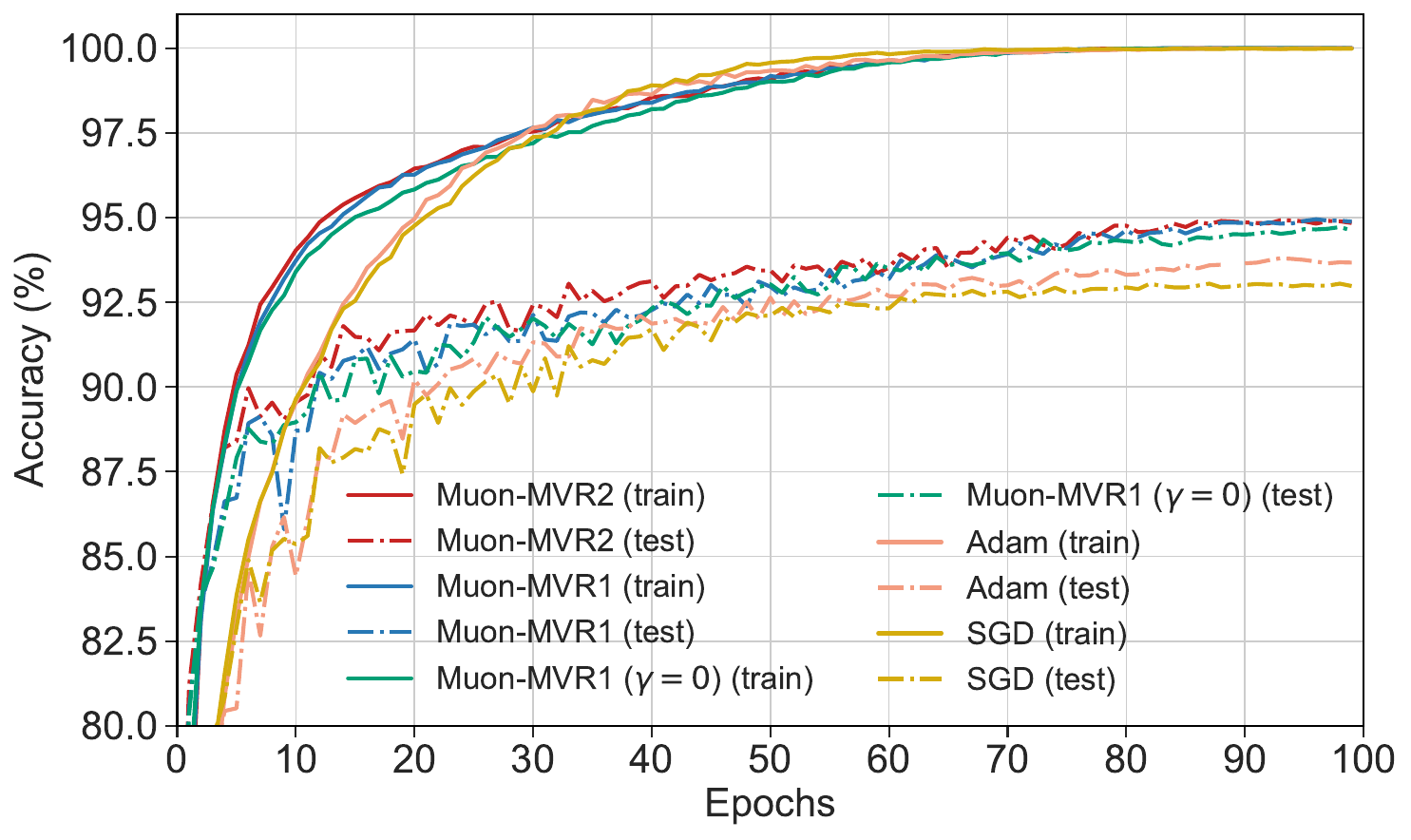}
    \end{minipage}
    }
    \subfloat[CIFAR-10: loss]{\label{fig:cifar_loss}%
    \begin{minipage}[t]{0.31\textwidth}
        \centering
        \includegraphics[width=\linewidth]{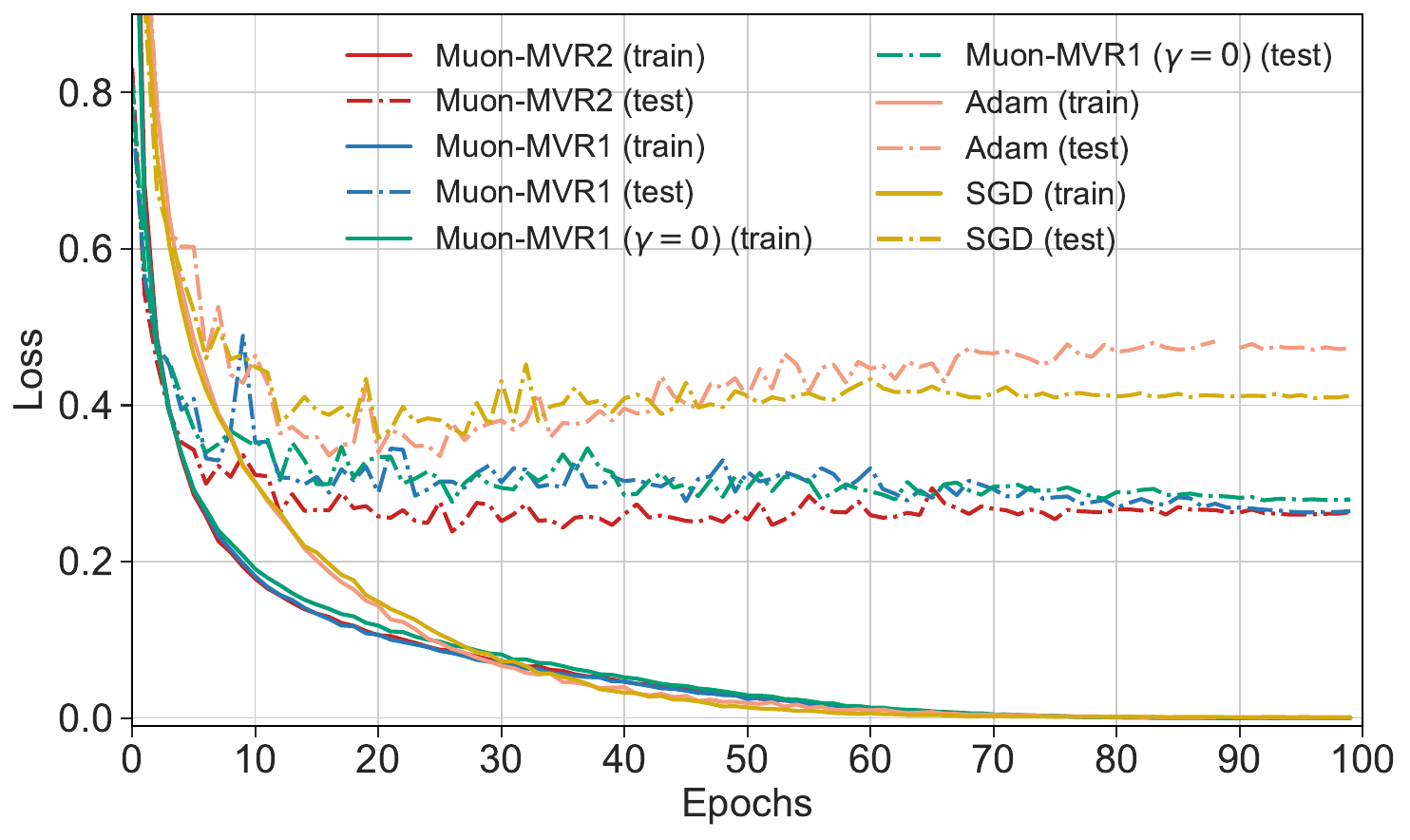}
    \end{minipage}
    }
    \subfloat[CIFAR-10: test accuracy vs. time]{\label{fig:cifar_time}%
    \begin{minipage}[t]{0.31\textwidth}
        \centering
        \includegraphics[width=\linewidth]{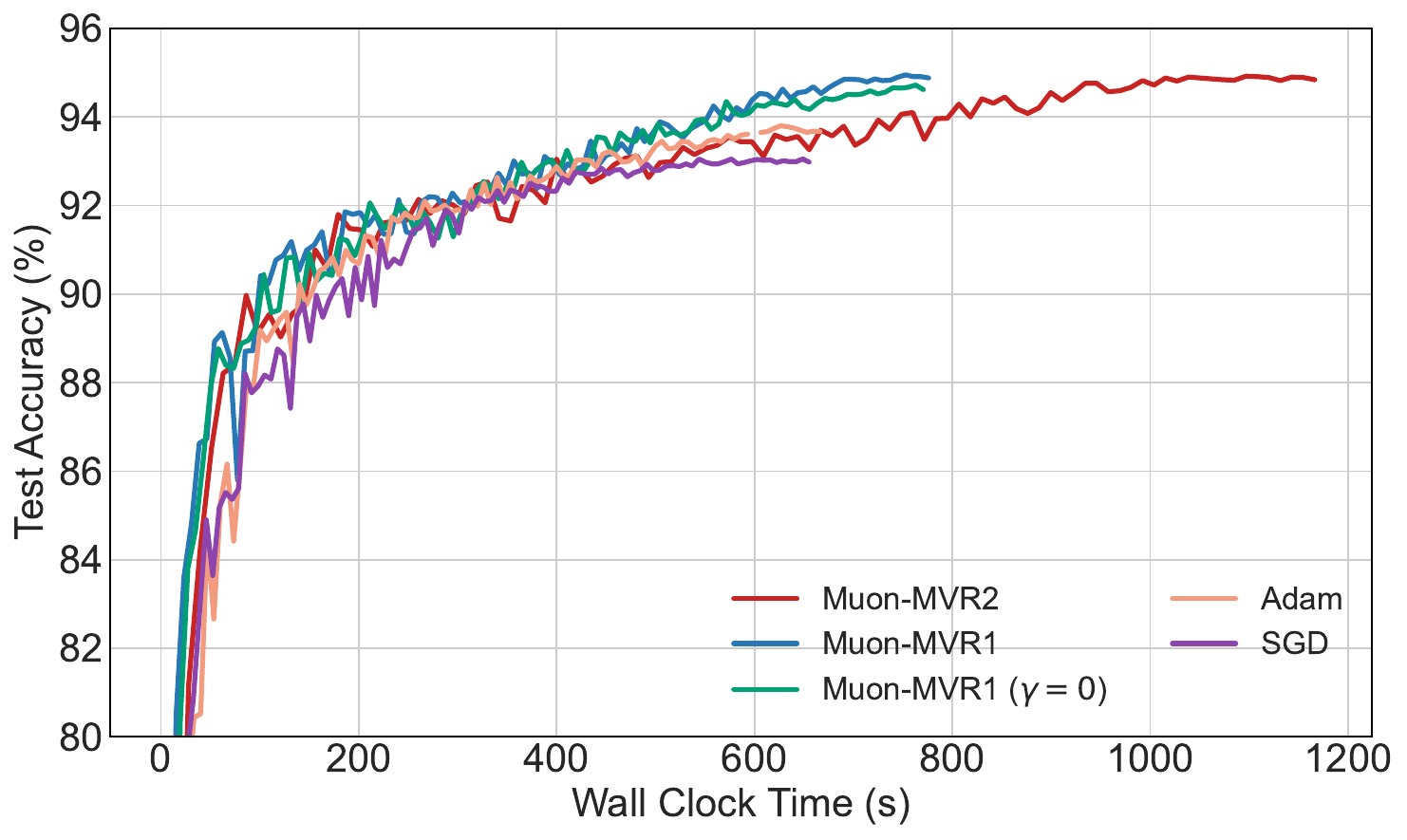}
    \end{minipage}
    }
    \vspace{-10pt}
    \subfloat[C4: training loss]{\label{fig:llama_train_loss}%
    \begin{minipage}[t]{0.31\textwidth}
        \centering
        \includegraphics[width=\linewidth]{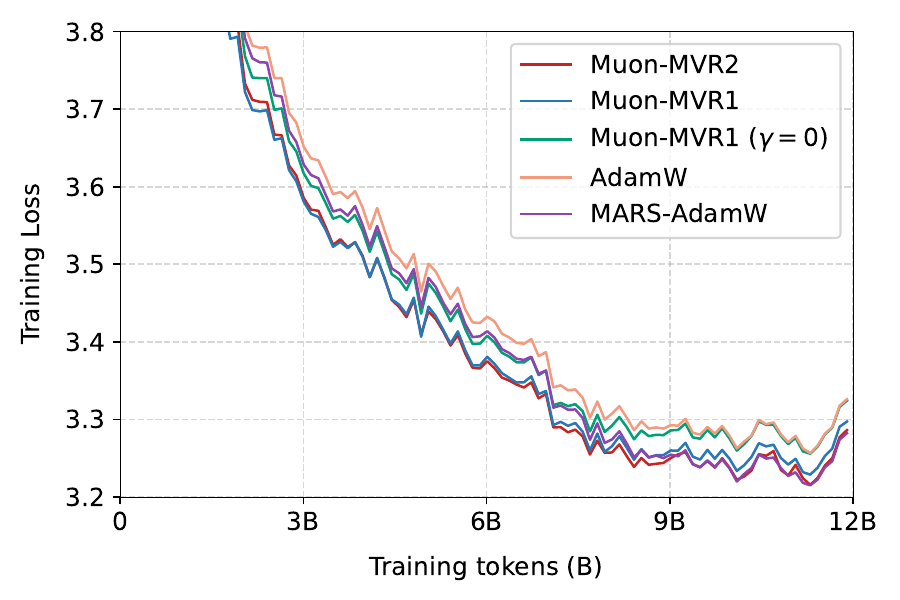}
    \end{minipage}
    }
    \subfloat[C4: validation loss]{\label{fig:llama_val_loss}%
    \begin{minipage}[t]{0.31\textwidth}
        \centering
        \includegraphics[width=\linewidth]{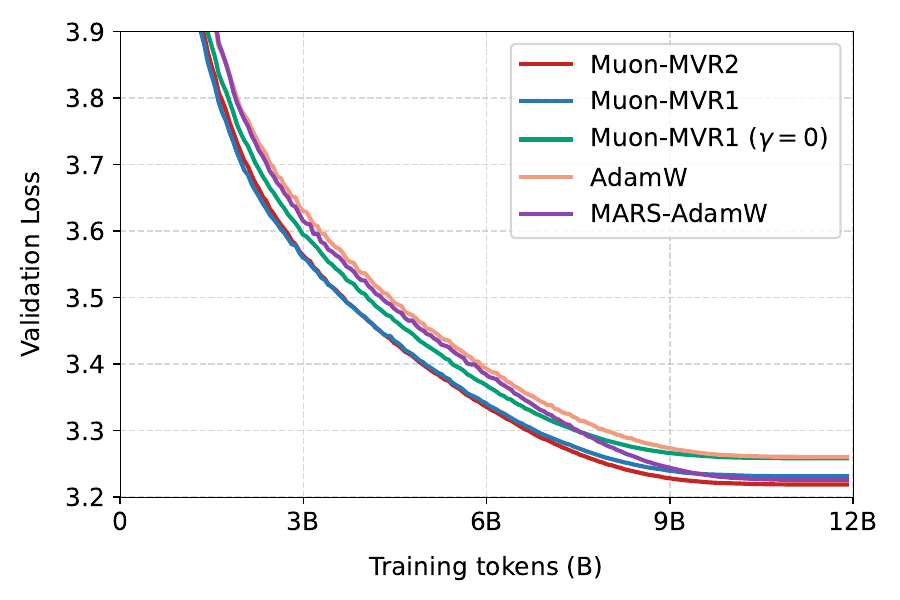}
    \end{minipage}
    }
    \subfloat[C4: validation loss vs. time]{\label{fig:llama_time}%
    \begin{minipage}[t]{0.31\textwidth}
        \centering
        \includegraphics[width=\linewidth]{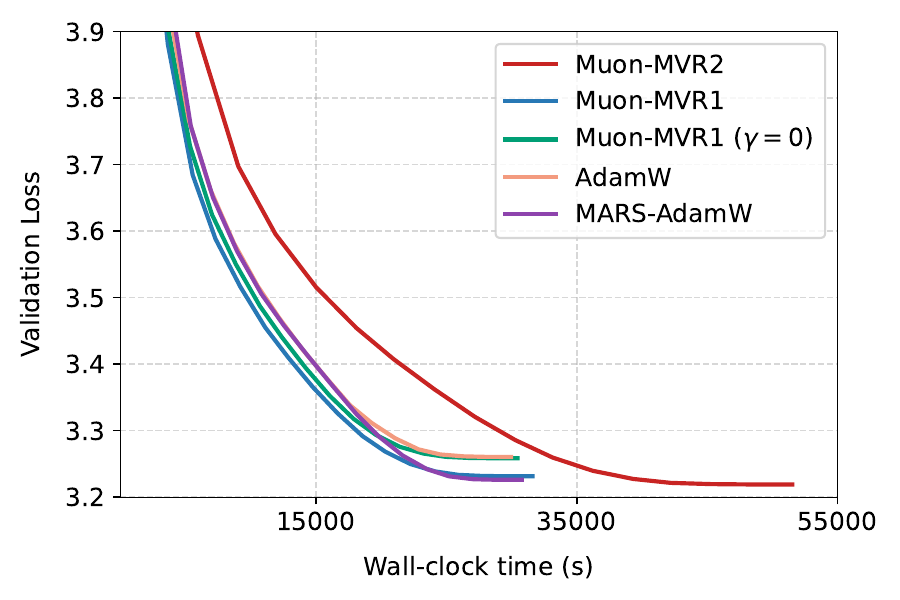}
    \end{minipage}
    }
    \caption{Training dynamics of Muon variants vs. baseline optimizers. Row 1: training/test accuracy and loss vs. epochs; test accuracy vs. wall-clock time for ResNet-18 on CIFAR-10. Row 2: Training and validation loss vs. steps, and validation loss vs. wall-clock time for LLaMA2-130M on C4.}
    \label{fig:training_dynamics}
\end{figure}

$\blacktriangleright$ \textbf{LLaMA2 on C4 Dataset.} 
We pre-train LLaMA2-130M~\cite{Touvron2023Llama2O} on C4 for 22k steps ($\sim$12B tokens), comparing Muon-MVR with AdamW and MARS-AdamW. Learning-rate and $\gamma$ search spaces are given in Appendix~\ref{appendix:exp}. Figures~\ref{fig:llama_val_loss} and~\ref{fig:llama_time} show lower per-step loss for MVR2 at roughly twice the time cost. Figure~\ref{Fig:FG4-1} shows a range of stable learning rates for each method, with Muon variants attaining lower validation loss at their best settings. Figure~\ref{Fig:FG4-2} further shows limited sensitivity to $\gamma$ near the best MVR2 learning rate.

Table~\ref{tab:additional-seed-results} reports additional evaluations over five CIFAR-10 seeds and three matched C4 seeds, with configurations distinguished in Appendix~\ref{appendix:additional-evaluations}. MVR1 and MVR2 have similar CIFAR-10 accuracy; MVR2 attains lower C4 perplexity but costs nearly twice as much per step. Under scaled diminishing schedules and an equal gradient-evaluation budget, MVR1 instead attains higher CIFAR-10 accuracy (Table~\ref{tab:theory-schedule-cifar}), illustrating finite-horizon sensitivity when transferring asymptotic schedules to deep-network training.

\begin{table}[!htb]
\centering
\small
\caption{Additional practical-recipe evaluations (mean $\pm$ standard deviation). Times are averages; dashes denote unreported combinations. Configurations are given in Appendix~\ref{appendix:additional-evaluations}.}
\label{tab:additional-seed-results}
\begin{tabular}{lrrrr}
\toprule
& \multicolumn{2}{c}{CIFAR-10 (5 seeds)} & \multicolumn{2}{c}{C4 (3 matched seeds)} \\
\cmidrule(lr){2-3}\cmidrule(lr){4-5}
Optimizer & Accuracy (\%) & s/epoch & Perplexity & s/step \\
\midrule
SGD & $93.01\pm0.03$ & 6.5519 & --- & --- \\
AdamW & $93.71\pm0.05$ & 6.6621 & $27.2777\pm0.11$ & 0.208 \\
MARS-AdamW & --- & --- & $26.2102\pm0.07$ & 0.228 \\
Muon (EMA) & $94.59\pm0.08$ & 7.7050 & $26.1113\pm0.11$ & 0.258 \\
Muon-MVR1 & $94.85\pm0.05$ & 7.7601 & $25.1028\pm0.13$ & 0.267 \\
Muon-MVR2 & $94.87\pm0.03$ & 11.6623 & $24.9567\pm0.05$ & 0.5210 \\
\bottomrule
\end{tabular}
\end{table}

\begin{figure}[!htb]
	\centering
	\subfloat[C4 Dataset 12B (LLaMA2-130M)]{\label{Fig:FG4-1}%
	\begin{minipage}[h]{0.4\textwidth}
	\centering
    \includegraphics[width=\textwidth]{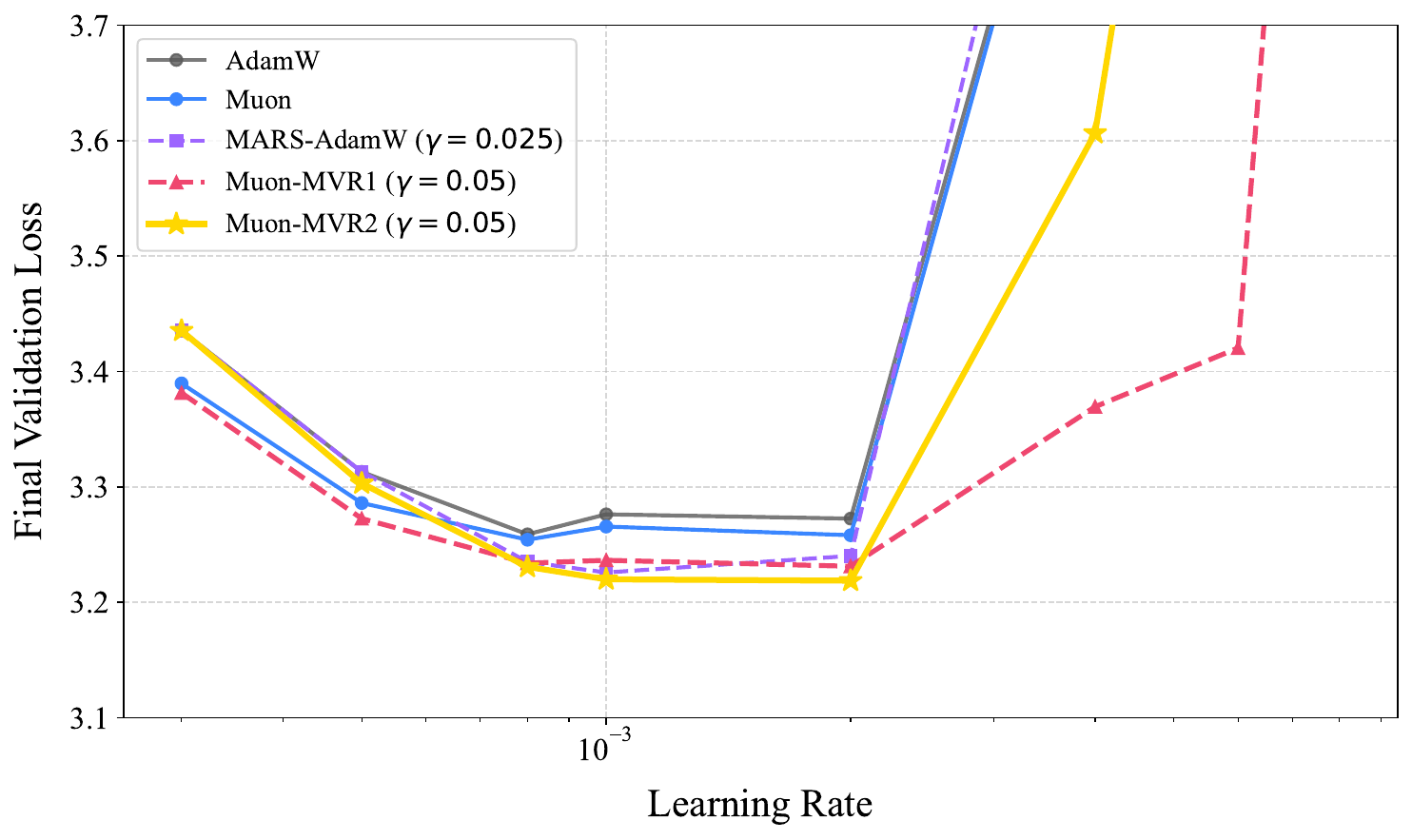}
    \end{minipage}
    }
    \quad
    \subfloat[Muon-MVR2 Sensitivity]{\label{Fig:FG4-2}%
    \begin{minipage}[h]{0.32\textwidth}
	\centering
    \includegraphics[width=\textwidth]{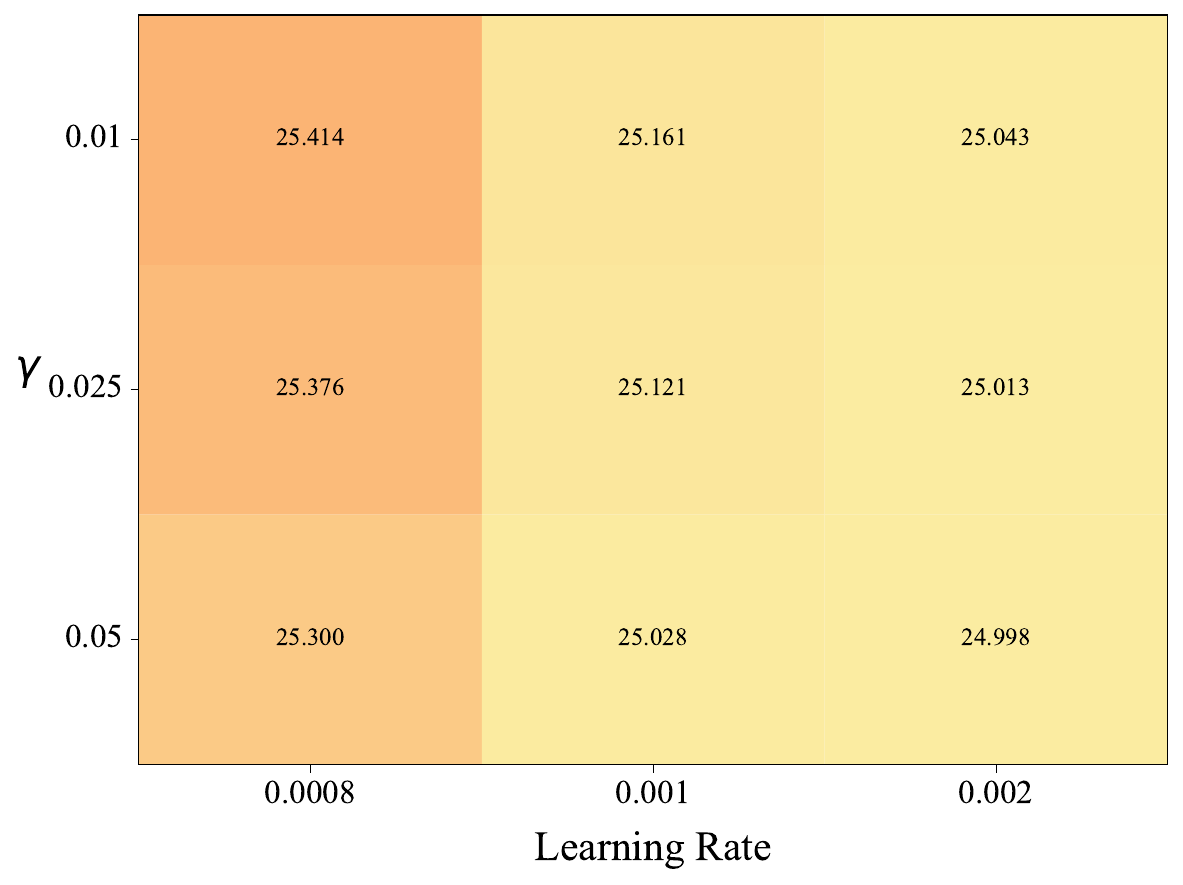}
    \end{minipage}
    }
    \caption{(a) Final validation loss with varying learning rates on C4 Dataset 12B; (b) Heatmap of the final validation perplexity of the Muon-MVR2 model for different $\gamma$ values around the optimal learning rate.}
    \label{fig:hyper_c4}
\end{figure}

Figure~\ref{fig:ns-wd-ablation} examines the finite-step constants in
Theorem~\ref{th-nonconvex-mvr2}. The alignment ratio and direction norm
characterize the NS recurrence across $K$, while the budget $\rho$ controls
the effect of weight decay on stationarity and tracking error.
\begin{figure}[!htb]
    \centering
    \includegraphics[width=0.85\textwidth]{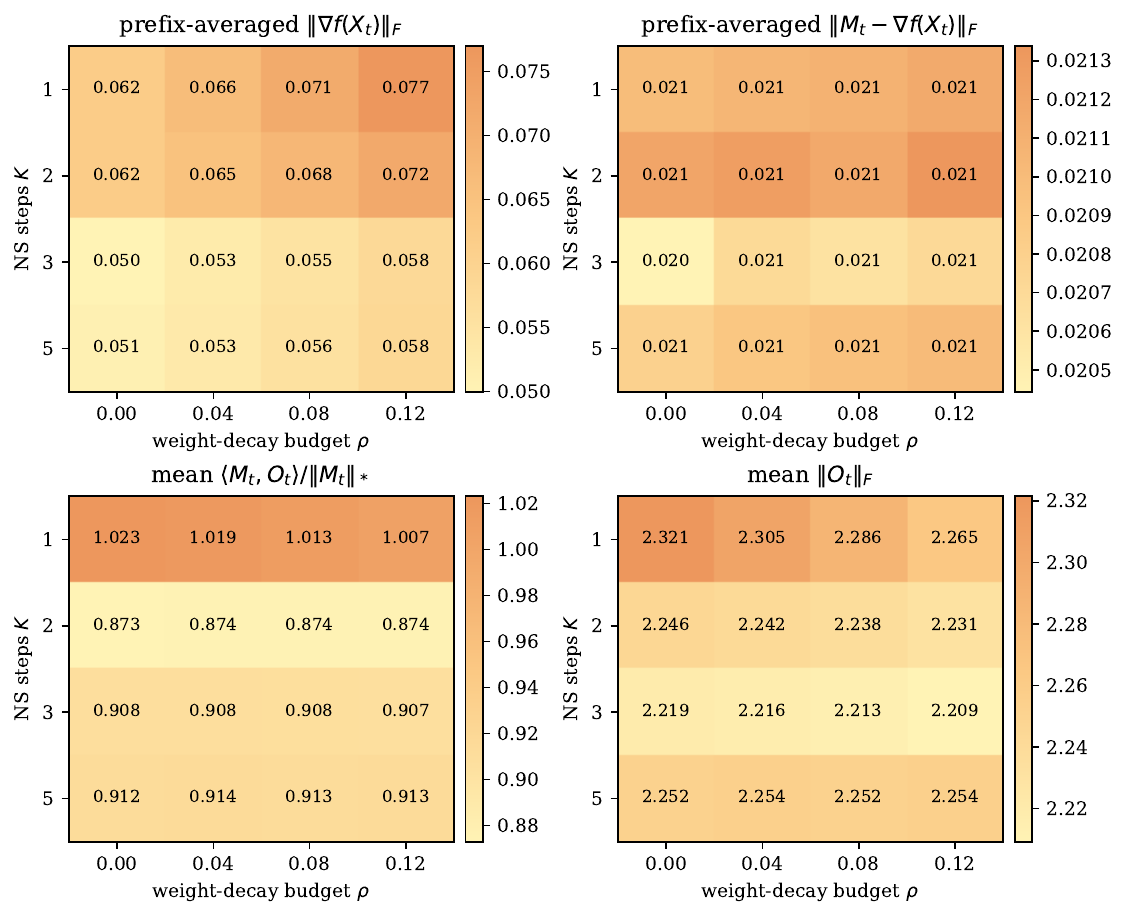}
    \caption{Actual-coefficient NS-step $\times$ decoupled-weight-decay ablation,
    averaged over 10 seeds. Panels show the prefix-averaged true-gradient norm,
    MVR tracking error, nuclear-norm alignment ratio, and direction Frobenius norm.
    Settings are given in Appendix~\ref{appendix:ns-wd-settings}.}
    \label{fig:ns-wd-ablation}
\end{figure}


\begin{remark}
MVR1 offers the more favorable trade-off when training time is the main constraint, while MVR2 achieves lower C4 perplexity per optimizer step at the cost of a second gradient evaluation. Appendix~\ref{appendix:additional-evaluations} also reports its memory overhead.
\end{remark}

\newpage
\section{Conclusion}
\label{section-conclusion}
We establish horizon-free convergence guarantees for Muon-MVR1 and Muon-MVR2. To the best of our knowledge, this is the first analysis of variance-reduced Muon to combine horizon-free schedules and constant mini-batches with finite practical fixed-coefficient NS steps and decoupled weight decay in unconstrained stochastic nonconvex optimization. Muon-MVR2 attains $\widetilde{\mathcal{O}}(T^{-1/3})$, matching the stochastic first-order lower bound under mean-square smoothness up to logarithmic factors~\citep{arjevani2023lower}. {For the practical NS coefficients, we establish bounds on momentum alignment and update norms that are uniform in the number of NS steps. These bounds yield the respective $\widetilde{\mathcal{O}}(T^{-1/4})$ and $\widetilde{\mathcal{O}}(T^{-1/3})$ rates for MVR1 and MVR2 with decoupled weight decay. The nuclear-norm alignment bound also improves with the NS step count.} Under the PL condition and an additional uniform gradient-bound assumption, we establish last-iterate objective-gap rates of $\mathcal{O}(T^{-1/4})$ and $\mathcal{O}(T^{-1/3})$ for Muon-MVR1 and Muon-MVR2, respectively, with finite NS steps and decoupled weight decay under the prescribed horizon-free schedules.
Our experimental results support the practical effectiveness of the proposed variance-reduced Muon variants on CIFAR-10 and C4.

\appendix

\section{Clarification of Theoretical Novelty}
\label{appendix:theoretical-novelty}
\noindent$\blacktriangleright$ \textbf{Relation to recent Muon-type convergence analyses.}
Recent Muon-type analyses address standard and Nesterov-style updates, constrained formulations, finite-step Newton--Schulz orthogonalization, and spectral interpretations~\citep{chen2025muon,Sato2025ConvergenceBA,kim2026convergence,nagashima2026improved,sfyraki2025lions,kovalev2025understanding}.
Related non-Euclidean and LMO methods also attain variance-reduced stochastic nonconvex rates~\citep{kovalev2025non,khirirat2025better}.
The initial version of this work already established a $\widetilde{\mathcal O}(T^{-1/3})$ stationarity guarantee for Muon-MVR2, then termed Muon-VR2, with constant mini-batches and a horizon-free decreasing-step-size schedule.
Subsequent work by \citet{qian2025muon} studies MVR within Gluon under layer-wise smoothness and general norms, including a decreasing-step-size variant, and explicitly discusses our earlier result.\footnote{Our initial version (\href{https://arxiv.org/abs/2509.15816v1}{arXiv:2509.15816v1}) was submitted on 19 September 2025; see Theorem~4.3 and Remark~4.2. The first version of \citet{qian2025muon} (\href{https://arxiv.org/abs/2512.16598v1}{arXiv:2512.16598v1}) was submitted on 18 December 2025.}
The present manuscript further extends the analysis to finite NS steps with the original Muon coefficients and decoupled weight decay for both MVR1 and MVR2.
For MVR2, the $\widetilde{\mathcal{O}}(T^{-1/3})$ stationarity rate corresponds to $\widetilde{\mathcal{O}}(\varepsilon^{-3})$ iteration complexity. The PL-based last-iterate results directly cover finite NS steps and decoupled weight decay. The exact branch without weight decay follows from the same proofs using its sharper geometric bounds.

\noindent$\blacktriangleright$ \textbf{Relation to STORM/MARS-style variance reduction.}
The MVR2 estimator builds on STORM/MARS-style~\cite{Cutkosky2019MomentumBasedVR,Yuan2024MARSUT} variance reduction, using a same-sample gradient difference to reduce stochastic noise. 
The novelty lies in showing that this mechanism remains effective when coupled with the polar-normalized Muon update.\footnote{
Our initial version (\href{https://arxiv.org/abs/2509.15816v1}{arXiv:2509.15816v1}, 19 September 2025)
and LiMuon~\citep{huang2026limuon}
(\href{https://arxiv.org/abs/2509.14562v1}{arXiv:2509.14562v1}, 18 September 2025)
were developed independently and concurrently.
}

This coupling is not a direct Euclidean substitution. 
Classical STORM-type analyses study updates of the form
\(\mathbf{x}_{t+1}=\mathbf{x}_t-\eta_t\mathbf{d}_t\),
where the recursive estimator itself is the descent direction. 
For the exact branch with $a=1$ and $\lambda_t=0$, Muon instead updates
\(\mathbf{X}_{t+1}=\mathbf{X}_t-\eta_t\mathbf{O}_t\), where
\(\mathbf{O}_t=\operatorname{msign}(\mathbf{M}_t)\).
Thus the proof must control the polar-normalized direction rather than $\mathbf{M}_t$ directly. 
The key identity
\(\langle \mathbf{M}_t,\mathbf{O}_t\rangle_F=\|\mathbf{M}_t\|_*\)
replaces the Euclidean quadratic term based on $\|\mathbf{M}_t\|_F^2$. 
Consequently, the analysis must jointly control the tracking error between $\mathbf{M}_t$ and $\nabla f(\mathbf{X}_t)$ and the effect of polar normalization in the descent inequality. 
This nuclear-norm/polar geometry is specific to Muon-style methods and is not covered by standard Euclidean STORM~\cite{Cutkosky2019MomentumBasedVR} or MARS~\cite{Yuan2024MARSUT} analyses.

The resulting stepsize balance is also Muon-specific. 
Writing $a_t:=1-\beta_t$, our MVR2 schedule uses $\eta_t=t^{-2/3}, a_t=t^{-2/3}$, whereas classical STORM analyses typically use an optimization stepsize of order $t^{-1/3}$ and an averaging coefficient of order $t^{-2/3}$. 
This difference reflects the normalized spectral geometry induced by polar orthogonalization. 
Under this balance, MVR2 improves the stochastic nonconvex rate from the Muon/MVR1 baseline $\widetilde{\mathcal{O}}(T^{-1/4})$ to $\widetilde{\mathcal{O}}(T^{-1/3})$.

For finite NS steps, the analysis uses the geometry of the actual fixed-coefficient NS recurrence. Its finite-step alignment and radius bounds give $\kappa_\rho=2a/3-\rho>0$ and $d_\rho=5a\sqrt n/4+\rho$, which enter both descent and tracking-error estimates. The weight-decay schedule controls $\lambda_t\|\mathbf X_t\|_F\le\rho$, preserving a positive alignment margin while bounding the additional displacement. Thus the finite NS geometry and decoupled weight decay are carried jointly through the convergence analysis.

\noindent$\blacktriangleright$ \textbf{Schedules, initialization, and stationarity measures.}
The stochastic guarantee in Corollary~4.2 of \citet{shen2025convergence} uses a single-sample batch and horizon-dependent constant parameters: asymptotically, $\eta=\Theta(T^{-3/4})$ and $1-\beta=\Theta(T^{-1/2})$. It controls the average nuclear norm of the unconstrained gradient. For variance-reduced stochastic Frank--Wolfe, Theorem~4 and Corollary~2 of \citet{sfyraki2025lions} use constant stepsize and averaging parameters of order $T^{-2/3}$, an initial batch of size $T^{1/3}$, and single-sample batches thereafter. Their stationarity measure is the Frank--Wolfe gap over a bounded constraint set, whose vanishing characterizes constrained KKT points. Muon-MVR2 instead uses $\eta_t=1-\beta_t=t^{-2/3}$ with a constant batch size from initialization onward, and controls $T^{-1}\sum_{t=1}^T\mathbb E\|\nabla f(\mathbf X_t)\|_F$ for every $T$. The schedule can therefore continue without specifying a terminal horizon. These distinctions concern the stated guarantees; the decreasing-step-size Gluon result discussed above also provides a horizon-free analysis.

The constant initial batch is supported by the weighted tracking estimate in Lemma~\ref{lem:practical-tracking-sum}. With $A_t=\mathbb E\|\mathbf M_t-\nabla f(\mathbf X_t)\|_F^2$ and $\mathcal V_\rho=4L^2d_\rho^2+\sigma^2$, initialization gives $A_1\le\sigma^2$, and
\[
\sum_{t=1}^T\sqrt{\eta_t}A_t\le4\sigma^2+8\mathcal V_\rho(1+\ln T),
\qquad
\frac1T\sum_{t=1}^T A_t\le\frac{4\sigma^2+8\mathcal V_\rho(1+\ln T)}{T^{2/3}}.
\]
The initial error enters as a fixed additive term and decays after averaging. This weighted summation accommodates the noise of a constant initial batch without imposing a small initial gradient or a horizon-dependent initialization budget.

\noindent$\blacktriangleright$ \textbf{Relation to normalized SGDM-style analyses.}
Muon is also related to normalized gradient and SGDM-style methods~\cite{cutkosky2020momentum,cutkosky2021high,Chen2023SymbolicDO}, but our setting differs in three ways. 
First, the geometry is matrix-valued: the exact polar map satisfies the nuclear-norm identity above, while the finite-NS map is controlled through its momentum-alignment and norm bounds. 
Second, the stepsize schedules are tied to this polar geometry: MVR1 uses $\eta_t=\Theta(t^{-3/4})$, while MVR2 uses $\eta_t=\Theta(t^{-2/3})$ together with $1-\beta_t=\Theta(t^{-2/3})$, yielding \textbf{horizon-free} anytime guarantees. 
Third, our PL analysis establishes last-iterate objective-gap bounds for both MVR variants with finite NS steps and decoupled weight decay under an additional uniform-gradient assumption. 
These arguments differ from standard Euclidean normalized-SGD analyses.

\noindent$\blacktriangleright$ \textbf{Scope of the contribution.}
We establish the compatibility of same-sample variance reduction with the polar geometry of Muon, achieving the first-order stochastic nonconvex rate $\widetilde{\mathcal{O}}(T^{-1/3})$ under a \textbf{horizon-free}, constant-mini-batch schedule. 
The analysis establishes this rate with finite fixed-coefficient NS steps and decoupled weight decay under the stated conditions, complementing existing exact-LMO and Muon convergence results.

\section{Synthetic Sanity Checks}
\label{appendix:synthetic-rates}
\begin{figure}[!htbp]
	\centering
	\subfloat[General nonconvex]{\label{Fig:FG5-1}%
	\begin{minipage}[h]{0.45\textwidth}
	\centering
    \includegraphics[width=\textwidth]{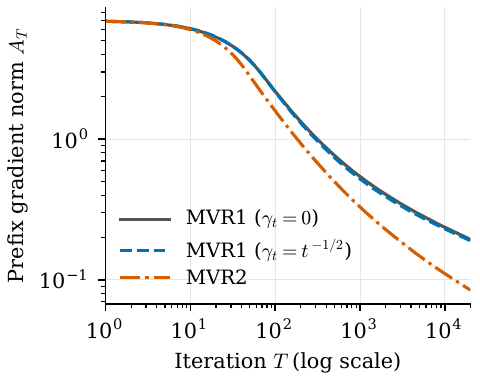}
    \end{minipage}
    }
    \subfloat[PL-condition]{\label{Fig:FG5-3}%
    \begin{minipage}[h]{0.45\textwidth}
	\centering
    \includegraphics[width=\textwidth]{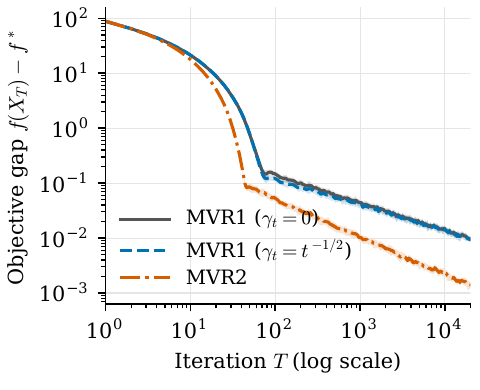}
    \end{minipage}
    }
    \caption{Synthetic sanity checks using exact $\operatorname{msign}$ updates and diminishing step sizes. 
    (a) A nonconvex periodic objective, reporting the prefix-averaged population gradient norm.
    (b) A PL quadratic with bounded additive gradient noise, reporting the last-iterate objective gap.
    Both panels use $20\times4$ matrices, the same initialization and algorithm schedules, and 20 paired seeds. Curves show seed means; shaded bands indicate one sample standard deviation.}
    \label{fig:toy_check}
\end{figure}
To complement the large-scale experiments, we include two synthetic problems with directly computable gradient or excess-risk metrics. These experiments illustrate convergence behavior under the horizon-free schedules in Section~\ref{section-conv}; they do not establish tightness of the worst-case upper bounds. In both cases, we optimize a single matrix parameter using the exact $\operatorname{msign}$ update in Algorithm~\ref{alg:muon}. Expectations are approximated by averaging over random seeds.
\paragraph{General nonconvex setting.}
For Figure~\ref{Fig:FG5-1}, let $\mathbf{X}\in\mathbb{R}^{20\times4}$, $d=80$, and
\[
f(\mathbf{X})=\sum_{i,j}(1-\cos X_{ij}),\qquad
f(\mathbf{X};\mathbf{Z})=f(\mathbf{X})+\sigma\langle\mathbf{Z},\sin\mathbf{X}\rangle_F,
\]
where $\sigma=2$ and the entries of $\mathbf{Z}$ are independent uniform signs divided by $\sqrt d$. Trigonometric functions act entrywise. This objective is nonconvex, has $f^*=0$, and is globally $1$-smooth. Its unbiased stochastic gradient is $g(\mathbf{X};\mathbf{Z})=\sin\mathbf{X}+\sigma\mathbf{Z}\odot\cos\mathbf{X}$, with variance at most $\sigma^2$. Moreover,
\[
\begin{aligned}
\mathbb{E}\|g(\mathbf{X};\mathbf{Z})-g(\mathbf{Y};\mathbf{Z})\|_F^2
&=\|\sin\mathbf{X}-\sin\mathbf{Y}\|_F^2
+\frac{\sigma^2}{d}\|\cos\mathbf{X}-\cos\mathbf{Y}\|_F^2\\
&\le\|\mathbf{X}-\mathbf{Y}\|_F^2,
\end{aligned}
\]
so Assumptions~\ref{ass:1}, \ref{ass:2}, \ref{ass:2.2}, and~\ref{ass:3} hold globally with $L=1$. We run both MVR1 schedules in Theorem~\ref{th-nonconvex-mvr1} and MVR2 in Theorem~\ref{th-nonconvex-mvr2} for 20,000 iterations with batch size one and 20 seeds, using $a=1$, $\lambda_t=0$, and $K_t=\infty$. All runs share a fixed initial matrix with entries sampled uniformly from $[-2.5,2.5]$; methods use paired noise streams. MVR2 reuses the current sample for both gradient evaluations. The plotted quantity is $A_T=T^{-1}\sum_{t=1}^T\mathbb{E}\|\sin\mathbf{X}_t\|_F$. The horizontal axis counts iterations, with $T$ gradient evaluations for MVR1 and $2T-1$ for MVR2. We report the observed decay without inferring a worst-case exponent from this instance.

\paragraph{PL setting.}
Figure~\ref{Fig:FG5-3} examines last-iterate objective-gap decay on
\[
f(\mathbf X)=\tfrac12\|\mathbf X\|_F^2,\qquad
f(\mathbf X;\mathbf Z)=\tfrac12\|\mathbf X\|_F^2+\sigma\langle\mathbf Z,\mathbf X\rangle_F,
\qquad g(\mathbf X;\mathbf Z)=\mathbf X+\sigma\mathbf Z.
\]
We retain the $20\times4$ matrix dimensions, $\sigma=2$, Rademacher noise distribution, fixed initialization, batch size, 20 seeds, and 20,000-iteration budget from panel~(a). Both MVR1 schedules and the MVR2 schedule are unchanged, with stepsize prefactor one, $a=1$, $\lambda_t=0$, and exact $\operatorname{msign}$. Methods share each seed's noise stream, and MVR2 evaluates its two gradients using the same $\mathbf Z_t$. Here $f^*=0$, $L=\mu=1$, the gradient-noise variance equals $\sigma^2$, and $g(\mathbf X;\mathbf Z)-g(\mathbf Y;\mathbf Z)=\mathbf X-\mathbf Y$.

The bounded noise also yields a uniform trajectory bound. The momentum recursions give $\|\mathbf M_t-\mathbf X_t\|_F\le B:=2(1+\sqrt6)+\sqrt{12}$ for all three schedules. Since $\|\mathbf O_t\|_F\le2$ and $\langle\mathbf X_t,\mathbf O_t\rangle_F\ge\|\mathbf X_t\|_F-3B$, the update satisfies
\[
\|\mathbf X_{t+1}\|_F^2\le\|\mathbf X_t\|_F^2-2\eta_t(\|\mathbf X_t\|_F-3B)+4\eta_t^2.
\]
Thus $\|\nabla f(\mathbf X_t)\|_F=\|\mathbf X_t\|_F\le G:=\max\{\|\mathbf X_1\|_F,3B+4\}$ for every $t$; the stepsize prefactor one satisfies the restrictions for the exact-update specializations of Theorems~\ref{th-last-mvr1} and~\ref{th-last-mvr2} given in Appendices~\ref{proof:th-last-mvr1} and~\ref{proof:th-last-mvr2}, with this $G$.

Panel~(b) plots the mean of $f(\mathbf X_T)-f^*$ across seeds at each iteration, with one-standard-deviation bands. At $T=20{,}000$, the objective gaps are $(9.75\pm1.31)\times10^{-3}$ for MVR1 with $\gamma_t=0$, $(9.55\pm1.29)\times10^{-3}$ for the second MVR1 schedule, and $(1.41\pm0.21)\times10^{-3}$ for MVR2. MVR2 attains a smaller last-iterate gap on this instance under the shared iteration budget.

\clearpage
\begingroup
\setlength{\parindent}{0pt}
\setlength{\parskip}{0pt}
\definecolor{CoefficientNSBlue}{RGB}{30,81,142}
\definecolor{CoefficientWDOrange}{RGB}{170,77,18}
\newcommand{\coeffns}[1]{{\color{CoefficientNSBlue}#1}}
\newcommand{\coeffwdmark}[1]{{\color{CoefficientWDOrange}#1}}
\renewcommand{\arraystretch}{1.02}
\setlength{\tabcolsep}{4pt}
\setlength{\aboverulesep}{1pt}
\setlength{\belowrulesep}{1pt}
\captionsetup{skip=3pt,hypcap=false}
\newcommand{\coeffsetting}[2]{\multirow{2}{24mm}{\shortstack[l]{\textbf{#1}\\#2}}}
\newcommand{\coeffoneline}[1]{\mbox{$\displaystyle #1$}}
\newsavebox{\coeffcomparisonpage}
\newcommand{\coeffncone}[2]{%
  \coeffoneline{\widetilde{\mathcal O}\!\left(
  \frac{#1\bigl(\Delta_1+#2La^2n+\sigma^2/L\bigr)}{aT^{1/4}}\right)}}
\newcommand{\coeffnctwo}[2]{%
  \coeffoneline{\widetilde{\mathcal O}\!\left(
  \frac{#1\bigl(\Delta_1+#2La^2n+\sigma^2/L\bigr)}{aT^{1/3}}
  +\frac{La\sqrt{#2n}+\sigma}{T^{1/3}}\right)}}
\newcommand{\coeffplone}[3]{%
  \coeffoneline{\mathcal O\!\left(
  \left[\Delta_1+\frac{\sigma^2\eta^{2/3}}{L}\right]e^{-#3T^{1/4}}
  +\frac{#1G^2}{\mu^2a^2T^{1/4}}
  \left[\frac{#2La^2n}{\eta^{2/3}}+\frac{\sigma^2}{L\eta^{4/3}}\right]\right)}}
\newcommand{\coeffpltwo}[3]{%
  \coeffoneline{\mathcal O\!\left(
  \left[\Delta_1+\frac{\sigma^2\eta^{1/2}}{L}\right]e^{-#3T^{1/3}}
  +\frac{#1G^2}{\mu^2a^2T^{1/3}}
  \left[\frac{#2La^2n}{\eta^{1/2}}+\frac{\sigma^2}{L\eta^{3/2}}\right]\right)}}
\begin{lrbox}{\coeffcomparisonpage}
\begin{minipage}[t]{\textwidth}
\section{Comparison of Convergence Bounds}
\label{appendix:coefficient-comparison}
\fontsize{9}{10.5}\selectfont
\setlength{\abovedisplayskip}{4pt}
\setlength{\belowdisplayskip}{4pt}
\setlength{\abovedisplayshortskip}{3pt}
\setlength{\belowdisplayshortskip}{3pt}

\makeatletter\def\@captype{table}\makeatother
\caption{Nonconvex convergence.}
\label{tab:coefficient-nonconvex}
\vspace{3pt}
\noindent\makebox[\textwidth][c]{%
\resizebox{0.9\textwidth}{!}{%
\begin{tabular*}{\textwidth}{@{\extracolsep{\fill}}lcc@{}}
\toprule
Setting & Method & Bound on $\frac1T\sum_{t=1}^T\mathbb E\|\nabla f(\mathbf X_t)\|_F$ \\
\midrule
\coeffsetting{Exact}{No weight decay} & MVR1 & \coeffncone{}{} \\[0pt]
 & MVR2 & \coeffnctwo{}{} \\[0pt]
\midrule
\coeffsetting{Finite NS}{No weight decay} & MVR1 & \coeffncone{\coeffns{\chi_0}}{\coeffns{\psi_0}} \\[0pt]
 & MVR2 & \coeffnctwo{\coeffns{\chi_0}}{\coeffns{\psi_0}} \\[0pt]
\midrule
\coeffsetting{Exact}{Weight decay} & MVR1 & \coeffncone{\coeffwdmark{\chi_\rho}}{\coeffwdmark{\psi_\rho}} \\[0pt]
 & MVR2 & \coeffnctwo{\coeffwdmark{\chi_\rho}}{\coeffwdmark{\psi_\rho}} \\[0pt]
\midrule
\coeffsetting{Finite NS}{Weight decay} & MVR1 & \coeffncone{\coeffns{\chi}_{\coeffwdmark{\rho}}}{\coeffns{\psi}_{\coeffwdmark{\rho}}} \\[0pt]
 & MVR2 & \coeffnctwo{\coeffns{\chi}_{\coeffwdmark{\rho}}}{\coeffns{\psi}_{\coeffwdmark{\rho}}} \\[0pt]
\bottomrule
\end{tabular*}%
}}\par

\vspace{3pt}
\caption{PL last-iterate convergence.}
\label{tab:coefficient-pl}
\vspace{3pt}
\noindent\makebox[\textwidth][c]{%
\resizebox{0.9\textwidth}{!}{%
\begin{tabular*}{\textwidth}{@{\extracolsep{\fill}}lcc@{}}
\toprule
Setting & Method & Bound on $\mathbb E[f(\mathbf X_T)]-f^*$ \\
\midrule
\coeffsetting{Exact}{No weight decay} & MVR1 & \coeffplone{}{}{c} \\[0pt]
 & MVR2 & \coeffpltwo{}{}{c} \\[0pt]
\midrule
\coeffsetting{Finite NS}{No weight decay} & MVR1 & \coeffplone{\coeffns{\chi_0^2}}{\coeffns{\psi_0}}{\coeffns{c}} \\[0pt]
 & MVR2 & \coeffpltwo{\coeffns{\chi_0^2}}{\coeffns{\psi_0}}{\coeffns{c}} \\[0pt]
\midrule
\coeffsetting{Exact}{Weight decay} & MVR1 & \coeffplone{\coeffwdmark{\chi_\rho^2}}{\coeffwdmark{\psi_\rho}}{\coeffwdmark{c}} \\[0pt]
 & MVR2 & \coeffpltwo{\coeffwdmark{\chi_\rho^2}}{\coeffwdmark{\psi_\rho}}{\coeffwdmark{c}} \\[0pt]
\midrule
\coeffsetting{Finite NS}{Weight decay} & MVR1 & \coeffplone{\coeffns{\chi}_{\coeffwdmark{\rho}}^{\coeffns{2}}}{\coeffns{\psi}_{\coeffwdmark{\rho}}}{\coeffns{c}} \\[0pt]
 & MVR2 & \coeffpltwo{\coeffns{\chi}_{\coeffwdmark{\rho}}^{\coeffns{2}}}{\coeffns{\psi}_{\coeffwdmark{\rho}}}{\coeffns{c}} \\[0pt]
\bottomrule
\end{tabular*}%
}}\par

\vspace{5pt}
{\fontsize{8.5}{10.2}\selectfont
\textcolor{CoefficientNSBlue}{Blue: finite NS.}\quad
\textcolor{CoefficientWDOrange}{Orange: weight decay.}
Notation and assumptions follow Theorems~\ref{th-nonconvex-mvr1}, \ref{th-nonconvex-mvr2},
\ref{th-last-mvr1}, and~\ref{th-last-mvr2}, with $\Delta_1=f(\mathbf X_1)-f^*$. Define
\[
\chi_\rho:=\frac{a}{\kappa_\rho},\qquad
\psi_\rho:=\frac{d_\rho^2}{a^2n},\qquad
(\kappa_\rho,d_\rho)=
\begin{cases}
(a-\rho,\ a\sqrt n+\rho),&\text{exact},\\
(2a/3-\rho,\ 5a\sqrt n/4+\rho),&\text{finite NS}.
\end{cases}
\]
No weight decay means $\rho=0$; the exact baseline has $\chi_0=\psi_0=1$.
Use $\eta_t=\eta t^{-p}$, with $p=3/4$ for MVR1 and $p=2/3$ for MVR2.
Table~\ref{tab:coefficient-nonconvex} uses $\eta=1$.
For Table~\ref{tab:coefficient-pl}, set $\kappa=2\mu\kappa_\rho/G$;
MVR1 uses $c=4\kappa\eta$ and $0<\eta\le\min\{1,1/(4\kappa)\}$,
while MVR2 uses $c=3\kappa\eta$ and $0<\eta\le\min\{1,1/(8\kappa)\}$.
Weight decay follows Lemma~\ref{lem:practical-wd-control} with $\kappa_\rho>0$;
its schedules also apply when $\eta\le1$.
Both MVR1 parameterizations are covered. Finite NS uses fixed $K\ge1$;
the exact-with-weight-decay rows use the common proof with exact geometry.
$\mathcal O$ hides universal numerical constants; $\widetilde{\mathcal O}$ additionally hides
logarithms in $T$. The PL forms combine dominated terms under the stated step-size restrictions.
}
\end{minipage}
\end{lrbox}
\typeout{COEFFICIENT TABLE PAGE HEIGHT: \the\dimexpr\ht\coeffcomparisonpage+\dp\coeffcomparisonpage\relax}
\ifdim\dimexpr\ht\coeffcomparisonpage+\dp\coeffcomparisonpage\relax>\textheight
  \PackageError{coefficient-tables}{The two tables exceed one appendix page
    (\the\dimexpr\ht\coeffcomparisonpage+\dp\coeffcomparisonpage\relax\space > \the\textheight)}
    {Reduce spacing or condense the shared notes.}
\fi
\noindent\usebox{\coeffcomparisonpage}
\endgroup
\clearpage
\begingroup

\section{Actual-Coefficient Newton--Schulz Geometry}
\label{appendix:practical-ns-geometry}

This section proves the geometric bounds for the finite-step NS map defined
in Eqs.~\eqref{eq:practical-ns-iteration}--\eqref{eq:practical-ns-polynomial}
of Section~\ref{section:practical-muon-mvr}, with the fixed coefficients
$(b_1,b_2,b_3)$ specified there.

\begin{lemma}
\label{lem:practical-ns-geometry}
For every $\mathbf A\in\mathbb R^{m\times n}$ and integer $K\ge1$,
\begin{align}
\langle\mathbf A,\operatorname{NS}_K(\mathbf A)\rangle_F
&\ge\tfrac23\|\mathbf A\|_F,
\label{eq:practical-ns-alignment}\\
\|\operatorname{NS}_K(\mathbf A)\|_2&\le\tfrac54,
\qquad
\|\operatorname{NS}_K(\mathbf A)\|_F\le\tfrac54\sqrt n.
\label{eq:practical-ns-norm}
\end{align}
Moreover,
\begin{equation}
\langle\mathbf A,\operatorname{NS}_K(\mathbf A)\rangle_F
\ge c_{K,n}\|\mathbf A\|_*,
\qquad
c_{K,n}:=\max\left\{\frac{2}{3\sqrt n},
\frac{\frac43(3^K-\frac23)}{
\sqrt{3^{2K}+\frac{2n}{3}(3^K-\frac23)}+3^K-\frac43}
\right\}.
\label{eq:practical-ns-nuclear-strong}
\end{equation}
In particular,
\begin{equation}
c_{K,n}\ge\max\left\{\frac{2}{3\sqrt n},
\frac{4\cdot3^K}{n+6\cdot3^K}\right\}.
\label{eq:practical-ns-nuclear-simple}
\end{equation}
For $\mathbf A\ne0$, $n$ can be replaced by $\operatorname{rank}(\mathbf A)$.
For $K=0$ and for exact matrix-sign normalization, the coefficient
$2/3$ in Eq.~\eqref{eq:practical-ns-alignment} and the coefficient
$5/4$ in Eq.~\eqref{eq:practical-ns-norm} can both be replaced by $1$.
\end{lemma}

\begin{proof}
The case $\mathbf A=0$ is immediate. Let
$\mathbf A=\mathbf U\operatorname{diag}(\sigma_1,\ldots,\sigma_r)
\mathbf V^\top$ be a compact SVD, and set
$s_i=\sigma_i/\|\mathbf A\|_F$. The scalar NS map satisfies
\[
\begin{aligned}
\phi(s)&=b_1s+b_2s^3+b_3s^5=s h(s^2),\\
h(u)&=b_1+b_2u+b_3u^2
\ge b_1-\frac{b_2^2}{4b_3}>\frac58.
\end{aligned}
\]
Differentiating $\phi$ to locate its extrema gives
\[
\begin{aligned}
\phi'(s)&=b_1+3b_2s^2+5b_3s^4
=5b_3(s^2-s_-^2)(s^2-s_+^2),
\end{aligned}
\]
where the positive roots satisfy $s_-\in(0.554,0.555)$ and
$s_+\in(1.050136,1.050137)$. Evaluations of $\phi'(\sqrt u)$ at the
squared rational endpoints give these intervals.
The derivative has sign pattern $(+,-,+)$, so $s_-$ and $s_+$ are
a local maximum and minimum, respectively. Rational endpoint bounds give
\[
\begin{aligned}
\phi(s_-)&\le b_1(0.555)+b_2(0.554)^3+b_3(0.555)^5<\frac54,\\
\phi(s_+)&\ge b_1(1.050136)+b_2(1.050137)^3
 +b_3(1.050136)^5>0.6818>\frac23,\\
\phi(5/4)&=\frac{482959}{409600}\in(2/3,5/4),
\end{aligned}
\]
where the endpoint choices use $b_1,b_3>0$ and $b_2<0$.
These bounds on the extrema, together with $h>0$, imply
\[
\phi([0,5/4])\subseteq[0,5/4].
\]
To obtain a lower bound that can be iterated, we split the interval at $2/9$:
\[
\begin{aligned}
0\le s\le2/9&:\quad
\phi(s)\ge s\bigl(b_1+b_2(2/9)^2\bigr)\ge3s,\\
2/9\le s\le5/4&:\quad
\phi(s)\ge\min\{\phi(2/9),\phi(s_+),\phi(5/4)\}\ge2/3.
\end{aligned}
\]
Thus $q(s):=\min\{3s,2/3\}\le\phi(s)$ on $[0,5/4]$.
Iterating this lower bound gives
\begin{equation}
\begin{aligned}
\min\{3^Ks,2/3\}=q^{\circ K}(s)
\stackrel{(\circ)}{\le}\phi^{\circ K}(s)
\stackrel{(\star)}{\le}5/4
\quad(K\ge1,\ 0\le s\le1),
\end{aligned}
\label{eq:practical-ns-scalar-envelope}
\end{equation}
where $(\circ)$ follows by induction using monotonicity of $q$, and
$(\star)$ follows from invariance.

The NS iteration preserves the singular vectors. With $z_i=\phi^{\circ K}(s_i)$,
the map satisfies $\operatorname{NS}_K(\mathbf A)=\mathbf U\operatorname{diag}(z_i)\mathbf V^\top$,
so applying the scalar bounds to each singular value gives
\[
\begin{aligned}
\|\operatorname{NS}_K(\mathbf A)\|_2&=\max_i z_i\le\tfrac54,\qquad
\|\operatorname{NS}_K(\mathbf A)\|_F^2=\sum_i z_i^2\le\tfrac{25}{16}r\le\tfrac{25}{16}n.
\end{aligned}
\]
For the alignment bound, expanding the inner product in the same singular vectors yields
\[
\begin{aligned}
\langle\mathbf A,\operatorname{NS}_K(\mathbf A)\rangle_F
&=\|\mathbf A\|_F\sum_i s_i z_i
\stackrel{(\circ)}{\ge}\|\mathbf A\|_F\sum_i s_i\min\{3^Ks_i,2/3\}\\
&\stackrel{(\star)}{\ge}\tfrac23\|\mathbf A\|_F\sum_i s_i^2
=\tfrac23\|\mathbf A\|_F,
\end{aligned}
\]
where $(\circ)$ uses Eq.~\eqref{eq:practical-ns-scalar-envelope}, and
$(\star)$ uses $0<s_i\le1$ and $\sum_i s_i^2=1$.

For the nuclear-norm bound, we pad $(s_i)$ with zeros to $n$ entries and set
$S=\sum_i s_i$, $Z=\sum_i s_i\phi^{\circ K}(s_i)$, $b=2/3$, and $g=3^K$.
Then $Z\ge b$, $S\le\sqrt n$, and
$\langle\mathbf A,\operatorname{NS}_K(\mathbf A)\rangle_F/\|\mathbf A\|_*
=Z/S\ge b/\sqrt n$.
For $c\in(0,b)$, put $d=g-b+c>0$. Completing the square yields
\[
\begin{aligned}
s\min\{gs,b\}-cs-(b-c)s^2
&=\begin{cases}
d(s-c/(2d))^2-c^2/(4d),&gs\le b,\\
(b-c)s(1-s),&gs>b,
\end{cases}\\
&\ge-\frac{c^2}{4d}\qquad(0\le s\le1).
\end{aligned}
\]
Applying this scalar inequality to each $s_i$ and summing over $i$, we obtain
\[
\begin{aligned}
Z-cS&\stackrel{(\circ)}{\ge}(b-c)\sum_i s_i^2-\frac{nc^2}{4d}
=b-c-\frac{nc^2}{4d},
\end{aligned}
\]
where $(\circ)$ sums the scalar bound using
Eq.~\eqref{eq:practical-ns-scalar-envelope}.
To ensure $Z\ge cS$, it suffices to choose $c$ such that
\[
b-c\ge\frac{nc^2}{4(g-b+c)}.
\]
Since $g-b+c>0$, this condition is equivalent to the quadratic inequality
\[
(n+4)c^2+4(g-2b)c-4b(g-b)\le0.
\]
Taking the positive root gives the admissible choice
\[
\begin{aligned}
c=c_*&:=\frac{2b(g-b)}{\sqrt{g^2+nb(g-b)}+g-2b}\in(0,b),
\end{aligned}
\]
which proves Eq.~\eqref{eq:practical-ns-nuclear-strong}.
A simpler bound follows by taking
$c_0=b/(1+nb/(4g))=4\cdot3^K/(n+6\cdot3^K)$, which satisfies
\[
\frac{nc_0^2}{4(g-b+c_0)}\le\frac{nbc_0}{4g}=b-c_0,
\]
since $g-b+c_0-gc_0/b=(g-b)(1-c_0/b)\ge0$; this proves
Eq.~\eqref{eq:practical-ns-nuclear-simple}. Using only the $r$ positive
singular values gives the rank version.
For $K=0$ and the exact map, direct evaluation gives
\[
\begin{aligned}
\|\operatorname{NS}_0(\mathbf A)\|_2
&\le\|\operatorname{NS}_0(\mathbf A)\|_F=1,
&\langle\mathbf A,\operatorname{NS}_0(\mathbf A)\rangle_F&=\|\mathbf A\|_F,\\
\|\operatorname{msign}(\mathbf A)\|_2&=1,
&\langle\mathbf A,\operatorname{msign}(\mathbf A)\rangle_F&=\|\mathbf A\|_*\ge\|\mathbf A\|_F.
\end{aligned}
\]
\end{proof}

By Eq.~\eqref{eq:practical-ns-nuclear-simple},
$K\ge\lceil\log_3 n\rceil$ gives $c_{K,n}\ge4/7$.
The explicit lower bound $4\cdot3^K/(n+6\cdot3^K)$ is
non-decreasing in $K$.
Since the geometric estimates hold for every momentum matrix,
they also apply when a finite step count $K_t\ge1$ is chosen
as a measurable function of $\mathbf M_t$.

\section{Lemmas for Decoupled Weight Decay}
\label{appendix:practical-wd-control}

We use $c$ for the momentum-alignment coefficient and $r\sqrt n$
for the bound on the update norm.
With $\mathbf O_t=\mathcal Q(\mathbf M_t)$, write
\begin{equation}
\begin{aligned}
\langle\mathbf M_t,\mathbf O_t\rangle_F&\ge c\|\mathbf M_t\|_F,
\qquad \|\mathbf O_t\|_F\le r\sqrt n,\\
\kappa_\rho&:=ac-\rho>0,\qquad d_\rho:=ar\sqrt n+\rho.
\end{aligned}
\label{eq:practical-kappa-d}
\end{equation}
For every fixed finite $K\ge1$,
Lemma~\ref{lem:practical-ns-geometry} gives $c=2/3$ and $r=5/4$.
Exact matrix-sign normalization gives $c=r=1$.

\begin{lemma}
\label{lem:practical-wd-control}
\label{lem:practical-mvr1-wd-control}
Suppose $\eta_t=t^{-p}$, where $p\in\{3/4,2/3\}$, $0\le\rho<ac$, and
\begin{equation}
0\le\lambda_t\le\lambda_0t^{p-1},\qquad
\lambda_0:=\frac{\rho}{\|\mathbf X_1\|_F+ar\sqrt n/(1-p)}.
\label{eq:practical-wd-schedule-app}
\end{equation}
Then the update in Eq.~\eqref{eq:practical-model-update} satisfies, along
every realization and for every $t\ge1$,
\begin{equation}
\lambda_t\|\mathbf X_t\|_F\le\rho,\qquad
\|\mathbf X_{t+1}-\mathbf X_t\|_F\le d_\rho\eta_t.
\label{eq:practical-wd-step-control}
\end{equation}
\end{lemma}

\begin{proof}
Since $c\le r\sqrt n$, the assumed schedule gives
$0\le\eta_t\lambda_t\le\lambda_0/t<1$.
Applying the norm bound to the update gives
\[
\begin{aligned}
\|\mathbf X_{t+1}\|_F
&\le(1-\eta_t\lambda_t)\|\mathbf X_t\|_F+a\eta_t\|\mathbf O_t\|_F
\le\|\mathbf X_t\|_F+ar\sqrt n\,\eta_t.
\end{aligned}
\]
Iterating this inequality from the initial point, we obtain
\[
\begin{aligned}
\|\mathbf X_t\|_F
&\le\|\mathbf X_1\|_F+ar\sqrt n\sum_{s=1}^{t-1}s^{-p}
\stackrel{(\circ)}{\le}\|\mathbf X_1\|_F+\frac{ar\sqrt n}{1-p}t^{1-p},
\end{aligned}
\]
where $(\circ)$ follows from
$\sum_{s=1}^{t-1}s^{-p}\le1+\int_1^t x^{-p}\,dx\le t^{1-p}/(1-p)$.
Multiplying by the prescribed upper bound on $\lambda_t$ now gives
\[
\begin{aligned}
\lambda_t\|\mathbf X_t\|_F
&\le\lambda_0\left(t^{p-1}\|\mathbf X_1\|_F+\frac{ar\sqrt n}{1-p}\right)
\le\rho.
\end{aligned}
\]
Substituting this bound into the update controls the displacement:
\[
\begin{aligned}
\|\mathbf X_{t+1}-\mathbf X_t\|_F
&=\eta_t\|a\mathbf O_t+\lambda_t\mathbf X_t\|_F
\le\eta_t(ar\sqrt n+\rho)=d_\rho\eta_t.
\end{aligned}
\]
\end{proof}

For $(c,r)=(2/3,5/4)$, the choices $p=3/4$ and $p=2/3$ give exactly
the weight-decay schedules in Theorems~\ref{th-nonconvex-mvr1}
and~\ref{th-nonconvex-mvr2}, respectively.

\section{Lemmas for Theorem \ref{th-nonconvex-mvr1}}
\label{appendix:practical-mvr1-tracking}

We use the constants in Eq.~\eqref{eq:practical-kappa-d} and write
$h_t:=t^{-1/2}=\eta_t^{2/3}$,
$\mathcal V_{1,\rho}:=2\sqrt2L^2d_\rho^2+\sigma^2$, and
$\mathcal R_{1,\rho}:=La^2(c^2+r^2n)/2+Ld_\rho^2/2$.
Let $\mathbf G_t:=\nabla f(\mathbf X_t;\xi_t)$ and
$\mathcal F_t:=\sigma(\xi_1,\ldots,\xi_t)$.
For $\gamma_t=0$, set
$\mathbf S_t^{(1)}:=\mathbf M_t-\nabla f(\mathbf X_t)$ and
$A_t^{(1)}:=\mathbb E\|\mathbf S_t^{(1)}\|_F^2$.

\begin{lemma}
\label{lem:practical-mvr1-ema}
\label{lemma:a2}
Under Assumptions~\ref{ass:2} and~\ref{ass:3}, the MVR1 estimator with
$\gamma_t=0$ and a deterministic sequence $0\le\beta_t<1$ satisfies
\begin{equation}
\begin{aligned}
A_{t+1}^{(1)}
\le{}&\beta_{t+1}A_t^{(1)}
+\frac{\beta_{t+1}^2L^2}{1-\beta_{t+1}}
  \mathbb E\|\mathbf X_{t+1}-\mathbf X_t\|_F^2
+(1-\beta_{t+1})^2\sigma^2.
\end{aligned}
\label{eq:practical-mvr1-ema-general}
\end{equation}
If $\eta_t=t^{-3/4}$, $\beta_t=1-h_t$, and
Lemma~\ref{lem:practical-wd-control} applies with $p=3/4$, then
\begin{align}
A_{t+1}^{(1)}&\le(1-h_{t+1})A_t^{(1)}
+\mathcal V_{1,\rho}h_{t+1}^2,
\label{eq:practical-mvr1-ema-recursion}\\
\sum_{t=1}^T h_tA_t^{(1)}&\le2\sigma^2
+2\mathcal V_{1,\rho}(1+\ln T),\qquad T\ge1.
\label{eq:practical-mvr1-ema-sum}
\end{align}
\end{lemma}

\begin{proof}
Set $\mathbf R_{t+1}:=\mathbf G_{t+1}-\nabla f(\mathbf X_{t+1})$ and
$\mathbf D_t:=\nabla f(\mathbf X_t)-\nabla f(\mathbf X_{t+1})$.
The EMA recursion gives
$\mathbf S_{t+1}^{(1)}=\beta_{t+1}(\mathbf S_t^{(1)}+\mathbf D_t)
+(1-\beta_{t+1})\mathbf R_{t+1}$.
Taking the squared norm and conditioning on $\mathcal F_t$ gives
\[
\begin{aligned}
\mathbb E[\|\mathbf S_{t+1}^{(1)}\|_F^2\mid\mathcal F_t]
&\stackrel{(\circ)}{=}\beta_{t+1}^2\|\mathbf S_t^{(1)}+\mathbf D_t\|_F^2
+(1-\beta_{t+1})^2\mathbb E[\|\mathbf R_{t+1}\|_F^2\mid\mathcal F_t]\\
&\stackrel{(\star)}{\le}\beta_{t+1}\|\mathbf S_t^{(1)}\|_F^2
+\frac{\beta_{t+1}^2}{1-\beta_{t+1}}\|\mathbf D_t\|_F^2
+(1-\beta_{t+1})^2\sigma^2\\
&\le\beta_{t+1}\|\mathbf S_t^{(1)}\|_F^2
+\frac{\beta_{t+1}^2L^2}{1-\beta_{t+1}}\|\mathbf X_{t+1}-\mathbf X_t\|_F^2
+(1-\beta_{t+1})^2\sigma^2,
\end{aligned}
\]
where $(\circ)$ uses $\mathbb E[\mathbf R_{t+1}\mid\mathcal F_t]=0$;
$(\star)$ follows from Young's inequality and Assumption~\ref{ass:3},
with $\beta_{t+1}=0$ immediate; and the last step uses
Assumption~\ref{ass:2}. Taking total expectation proves
Eq.~\eqref{eq:practical-mvr1-ema-general}.

For the stated schedules,
$\eta_t^2/h_{t+1}=((t+1)/t)^{3/2}h_{t+1}^2
\le2\sqrt2h_{t+1}^2$. Substituting
Eq.~\eqref{eq:practical-wd-step-control} into
Eq.~\eqref{eq:practical-mvr1-ema-general} gives
Eq.~\eqref{eq:practical-mvr1-ema-recursion}.
To sum the tracking errors, first use $h_t/h_{t+1}\le2$ to write
\[
\begin{aligned}
h_tA_t^{(1)}
&\stackrel{(\circ)}{\le}2\bigl(A_t^{(1)}-A_{t+1}^{(1)}
+\mathcal V_{1,\rho}h_{t+1}^2\bigr),
\end{aligned}
\]
where $(\circ)$ rearranges Eq.~\eqref{eq:practical-mvr1-ema-recursion}
and uses $h_t\le2h_{t+1}$, as in the sequence estimate of
\citet[Lemma C.1]{chang2025mgup}.
Summing from $t=1$ to $T$ makes the error differences telescope and gives
\[
\begin{aligned}
\sum_{t=1}^T h_tA_t^{(1)}
&\le2A_1^{(1)}+2\mathcal V_{1,\rho}\sum_{t=1}^T\frac1{t+1}
\stackrel{(\star)}{\le}2\sigma^2+2\mathcal V_{1,\rho}(1+\ln T),
\end{aligned}
\]
where $(\star)$ uses
$\mathbf M_1=\mathbf G_1$, Assumption~\ref{ass:3}, and the harmonic-sum bound.
\end{proof}

For $\gamma_t=h_t$ and $\beta_t=1-h_{t+1}$, set
$\mathbf C_1=\mathbf M_1=\mathbf G_1$. Since $\gamma_t=1-\beta_{t-1}$,
Eq.~\eqref{o2} has the equivalent form, for $t\ge2$,
\begin{equation}
\begin{aligned}
\mathbf C_t&=\beta_{t-1}\mathbf C_{t-1}+(1-\beta_{t-1})\mathbf G_t,\\
\mathbf M_t&=\beta_t\mathbf C_t+(1-\beta_t)\mathbf G_t.
\end{aligned}
\label{eq:practical-mvr1-buffer}
\end{equation}
Define $\mathbf U_t:=\mathbf C_t-\nabla f(\mathbf X_t)$,
$B_t:=\mathbb E\|\mathbf U_t\|_F^2$,
$\mathbf S_t^{(2)}:=\mathbf M_t-\nabla f(\mathbf X_t)$, and
$A_t^{(2)}:=\mathbb E\|\mathbf S_t^{(2)}\|_F^2$.

\begin{lemma}
\label{lem:practical-mvr1-buffer}
Under Assumptions~\ref{ass:2} and~\ref{ass:3}, with $\eta_t=t^{-3/4}$,
$\gamma_t=h_t$, $\beta_t=1-h_{t+1}$, and the weight-decay schedule in
Lemma~\ref{lem:practical-wd-control} for $p=3/4$,
\begin{align}
B_{t+1}&\le(1-h_{t+1})B_t+\mathcal V_{1,\rho}h_{t+1}^2,
\label{eq:practical-mvr1-buffer-recursion}\\
A_t^{(2)}&\le2B_t+2h_{t+1}^2\sigma^2,
\label{eq:practical-mvr1-buffer-transfer}\\
\sum_{t=1}^T h_tA_t^{(2)}&\le8\sigma^2+4\mathcal V_{1,\rho}(1+\ln T),
\qquad T\ge1.
\label{eq:practical-mvr1-buffer-sum}
\end{align}
\end{lemma}

\begin{proof}
Applying Eq.~\eqref{eq:practical-mvr1-ema-general} to the
$\mathbf C_{t+1}$ recursion gives
\[
\begin{aligned}
B_{t+1}
&\le\beta_tB_t+\frac{\beta_t^2L^2}{1-\beta_t}
\mathbb E\|\mathbf X_{t+1}-\mathbf X_t\|_F^2+(1-\beta_t)^2\sigma^2\\
&\stackrel{(\circ)}{\le}(1-h_{t+1})B_t+
\left(2\sqrt2L^2d_\rho^2+\sigma^2\right)h_{t+1}^2,
\end{aligned}
\]
where $(\circ)$ uses Eq.~\eqref{eq:practical-wd-step-control},
$1-\beta_t=h_{t+1}$, and $\eta_t^2/h_{t+1}\le2\sqrt2h_{t+1}^2$.
Applying the same weighted summation to this recursion gives
\[
\begin{aligned}
\sum_{t=1}^T h_tB_t
&\stackrel{(\star)}{\le}2\sigma^2+2\mathcal V_{1,\rho}(1+\ln T),
\end{aligned}
\]
where $(\star)$ follows by the summation in Lemma~\ref{lem:practical-mvr1-ema},
since $B_1\le\sigma^2$.

To pass from the buffer error to the momentum error, note that
Eq.~\eqref{eq:practical-mvr1-buffer} implies
$\mathbf S_t^{(2)}=\beta_t\mathbf U_t+h_{t+1}(\mathbf G_t-\nabla f(\mathbf X_t))$,
also for $t=1$. Taking expectations of the squared norms gives
\[
\begin{aligned}
A_t^{(2)}
&\le2\beta_t^2B_t+2h_{t+1}^2
\mathbb E\|\mathbf G_t-\nabla f(\mathbf X_t)\|_F^2
\le2B_t+2h_{t+1}^2\sigma^2.
\end{aligned}
\]
Multiplying both sides by $h_t$ and summing from $t=1$ to $T$, we obtain
\[
\begin{aligned}
\sum_{t=1}^T h_tA_t^{(2)}
&\le2\sum_{t=1}^T h_tB_t+2\sigma^2\sum_{t=1}^T h_th_{t+1}^2\\
&\stackrel{(\circ)}{\le}4\sigma^2+4\mathcal V_{1,\rho}(1+\ln T)+4\sigma^2,
\end{aligned}
\]
where $(\circ)$ uses the preceding bound on $\sum_t h_tB_t$ and
$h_th_{t+1}^2\le2(h_t-h_{t+1})$, whose sum is at most $2$.
\end{proof}

\begin{lemma}
\label{lem:practical-mvr1-descent}
Under Assumption~\ref{ass:2}, suppose Eq.~\eqref{eq:practical-kappa-d} holds,
$\lambda_t\|\mathbf X_t\|_F\le\rho$, and $0<\eta_t\le1$.
For either MVR1 parameterization, let
$\mathcal E_t:=\mathbb E\|\mathbf M_t-\nabla f(\mathbf X_t)\|_F^2$. Then
\begin{equation}
\begin{aligned}
\mathbb E[f(\mathbf X_{t+1})]
\le{}&\mathbb E[f(\mathbf X_t)]
-\kappa_\rho\eta_t\mathbb E\|\nabla f(\mathbf X_t)\|_F
+\frac{\eta_t^{2/3}}{L}\mathcal E_t
+\mathcal R_{1,\rho}\eta_t^{4/3}.
\end{aligned}
\label{eq:practical-mvr1-descent}
\end{equation}
\end{lemma}

\begin{proof}
Write $\mathbf S_t:=\mathbf M_t-\nabla f(\mathbf X_t)$.
The update and $L$-smoothness give
\begin{equation}
\begin{aligned}
f(\mathbf X_{t+1})
&\le f(\mathbf X_t)-a\eta_t\langle\nabla f(\mathbf X_t),\mathbf O_t\rangle_F
-\lambda_t\eta_t\langle\nabla f(\mathbf X_t),\mathbf X_t\rangle_F
+\frac{L\eta_t^2}{2}\|a\mathbf O_t+\lambda_t\mathbf X_t\|_F^2\\
&\stackrel{(\circ)}{\le}f(\mathbf X_t)-ac\eta_t\|\mathbf M_t\|_F
+ar\sqrt n\,\eta_t\|\mathbf S_t\|_F
+\rho\eta_t\|\nabla f(\mathbf X_t)\|_F+\frac{Ld_\rho^2}{2}\eta_t^2\\
&\stackrel{(\star)}{\le}f(\mathbf X_t)-\kappa_\rho\eta_t\|\nabla f(\mathbf X_t)\|_F
+ac\eta_t\|\mathbf S_t\|_F+ar\sqrt n\,\eta_t\|\mathbf S_t\|_F
+\frac{Ld_\rho^2}{2}\eta_t^2\\
&\stackrel{(\bullet)}{\le}f(\mathbf X_t)-\kappa_\rho\eta_t\|\nabla f(\mathbf X_t)\|_F
+\frac{\eta_t^{2/3}}{L}\|\mathbf S_t\|_F^2
+\frac{La^2(c^2+r^2n)}2\eta_t^{4/3}+\frac{Ld_\rho^2}{2}\eta_t^2,
\end{aligned}
\label{eq:practical-mvr1-descent-start}
\end{equation}
where $(\circ)$ uses Eq.~\eqref{eq:practical-kappa-d} and
$\lambda_t\|\mathbf X_t\|_F\le\rho$; $(\star)$ uses
$\|\mathbf M_t\|_F\ge\|\nabla f(\mathbf X_t)\|_F-\|\mathbf S_t\|_F$;
and $(\bullet)$ applies Young's inequality to each tracking term with
quadratic coefficient $\eta_t^{2/3}/(2L)$.
Taking expectation and using $\eta_t^2\le\eta_t^{4/3}$ proves the claim.
\end{proof}

\section{Proof of Theorem \ref{th-nonconvex-mvr1}}
\label{proof:th-nonconvex-mvr1}
\label{appendix:practical-mvr1-proof}
\label{proof:th-nonconvex-mvr1-practical}

\begin{proof}
Set $\eta_t=t^{-3/4}$ and $h_t=t^{-1/2}$, and let
$\Gamma_t:=h_t\mathcal E_t/L+\mathcal R_{1,\rho}h_t^2$.
Summing the descent inequality in Lemma~\ref{lem:practical-mvr1-descent}
and using the lower bound on $f$, we obtain
\begin{equation}
\begin{aligned}
\kappa_\rho T^{1/4}\frac1T\sum_{t=1}^T\mathbb E\|\nabla f(\mathbf X_t)\|_F
&\stackrel{(\circ)}{\le}\kappa_\rho\sum_{t=1}^T\eta_t
\mathbb E\|\nabla f(\mathbf X_t)\|_F\\
&\le\sum_{t=1}^T\bigl(\mathbb E f(\mathbf X_t)-\mathbb E f(\mathbf X_{t+1})\bigr)
+\sum_{t=1}^T\Gamma_t\\
&\stackrel{(\star)}{\le}f(\mathbf X_1)-f^*+\sum_{t=1}^T\Gamma_t,
\end{aligned}
\label{eq:practical-mvr1-telescoping}
\end{equation}
where $(\circ)$ uses $\eta_t\ge T^{-3/4}$ for $t\le T$, and
$(\star)$ uses Assumption~\ref{ass:1}.

\textbf{Case 1: $\gamma_t=0$.} Here $\beta_t=1-h_t$ and
$\mathcal E_t=A_t^{(1)}$. Substituting the accumulated tracking-error bound gives
\[
\begin{aligned}
\sum_{t=1}^T\Gamma_t
&=\frac1L\sum_{t=1}^T h_tA_t^{(1)}+\mathcal R_{1,\rho}\sum_{t=1}^T h_t^2\\
&\stackrel{(\circ)}{\le}\frac{2\sigma^2}{L}
+\left(\frac{2\mathcal V_{1,\rho}}L+\mathcal R_{1,\rho}\right)(1+\ln T).
\end{aligned}
\]
Here $(\circ)$ follows from Eq.~\eqref{eq:practical-mvr1-ema-sum}
and $\sum_{t=1}^Tt^{-1}\le1+\ln T$.

\textbf{Case 2: $\gamma_t=h_t$.} Here $\beta_t=1-h_{t+1}$ and
$\mathcal E_t=A_t^{(2)}$. The corresponding tracking-error bound instead gives
\[
\begin{aligned}
\sum_{t=1}^T\Gamma_t
&=\frac1L\sum_{t=1}^T h_tA_t^{(2)}+\mathcal R_{1,\rho}\sum_{t=1}^T h_t^2\\
&\stackrel{(\circ)}{\le}\frac{8\sigma^2}{L}
+\left(\frac{4\mathcal V_{1,\rho}}L+\mathcal R_{1,\rho}\right)(1+\ln T),
\end{aligned}
\]
where $(\circ)$ uses Eq.~\eqref{eq:practical-mvr1-buffer-sum}
and the same harmonic-sum bound.

For $(c,r)=(2/3,5/4)$, expanding $\mathcal V_{1,\rho}$ and
$\mathcal R_{1,\rho}$ gives the four constants
$\mathcal A_{i,\rho}^{(j)}$ in the theorem.
Substitution into Eq.~\eqref{eq:practical-mvr1-telescoping} and division
by $\kappa_\rho T^{1/4}$ prove Eq.~\eqref{eq:practical-mvr1-main-bound}.

For the exact map with $a=1$ and $\lambda_t=0$, take
$(c,r)=(1,1)$ and $\rho=0$. Then $\kappa_0=1$, $d_0=\sqrt n$,
$\mathcal V_{1,0}=2\sqrt2L^2n+\sigma^2$, and
$\mathcal R_{1,0}=L/2+Ln$. Writing the numerator as
$f(\mathbf X_1)-f^*+A_1\ln T+A_2$, the two cases give
\[
\begin{aligned}
\gamma_t=0:\quad
&A_1=4\sqrt2Ln+Ln+2L^{-1}\sigma^2+L/2,
\qquad A_2=A_1+2L^{-1}\sigma^2,\\
\gamma_t=h_t:\quad
&A_1=8\sqrt2Ln+Ln+4L^{-1}\sigma^2+L/2,
\qquad A_2=A_1+8L^{-1}\sigma^2.
\end{aligned}
\]
\end{proof}

\section{Lemmas for Theorem \ref{th-nonconvex-mvr2}}
\label{appendix:practical-mvr-tracking}

Set $\mathbf S_t:=\mathbf M_t-\nabla f(\mathbf X_t)$ and
$A_t:=\mathbb E\|\mathbf S_t\|_F^2$, and retain $d_\rho$ from
Eq.~\eqref{eq:practical-kappa-d}.

\begin{lemma}
\label{lem:practical-mvr-tracking}
\label{lemma:b2}
Under Assumptions~\ref{ass:2.2} and~\ref{ass:3}, the MVR2 estimator with
$\gamma_t=1$ and a deterministic sequence $0\le\beta_t\le1$ satisfies
\begin{equation}
A_{t+1}\le\beta_{t+1}A_t
+2\beta_{t+1}^2L^2\mathbb E\|\mathbf X_{t+1}-\mathbf X_t\|_F^2
+2(1-\beta_{t+1})^2\sigma^2.
\label{eq:practical-tracking-general}
\end{equation}
If $\eta_t=t^{-2/3}$, $\beta_t=1-\eta_t$, and
Lemma~\ref{lem:practical-wd-control} applies with $p=2/3$, then
\begin{equation}
A_{t+1}\le(1-\eta_{t+1})A_t+2\mathcal V_\rho\eta_{t+1}^2,
\qquad \mathcal V_\rho:=4L^2d_\rho^2+\sigma^2.
\label{eq:practical-tracking-wd-ns}
\end{equation}
\end{lemma}

\begin{proof}
Let $\mathcal F_t:=\sigma(\xi_1,\ldots,\xi_t)$.
The matrices $\mathbf M_t$, $\mathbf X_t$, and $\mathbf X_{t+1}$ are
$\mathcal F_t$-measurable, while $\xi_{t+1}$ is independent of
$\mathcal F_t$. Define
$\mathbf\Delta_t:=\nabla f(\mathbf X_{t+1};\xi_{t+1})
-\nabla f(\mathbf X_t;\xi_{t+1})$,
$\overline{\mathbf\Delta}_t:=\mathbb E[\mathbf\Delta_t\mid\mathcal F_t]
=\nabla f(\mathbf X_{t+1})-\nabla f(\mathbf X_t)$, and
$\mathbf R_{t+1}:=\nabla f(\mathbf X_{t+1};\xi_{t+1})-\nabla f(\mathbf X_{t+1})$.
For $\gamma_{t+1}=1$, Eq.~\eqref{o3} gives
\begin{equation}
\mathbf S_{t+1}=\beta_{t+1}\mathbf S_t
+(1-\beta_{t+1})\mathbf R_{t+1}
+\beta_{t+1}(\mathbf\Delta_t-\overline{\mathbf\Delta}_t).
\label{eq:practical-tracking-decomposition}
\end{equation}
Taking the squared norm of this decomposition and conditioning on $\mathcal F_t$, we obtain
\[
\begin{aligned}
\mathbb E[\|\mathbf S_{t+1}\|_F^2\mid\mathcal F_t]
&\stackrel{(\circ)}{=}\beta_{t+1}^2\|\mathbf S_t\|_F^2
+\mathbb E[\|(1-\beta_{t+1})\mathbf R_{t+1}
+\beta_{t+1}(\mathbf\Delta_t-\overline{\mathbf\Delta}_t)\|_F^2\mid\mathcal F_t]\\
&\le\beta_{t+1}^2\|\mathbf S_t\|_F^2
+2(1-\beta_{t+1})^2\mathbb E[\|\mathbf R_{t+1}\|_F^2\mid\mathcal F_t]\\
&\quad+2\beta_{t+1}^2\mathbb E[\|\mathbf\Delta_t-\overline{\mathbf\Delta}_t\|_F^2\mid\mathcal F_t]\\
&\stackrel{(\star)}{\le}\beta_{t+1}\|\mathbf S_t\|_F^2
+2(1-\beta_{t+1})^2\sigma^2
+2\beta_{t+1}^2L^2\|\mathbf X_{t+1}-\mathbf X_t\|_F^2,
\end{aligned}
\]
where $(\circ)$ uses the conditional mean-zero property of the two noise
terms. For $(\star)$, Assumptions~\ref{ass:2.2} and~\ref{ass:3} give
$\mathbb E[\|\mathbf R_{t+1}\|_F^2\mid\mathcal F_t]\le\sigma^2$ and
$\mathbb E[\|\mathbf\Delta_t-\overline{\mathbf\Delta}_t\|_F^2\mid\mathcal F_t]
\le\mathbb E[\|\mathbf\Delta_t\|_F^2\mid\mathcal F_t]
\le L^2\|\mathbf X_{t+1}-\mathbf X_t\|_F^2$;
also, $\beta_{t+1}^2\le\beta_{t+1}$.
Taking total expectation proves Eq.~\eqref{eq:practical-tracking-general}.
For the stated schedules, Eq.~\eqref{eq:practical-wd-step-control} and
$\eta_t^2/\eta_{t+1}^2=((t+1)/t)^{4/3}\le2^{4/3}<4$ give
\[
\begin{aligned}
A_{t+1}
&\le(1-\eta_{t+1})A_t+2L^2d_\rho^2\eta_t^2+2\sigma^2\eta_{t+1}^2\\
&\le(1-\eta_{t+1})A_t+2(4L^2d_\rho^2+\sigma^2)\eta_{t+1}^2.
\end{aligned}
\]
\end{proof}

\begin{lemma}
\label{lem:practical-tracking-sum}
Under the assumptions of Lemma~\ref{lem:practical-mvr-tracking}, with
$\eta_t=t^{-2/3}$, $\beta_t=1-\eta_t$, and the weight-decay schedule
for $p=2/3$, every $T\ge1$ satisfies
\begin{align}
\sum_{t=1}^T\sqrt{\eta_t}A_t
&\le4\sigma^2+8\mathcal V_\rho(1+\ln T),
\label{eq:practical-tracking-sum}\\
\frac1T\sum_{t=1}^T A_t
&\le\frac{4\sigma^2+8\mathcal V_\rho(1+\ln T)}{T^{2/3}}.
\label{eq:practical-tracking-average}
\end{align}
\end{lemma}

\begin{proof}
For nonnegative sequences satisfying
$A_{t+1}\le(1-\eta_{t+1})A_t+B_{t+1}$, we have
\begin{equation}
\begin{aligned}
\frac{A_t}{\sqrt{\eta_t}}-\frac{A_{t+1}}{\sqrt{\eta_{t+1}}}
+\frac{B_{t+1}}{\sqrt{\eta_{t+1}}}
&\ge\bigl(t^{1/3}-(t+1)^{1/3}+(t+1)^{-1/3}\bigr)A_t\\
&\stackrel{(\circ)}{\ge}\tfrac14t^{-1/3}A_t
=\tfrac14\sqrt{\eta_t}A_t,
\end{aligned}
\label{eq:practical-weighted-sequence}
\end{equation}
where $(\circ)$ follows from concavity of $x^{1/3}$:
$t^{1/3}-(t+1)^{1/3}\ge-\tfrac13t^{-2/3}$ and
$(t+1)^{-1/3}\ge2^{-1/3}t^{-1/3}$, with $2^{-1/3}-1/3>1/4$.
Applying this inequality to Eq.~\eqref{eq:practical-tracking-wd-ns}
and summing from $t=1$ to $T$, we obtain
\[
\begin{aligned}
\sum_{t=1}^T\sqrt{\eta_t}A_t
&\le4\sum_{t=1}^T\left(\frac{A_t}{\sqrt{\eta_t}}
-\frac{A_{t+1}}{\sqrt{\eta_{t+1}}}
+2\mathcal V_\rho\eta_{t+1}^{3/2}\right)\\
&\le4A_1+8\mathcal V_\rho\sum_{t=1}^T\frac1{t+1}
\stackrel{(\circ)}{\le}4\sigma^2+8\mathcal V_\rho(1+\ln T),
\end{aligned}
\]
where $(\circ)$ uses $\mathbf M_1=\nabla f(\mathbf X_1;\xi_1)$,
Assumption~\ref{ass:3}, and the harmonic-sum bound.
To obtain the unweighted average, use $\sqrt{\eta_t}\ge T^{-1/3}$ for
$t\le T$, which gives
\[
\begin{aligned}
\frac1T\sum_{t=1}^T A_t
&\le\frac1{T^{2/3}}\sum_{t=1}^T\sqrt{\eta_t}A_t.
\end{aligned}
\]
Substitution of the preceding sum bound proves Eq.~\eqref{eq:practical-tracking-average}.
\end{proof}

\section{Proof of Theorem \ref{th-nonconvex-mvr2}}
\label{proof:th-nonconvex-mvr2}

\begin{proof}
With $\eta_t=t^{-2/3}$, $L$-smoothness and the update give
\begin{equation}
\begin{aligned}
f(\mathbf X_{t+1})
&\le f(\mathbf X_t)-a\eta_t\langle\nabla f(\mathbf X_t),\mathbf O_t\rangle_F
-\lambda_t\eta_t\langle\nabla f(\mathbf X_t),\mathbf X_t\rangle_F
+\frac{L\eta_t^2}{2}\|a\mathbf O_t+\lambda_t\mathbf X_t\|_F^2\\
&\stackrel{(\circ)}{\le}f(\mathbf X_t)-ac\eta_t\|\mathbf M_t\|_F
+ar\sqrt n\,\eta_t\|\mathbf S_t\|_F
+\rho\eta_t\bigl(\|\mathbf M_t\|_F+\|\mathbf S_t\|_F\bigr)
+\frac{Ld_\rho^2}{2}\eta_t^2\\
&=f(\mathbf X_t)-\kappa_\rho\eta_t\|\mathbf M_t\|_F
+d_\rho\eta_t\|\mathbf S_t\|_F+\frac{Ld_\rho^2}{2}\eta_t^2\\
&\stackrel{(\star)}{\le}f(\mathbf X_t)-\kappa_\rho\eta_t\|\mathbf M_t\|_F
+\frac{\sqrt{\eta_t}}{2L}\|\mathbf S_t\|_F^2
+\frac{Ld_\rho^2}{2}(\eta_t^{3/2}+\eta_t^2),
\end{aligned}
\label{eq:practical-descent-core}
\end{equation}
where $(\circ)$ uses $\nabla f(\mathbf X_t)=\mathbf M_t-\mathbf S_t$,
Eq.~\eqref{eq:practical-kappa-d}, and Lemma~\ref{lem:practical-wd-control};
$(\star)$ applies Young's inequality.
Taking expectation and summing from $t=1$ to $T$ makes the function-value
differences telescope, giving
\begin{equation}
\begin{aligned}
\kappa_\rho\sum_{t=1}^T\eta_t\mathbb E\|\mathbf M_t\|_F
&\le f(\mathbf X_1)-f^*+\frac1{2L}\sum_{t=1}^T\sqrt{\eta_t}A_t
+\frac{Ld_\rho^2}{2}\sum_{t=1}^T(\eta_t^{3/2}+\eta_t^2)\\
&\stackrel{(\circ)}{\le}f(\mathbf X_1)-f^*+\frac{2\sigma^2}L+2Ld_\rho^2
+\left(\frac{4\mathcal V_\rho}L+\frac{Ld_\rho^2}2\right)(1+\ln T)
=:\mathcal G_T,
\end{aligned}
\label{eq:practical-momentum-sum}
\end{equation}
where $(\circ)$ uses Eq.~\eqref{eq:practical-tracking-sum},
$\sum_{t=1}^Tt^{-1}\le1+\ln T$, and $\sum_{t=1}^Tt^{-4/3}\le4$.
It remains to pass from the momentum bound to the gradient bound.
The triangle inequality and the tracking estimate give
\[
\begin{aligned}
\frac1T\sum_{t=1}^T\mathbb E\|\nabla f(\mathbf X_t)\|_F
&\le\frac1T\sum_{t=1}^T\mathbb E\|\mathbf M_t\|_F
+\frac1T\sum_{t=1}^T\mathbb E\|\mathbf S_t\|_F\\
&\stackrel{(\circ)}{\le}\frac{\mathcal G_T}{\kappa_\rho T^{1/3}}
+\sqrt{\frac1T\sum_{t=1}^T A_t}\\
&\stackrel{(\star)}{\le}\frac1{T^{1/3}}\left[
\frac{\mathcal G_T}{\kappa_\rho}
+\sqrt{4\sigma^2+8\mathcal V_\rho(1+\ln T)}\right],
\end{aligned}
\]
where $(\circ)$ uses $\eta_t\ge T^{-2/3}$, Eq.~\eqref{eq:practical-momentum-sum},
and Jensen's inequality; $(\star)$ uses Eq.~\eqref{eq:practical-tracking-average}.
Taking $(c,r)=(2/3,5/4)$ proves the theorem. For the exact map with
$a=1$ and $\lambda_t=0$, taking $(c,r)=(1,1)$ and $\rho=0$ gives
$\kappa_0=1$, $d_0=\sqrt n$, and $\mathcal V_0=4L^2n+\sigma^2$
in the same bound.
\end{proof}

\paragraph{Nuclear-norm stationarity for the actual NS map.}
For a fixed finite $K\ge1$, take $(c,r)=(2/3,5/4)$.
Under the same assumptions and weight-decay schedule,
Eq.~\eqref{eq:practical-ns-nuclear-strong} gives
\begin{equation}
\frac1T\sum_{t=1}^T\mathbb E\|\nabla f(\mathbf X_t)\|_*
\le\frac1{T^{1/3}}\left[
\frac{(2/3)\mathcal G_T}{c_{K,n}\kappa_\rho}
+\sqrt{n\bigl(4\sigma^2+8\mathcal V_\rho(1+\ln T)\bigr)}\right].
\label{eq:practical-nuclear-stationarity}
\end{equation}
To prove this bound, set $Z_t:=\langle\mathbf M_t,\mathbf O_t\rangle_F$.
Combining the two alignment bounds gives
\[
\begin{aligned}
aZ_t-\rho\|\mathbf M_t\|_F
&\stackrel{(\circ)}{\ge}(a-3\rho/2)Z_t
\stackrel{(\star)}{\ge}\tfrac32\kappa_\rho c_{K,n}\|\mathbf M_t\|_*,
\end{aligned}
\]
where $(\circ)$ uses Eq.~\eqref{eq:practical-ns-alignment}, and
$(\star)$ uses Eq.~\eqref{eq:practical-ns-nuclear-strong} and
$a-3\rho/2=3\kappa_\rho/2>0$.
Retaining this alignment term in Eq.~\eqref{eq:practical-descent-core},
taking expectation, and summing as in Eq.~\eqref{eq:practical-momentum-sum} yield
\[
\begin{aligned}
\tfrac32\kappa_\rho c_{K,n}\sum_{t=1}^T\eta_t\mathbb E\|\mathbf M_t\|_*
&\le\mathcal G_T.
\end{aligned}
\]
Finally, $\|\mathbf S_t\|_*\le\sqrt n\|\mathbf S_t\|_F$ and
Eq.~\eqref{eq:practical-tracking-average} prove the stated bound.

\endgroup

\section{Lemmas for Theorems \ref{th-last-mvr1} and \ref{th-last-mvr2}}
\label{app:pl-corrected}

In this section, we establish the auxiliary contraction, unrolling, tail-summation, and weighted tracking-error lemmas used in the proofs of Theorems \ref{th-last-mvr1} and \ref{th-last-mvr2}. In addition to the PL condition, we assume a uniform bound on the gradient norms along the generated iterates. For a terminal time $T \ge 2$ and a sequence $\{a_t\}_{t \ge 1} \subset (0,1)$, define
$
R_t^{(T)} := \prod_{i=t+1}^{T-1}(1-a_i), 0 \le t \le T-1,
$
with the convention that an empty product equals $1$.
In particular, $R_{T-1}^{(T)} = 1$ and $R_0^{(T)} = \prod_{i=1}^{T-1}(1-a_i)$.

For notational clarity, we add the following notation. 

For both variants, write $d_\rho:=\tfrac54a\sqrt n+\rho$.
The schedules in Theorems~\ref{th-last-mvr1} and~\ref{th-last-mvr2} give
$\|\mathbf X_{t+1}-\mathbf X_t\|_F\le d_\rho\eta_t$,
as shown in Appendices~\ref{proof:th-last-mvr1} and~\ref{proof:th-last-mvr2}.
(i) For MVR1, we distinguish the two schemes.

\textbf{Case 1 ($\gamma_t =0$).} Define
\(A_t:=\mathbb E\|\mathbf M_t-\nabla f(\mathbf X_t)\|_F^2\),
\(\epsilon_t:=t^{-2q/3}\),
\(\eta_t=\eta t^{-q}=\eta\epsilon_t^{3/2}\), and
\(\beta_t=1-\epsilon_t\).
By Eq.~\eqref{eq:practical-mvr1-ema-general}, using this displacement bound,
\[
A_{t+1} \le \beta_{t+1}A_t
+\frac{\beta_{t+1}^2}{1-\beta_{t+1}}L^2d_\rho^2\eta_t^2
+(1-\beta_{t+1})^2\sigma^2
\le (1-\epsilon_{t+1})A_t +L^2d_\rho^2\frac{\eta_t^2}{\epsilon_{t+1}} +\sigma^2\epsilon_{t+1}^2.
\]
Moreover,
\[
\begin{aligned}
\frac{\eta_t^2}{\epsilon_{t+1}}
&= \eta^2 t^{-2q}(t+1)^{2q/3}
= \eta^2\Bigl(\frac{t+1}{t}\Bigr)^{2q}(t+1)^{-4q/3}
\le 2^{2q}\eta^2\epsilon_{t+1}^2,\\
A_{t+1}&\le (1-\epsilon_{t+1})A_t+C_\sigma\epsilon_{t+1}^2.
\end{aligned}
\]
Here \(C_\sigma:=2^{2q}L^2d_\rho^2\eta^2+\sigma^2\).
In particular, for $q=3/4$, \(C_\sigma=2\sqrt2\,L^2d_\rho^2\eta^2+\sigma^2\).
\textbf{Case 2 ($\gamma_t=\epsilon_t$).} Set
\(\mathbf C_1=\mathbf M_1=\nabla f(\mathbf X_1;\xi_1)\). For $t\ge2$, write MVR1 in the equivalent form
\(\mathbf C_t=\beta_{t-1}\mathbf C_{t-1}+(1-\beta_{t-1})\nabla f(\mathbf X_t;\xi_t)\) and
\(\mathbf M_t=\beta_t\mathbf C_t+(1-\beta_t)\nabla f(\mathbf X_t;\xi_t)\).
Define
\(A_t:=\mathbb E\|\mathbf C_t-\nabla f(\mathbf X_t)\|_F^2\) and
\(A_t':=\mathbb E\|\mathbf M_t-\nabla f(\mathbf X_t)\|_F^2\),
and set \(\epsilon_t:=t^{-2q/3}\),
\(\eta_t=\eta t^{-q}=\eta\epsilon_t^{3/2}\), and
\(\beta_t=1-\epsilon_{t+1}\).
Applying Eq.~\eqref{eq:practical-mvr1-ema-general} to the $\mathbf C_t$-recursion gives
\[
\begin{aligned}
A_{t+1}&\le \beta_tA_t
+\frac{\beta_t^2}{1-\beta_t}L^2d_\rho^2\eta_t^2
+(1-\beta_t)^2\sigma^2\\
&\le (1-\epsilon_{t+1})A_t +L^2d_\rho^2\frac{\eta_t^2}{\epsilon_{t+1}}+\sigma^2\epsilon_{t+1}^2
\le (1-\epsilon_{t+1})A_t+C_\sigma\epsilon_{t+1}^2.
\end{aligned}
\]
Moreover, \(A_t' \le 2A_t+2\sigma^2\epsilon_{t+1}^2\).
For $q=3/4$, this becomes
\(\beta_t=1-(t+1)^{-1/2}\) and \(\gamma_t=t^{-1/2}\).
\textit{(ii)} For MVR2, define
\(A_t:=\mathbb E\|\mathbf M_t-\nabla f(\mathbf X_t)\|_F^2\),
\(\epsilon_t:=t^{-q}\),
\(\eta_t=\eta t^{-q}=\eta\epsilon_t\), and
\(\beta_t=1-\epsilon_t\).
By Eq.~\eqref{eq:practical-tracking-general},
\(A_{t+1} \le \beta_{t+1}A_t +2\beta_{t+1}^2L^2\mathbb E\|\mathbf X_{t+1}-\mathbf X_t\|_F^2 +2(1-\beta_{t+1})^2\sigma^2\).
Using $\|\mathbf X_{t+1}-\mathbf X_t\|_F\le d_\rho\eta_t$,
\[
\begin{aligned}
A_{t+1} &\le (1-\epsilon_{t+1})A_t +2L^2d_\rho^2\eta_t^2 +2\sigma^2\epsilon_{t+1}^2,\\
\eta_t^2 &= \eta^2 t^{-2q} = \eta^2\Bigl(\frac{t+1}{t}\Bigr)^{2q}(t+1)^{-2q} \le 2^{2q}\eta^2\epsilon_{t+1}^2,\\
A_{t+1} &\le (1-\epsilon_{t+1})A_t +\bigl(2^{1+2q}L^2d_\rho^2\eta^2+2\sigma^2\bigr)\epsilon_{t+1}^2\\
&= (1-\epsilon_{t+1})A_t +\left(2^{1+2q}L^2d_\rho^2+\frac{2\sigma^2}{\eta^2}\right)\eta_{t+1}^2.
\end{aligned}
\]
For $q=2/3$, we have
\(A_{t+1} \le (1-\epsilon_{t+1})A_t +C_{\sigma}\eta_{t+1}^2\), where
\(C_{\sigma} = 8L^2d_\rho^2+\frac{2\sigma^2}{\eta^2}\).

\subsection{Lemma \ref{lemma:pl_linearize}}
\begin{lemma}
\label{lemma:pl_linearize}
Assume that $f$ is $\mu$-PL, i.e., $\|\nabla f(\mathbf{X})\|_F^2 \ge 2\mu \bigl(f(\mathbf{X}) - f^*\bigr)$ for all $\mathbf{X}$, and that $\|\nabla f(\mathbf{X}_t)\|_F \le G$ almost surely for every $t$.
Let $\Delta_t := \mathbb{E}[f(\mathbf{X}_t)] - f^*$ and fix $\kappa_\rho>0$.
If, for every $t \ge 1$,
$
\Delta_{t+1} \le \Delta_t - \kappa_\rho\eta_t\mathbb{E}\|\nabla f(\mathbf{X}_t)\|_F + \Gamma_t,
$
then
\[
\Delta_{t+1} \le (1-\kappa \eta_t)\Delta_t + \Gamma_t.
\]
Here \(\kappa := \frac{2\mu\kappa_\rho}{G}\).
\end{lemma}

\begin{proof}
Set $Y_t := \|\nabla f(\mathbf{X}_t)\|_F \ge 0$.
Since $Y_t \le G$ almost surely, we have $Y_t^2 \le GY_t$ almost surely, hence
$\mathbb{E}Y_t \ge G^{-1}\mathbb{E}Y_t^2$.
By the PL inequality,
$Y_t^2 \ge 2\mu \bigl(f(\mathbf{X}_t)-f^*\bigr)$ pointwise, so
$\mathbb{E}Y_t^2 \ge 2\mu \Delta_t$.
Therefore,
\(\mathbb{E}\|\nabla f(\mathbf{X}_t)\|_F = \mathbb{E}Y_t
\ge \frac{1}{G}\mathbb{E}Y_t^2
\ge \frac{2\mu}{G}\Delta_t\).
Substituting this bound into the assumed recursion gives
$\Delta_{t+1} \le (1-\kappa\eta_t)\Delta_t + \Gamma_t$.
\end{proof}

\subsection{Lemma \ref{lemma:unroll}}
\begin{lemma}
\label{lemma:unroll}
Assume that a nonnegative sequence $\{\Delta_t\}_{t \ge 1}$ satisfies $\Delta_{t+1} \le (1-a_t)\Delta_t + \Gamma_t,
t \ge 1,$
where $a_t \in (0,1)$.
Then, for every $T \ge 2$, \(\Delta_T \le R_0^{(T)}\Delta_1 + \sum_{t=1}^{T-1}\Gamma_t R_t^{(T)}\).
\end{lemma}

\begin{proof}
The claim follows by iterating the one-step recursion:
\[
\Delta_T
\le \Delta_1 \prod_{i=1}^{T-1}(1-a_i) + \sum_{t=1}^{T-1}\Gamma_t \prod_{i=t+1}^{T-1}(1-a_i)
= R_0^{(T)}\Delta_1 + \sum_{t=1}^{T-1}\Gamma_t R_t^{(T)}.
\]
This is an extended version of the inequality in Lemma 22 from \cite{he2023convergence}.
\end{proof}

\subsection{Lemma \ref{lem:coeff-lb}}
\begin{lemma}\label{lem:coeff-lb}
Let $q\in(0,1)$, $\epsilon_t:=t^{-q}$, and $\alpha:=q/2\in(0,1/2)$. Then, for every $t\ge 1$,
\[
\frac{1}{\sqrt{\epsilon_t}} - \frac{1-\epsilon_{t+1}}{\sqrt{\epsilon_{t+1}}}
= t^\alpha-(t+1)^\alpha+(t+1)^{-\alpha} \ge \frac14 \sqrt{\epsilon_t}.
\]
\end{lemma}

\begin{proof}
Set $c_t:=t^\alpha-(t+1)^\alpha+(t+1)^{-\alpha}.$
If $t=1$, then $c_1=1-2^\alpha+2^{-\alpha}
\ge 1-\sqrt{2}+\frac{1}{\sqrt{2}}
>\frac14.$

If $t\ge2$, concavity of $s\mapsto s^\alpha$ gives $(t+1)^\alpha-t^\alpha \le \alpha t^{\alpha-1},$ hence
$c_t
\ge
(t+1)^{-\alpha}-\alpha t^{\alpha-1} = t^{-\alpha}\left[\left(\frac{t}{t+1}\right)^\alpha-\alpha t^{2\alpha-1}\right].$ Since $t/(t+1)\ge 2/3$, $\alpha<1/2$, and $2\alpha-1<0$,
$
\left(\frac{t}{t+1}\right)^\alpha-\alpha t^{2\alpha-1} \ge \left(\frac23\right)^\alpha-\alpha
\ge \sqrt{\frac23}-\frac12 > \frac14.
$
Therefore $c_t\ge \frac14 t^{-\alpha}=\frac14\sqrt{\epsilon_t}$.
\end{proof}

\subsection{Lemma \ref{lemma:tail-unified}}
\begin{lemma}
\label{lemma:tail-unified}
Let $\eta_t = \eta t^{-q}$ with $q \in (0,1)$, let $a_t = \kappa \eta_t$, and assume $\eta \le \min\{1,\frac{1}{2\kappa}\}$.
Then, for every $p \in (0,1)$, with
\[
C_p := \frac{1}{\kappa} \left(2^{p+1} + \left(\frac{4qp}{e(1-q)\kappa\eta}\right)^{\frac{qp}{1-q}}\right),
\]
we have for all $T \ge 2$,
\begin{equation}
\sum_{t=1}^{T-1}\eta_t^{1+p} R_t^{(T)} \le C_p \eta_T^p.
\end{equation}
The same bound also holds with $\eta_{t+1}^{1+p}$ in place of $\eta_t^{1+p}$.
\end{lemma}


\begin{proof}
Since $\eta_t$ is decreasing and $\kappa \eta_1 = \kappa\eta \le 1/2$, we have
$a_t \le 1/2$ for all $t$, and therefore
\begin{equation}
R_t^{(T)}
\le \exp\Bigl(-\sum_{i=t+1}^{T-1}a_i\Bigr)
= \exp\Bigl(-\kappa\sum_{i=t+1}^{T-1}\eta_i\Bigr).
\label{eq:tail-rt-exp}
\end{equation}
Moreover, the telescoping identity
$a_t R_t^{(T)} = R_t^{(T)} - R_{t-1}^{(T)}$ for $1 \le t \le T-1$
is the discrete summation-by-parts step used below. Hence
\begin{align}
\sum_{t=1}^{T-1}\eta_t^{1+p}R_t^{(T)}
&= \frac{1}{\kappa}\sum_{t=1}^{T-1}\eta_t^p \bigl(a_t R_t^{(T)}\bigr) \nonumber =
\frac{1}{\kappa}\sum_{t=1}^{T-1}\eta_t^p \bigl(R_t^{(T)} - R_{t-1}^{(T)}\bigr) \nonumber \\
&\le \frac{1}{\kappa}
\Bigl(\eta_{T-1}^p + \sum_{t=1}^{T-2}\bigl(\eta_t^p - \eta_{t+1}^p\bigr)R_t^{(T)}\Bigr).
\label{eq:tail-sbp}
\end{align}
Let $m := \lceil T/2\rceil$.
We split the sum in \eqref{eq:tail-sbp} into the ranges $t \le m-1$ and $t \ge m$.

For the second half, since $R_t^{(T)} \le 1$,
\[
\sum_{t=m}^{T-2}\bigl(\eta_t^p - \eta_{t+1}^p\bigr)R_t^{(T)}
\le \sum_{t=m}^{T-2}\bigl(\eta_t^p - \eta_{t+1}^p\bigr)
= \eta_m^p - \eta_{T-1}^p
\le \eta_m^p \le 2^p \eta_T^p.
\]
Also, $\eta_{T-1}^p \le 2^p \eta_T^p$.

For the first half, $R_t^{(T)}$ is increasing in $t$, so for $t \le m-1$,
$R_t^{(T)} \le R_{m-1}^{(T)}$.
Therefore,
\[
\sum_{t=1}^{m-1}\bigl(\eta_t^p - \eta_{t+1}^p\bigr)R_t^{(T)}
\le R_{m-1}^{(T)} \sum_{t=1}^{m-1}\bigl(\eta_t^p - \eta_{t+1}^p\bigr)
\le \eta_1^p R_{m-1}^{(T)}.
\]
By \eqref{eq:tail-rt-exp}, we have 
$
R_{m-1}^{(T)} \le \exp\Bigl(-\kappa \sum_{i=m}^{T-1}\eta_i\Bigr).
$
Since $i \le T$ for $m \le i \le T-1$ and $T-m \ge T/4$, we have
\(\sum_{i=m}^{T-1}\eta_i = \eta \sum_{i=m}^{T-1} i^{-q}
\ge \eta (T-m)T^{-q}
\ge \frac{\eta}{4} T^{1-q}\).
Thus $R_{m-1}^{(T)} \le \exp(-c_0 T^{1-q})$ with $c_0 := \kappa\eta/4$.
With $\alpha:=qp/(1-q)$, for every $x>0$,
\(x^\alpha e^{-c_0x}\le \left(\frac{pq}{ec_0(1-q)}\right)^\alpha\),
taking $x=T^{1-q}$ gives
\(T^{qp}e^{-c_0T^{1-q}}
\le \left(\frac{qp}{e(1-q)c_0}\right)^{\frac{qp}{1-q}}
= \left(\frac{4qp}{e(1-q)\kappa\eta}\right)^{\frac{qp}{1-q}}\).
Hence
\(\eta_1^p R_{m-1}^{(T)}
\le\left(\frac{4qp}{e(1-q)\kappa\eta}\right)^{\frac{qp}{1-q}}\eta_T^p\).
Substituting the above estimates into \eqref{eq:tail-sbp} yields
\[
\sum_{t=1}^{T-1}\eta_t^{1+p} R_t^{(T)}
\le\frac{1}{\kappa}
\left(2^{p+1}+\left(\frac{4qp}{e(1-q)\kappa\eta}\right)^{\frac{qp}{1-q}}\right)\eta_T^p
=C_p\eta_T^p.
\]
Finally, since $\eta_{t+1} \le \eta_t$, the same estimate holds with $\eta_{t+1}^{1+p}$ in place of $\eta_t^{1+p}$.
\end{proof}

\subsection{Lemma \ref{lemma:weighted_A}}
\begin{lemma}
\label{lemma:weighted_A}
Let $q \in (0,1)$, let $\epsilon_t := t^{-q}$, let $\eta_t := \eta \epsilon_t = \eta t^{-q}$, and let $a_t := \kappa \eta_t$.
Assume $\eta \le \min\{1,\frac{1}{8\kappa}\}$.
Suppose that a nonnegative sequence $\{A_t\}_{t \ge 1}$ satisfies, for all
\(t\ge1\),
\begin{equation}
A_{t+1} \le (1-\epsilon_{t+1})A_t + C_\sigma \eta_{t+1}^2,
\label{eq:mvr2-tracker-recursion}
\end{equation}
for some constant $C_\sigma > 0$.
Then, for every $T \ge 2$,
$
\sum_{t=1}^{T-1}\sqrt{\eta_t}A_t R_t^{(T)}
\le 8\sqrt{\eta}A_1 R_1^{(T)} + 8\eta C_\sigma \sum_{t=1}^{T-1}\eta_{t+1}^{3/2} R_t^{(T)}.
$ Consequently,
\[
\sum_{t=1}^{T-1}\sqrt{\eta_t}A_t R_t^{(T)} \le 8\sqrt{\eta}A_1 R_1^{(T)} + 8\eta C_pC_{\sigma} \eta_T^{1/2}
\]
where $C_p$ follows from Lemma~\ref{lemma:tail-unified} with $p=1/2$.
\end{lemma}

\begin{proof}
Define $U_t := A_t/\sqrt{\epsilon_t}$.
By \eqref{eq:mvr2-tracker-recursion},
\(U_{t+1} \le \frac{1-\epsilon_{t+1}}{\sqrt{\epsilon_{t+1}}} A_t + C_\sigma \frac{\eta_{t+1}^2}{\sqrt{\epsilon_{t+1}}}\).
Hence
\begin{equation}
U_t - U_{t+1} \ge \Bigl(\frac{1}{\sqrt{\epsilon_t}} -\frac{1-\epsilon_{t+1}}{\sqrt{\epsilon_{t+1}}}\Bigr) A_t - C_\sigma \frac{\eta_{t+1}^2}{\sqrt{\epsilon_{t+1}}}.
\label{eq:mvr2-ut-diff}
\end{equation}
By Lemma~\ref{lem:coeff-lb} and \eqref{eq:mvr2-ut-diff},
\(\sqrt{\epsilon_t}A_t \le 4(U_t-U_{t+1}) + 4C_\sigma \frac{\eta_{t+1}^2}{\sqrt{\epsilon_{t+1}}}\).
Since $\sqrt{\eta_t} = \sqrt{\eta}\sqrt{\epsilon_t}$, after multiplying by $\sqrt{\eta}R_t^{(T)}$ and summing from $t=1$ to $T-1$, we obtain
\begin{align}
\sum_{t=1}^{T-1}\sqrt{\eta_t}A_t R_t^{(T)}
&\le 4\sqrt{\eta}\sum_{t=1}^{T-1}(U_t-U_{t+1})R_t^{(T)} + 4\eta C_\sigma \sum_{t=1}^{T-1}\eta_{t+1}^{3/2}R_t^{(T)}.
\label{eq:mvr2-sum-step1}
\end{align}
For the first term, using $R_t^{(T)} - R_{t-1}^{(T)} = a_t R_t^{(T)}$, we have
\begin{align}
\sum_{t=1}^{T-1}(U_t-U_{t+1})R_t^{(T)}
&= U_1 R_1^{(T)} + \sum_{t=2}^{T-1}U_t\bigl(R_t^{(T)}-R_{t-1}^{(T)}\bigr) - U_T R_{T-1}^{(T)} \nonumber \\
&\le U_1 R_1^{(T)} + \sum_{t=2}^{T-1}U_t a_t R_t^{(T)}.
\end{align}
Moreover,
\(\sqrt{\eta}U_t a_t = \sqrt{\eta}\frac{A_t}{\sqrt{\epsilon_t}}\cdot \kappa\eta\epsilon_t = \kappa\eta \sqrt{\eta_t}A_t\).
Substituting this into \eqref{eq:mvr2-sum-step1} gives
\[
\sum_{t=1}^{T-1}\sqrt{\eta_t}A_t R_t^{(T)}
\le 4\sqrt{\eta}U_1 R_1^{(T)} + 4\kappa\eta \sum_{t=2}^{T-1}\sqrt{\eta_t}A_t R_t^{(T)} + 4\eta C_\sigma \sum_{t=1}^{T-1}\eta_{t+1}^{3/2}R_t^{(T)}.
\]
Since $\eta \le 1/(8\kappa)$, we have $4\kappa\eta \le 1/2$.
Rearranging and using $\epsilon_1 = 1$, hence $U_1 = A_1$, we obtain
\[
\sum_{t=1}^{T-1}\sqrt{\eta_t}A_t R_t^{(T)}
\le 8\sqrt{\eta}A_1 R_1^{(T)} + 8\eta C_\sigma \sum_{t=1}^{T-1}\eta_{t+1}^{3/2}R_t^{(T)}.
\]
Finally, we have
\(\sum_{t=1}^{T-1}\sqrt{\eta_t}A_t R_t^{(T)}
\le 8\sqrt{\eta}A_1 R_1^{(T)} + 8\eta C_pC_{\sigma} \eta_T^{1/2}\),
where $C_p$ follows from Lemma~\ref{lemma:tail-unified} with $p=1/2$.
\end{proof}

\subsection{Lemma \ref{lemma:weighted_A_MVR1}}
\begin{lemma}
\label{lemma:weighted_A_MVR1}
Let $q \in (0,1)$, let $\epsilon_t := t^{-2q/3}$, let $\eta_t := \eta t^{-q} = \eta \epsilon_t^{3/2}$, and let $a_t := \kappa \eta_t$.
Assume $\eta \le \min\{1,\frac{1}{4\kappa}\}$.
Suppose that a nonnegative sequence $\{A_t\}_{t \ge 1}$ satisfies, for all
\(t\ge1\),
\begin{equation}
A_{t+1} \le (1-\epsilon_{t+1})A_t + C_\sigma \epsilon_{t+1}^2,
\label{eq:mvr1-tracker-recursion}
\end{equation}
for some constant $C_\sigma > 0$.
Then, for every $T \ge 2$,
$
\sum_{t=1}^{T-1}\eta_t^{2/3} A_t R_t^{(T)}
\le 4\eta^{2/3}A_1 R_1^{(T)} + 4\eta^{-2/3} C_\sigma \sum_{t=1}^{T-1}\eta_{t+1}^{4/3} R_t^{(T)}.
$
Consequently,
\[
\sum_{t=1}^{T-1}\eta_t^{2/3} A_t R_t^{(T)}
\le 4\eta^{2/3}A_1 R_1^{(T)} + 4\eta^{-2/3}C_pC_\sigma\eta_T^{1/3}
\]
where $C_p$ follows from Lemma~\ref{lemma:tail-unified} with $p=1/3$.
\end{lemma}

\begin{proof}
Since $\epsilon_t/\epsilon_{t+1} = ((t+1)/t)^{2q/3} \le 2^{2q/3} \le 2$, the recursion \eqref{eq:mvr1-tracker-recursion} implies
\begin{equation}
\epsilon_t A_t \le 2\epsilon_{t+1}A_t \le 2\bigl(A_t-A_{t+1}+C_\sigma \epsilon_{t+1}^2\bigr).
\label{eq:mvr1-basic}
\end{equation}
Because $\eta_t^{2/3} = \eta^{2/3}\epsilon_t$, multiplying \eqref{eq:mvr1-basic} by $\eta^{2/3}R_t^{(T)}$ and summing from $t=1$ to $T-1$ yields
\begin{align}
\sum_{t=1}^{T-1}\eta_t^{2/3} A_t R_t^{(T)}
&\le 2\eta^{2/3}\sum_{t=1}^{T-1}(A_t-A_{t+1})R_t^{(T)} + 2\eta^{2/3}C_\sigma \sum_{t=1}^{T-1}\epsilon_{t+1}^2 R_t^{(T)}.
\label{eq:mvr1-sum-step1}
\end{align}
As before,
\(\sum_{t=1}^{T-1}(A_t-A_{t+1})R_t^{(T)}
\le A_1 R_1^{(T)} + \sum_{t=2}^{T-1}A_t a_t R_t^{(T)}\).
Moreover, using $\epsilon_t^{3/2} \le \epsilon_t$,
\(\eta^{2/3} A_t a_t = \kappa \eta^{2/3}\eta_t A_t = \kappa\eta^{5/3}\epsilon_t^{3/2}A_t
\le \kappa\eta^{5/3}\epsilon_t A_t = \kappa\eta \eta_t^{2/3} A_t\).
Substituting this into \eqref{eq:mvr1-sum-step1}, and using $\eta \le 1/(4\kappa)$, gives
\begin{align}
\sum_{t=1}^{T-1}\eta_t^{2/3} A_t R_t^{(T)}
&\le 2\eta^{2/3}A_1 R_1^{(T)} + 2\kappa\eta \sum_{t=2}^{T-1}\eta_t^{2/3} A_t R_t^{(T)} + 2\eta^{2/3}C_\sigma \sum_{t=1}^{T-1}\epsilon_{t+1}^2 R_t^{(T)} \nonumber \\
&\le 2\eta^{2/3}A_1 R_1^{(T)} + \frac{1}{2}\sum_{t=1}^{T-1}\eta_t^{2/3} A_t R_t^{(T)} + 2\eta^{2/3}C_\sigma \sum_{t=1}^{T-1}\epsilon_{t+1}^2 R_t^{(T)}.
\end{align}
Rearranging yields
\(\sum_{t=1}^{T-1}\eta_t^{2/3} A_t R_t^{(T)}
\le 4\eta^{2/3}A_1 R_1^{(T)} + 4\eta^{2/3}C_\sigma \sum_{t=1}^{T-1}\epsilon_{t+1}^2 R_t^{(T)}\).
Finally, since $\eta_{t+1}^{4/3} = \eta^{4/3}\epsilon_{t+1}^2$, we obtain
\[
\begin{aligned}
\sum_{t=1}^{T-1}\eta_t^{2/3} A_t R_t^{(T)}
&\le 4\eta^{2/3}A_1 R_1^{(T)} + 4\eta^{-2/3}C_\sigma \sum_{t=1}^{T-1}\eta_{t+1}^{4/3} R_t^{(T)}
\\
&\le 4\eta^{2/3}A_1 R_1^{(T)} + 4\eta^{-2/3}C_pC_\sigma\eta_T^{1/3}
\end{aligned}
\]
where $C_p$ from Lemma~\ref{lemma:tail-unified} with $p=1/3$.
\end{proof}

\section{Proof of Theorem \ref{th-last-mvr1}}
\label{proof:th-last-mvr1}

\begin{proof}
Write $d_\rho:=\tfrac54a\sqrt n+\rho$ and
$\mathcal R_\rho:=\frac L2[a^2(\tfrac49+\tfrac{25}{16}n)+d_\rho^2]$.
The weight-decay schedule and $\eta\le1$ give $0\le\eta_t\lambda_t<1$.
Iterating the norm bound for the update, we obtain
\[
\begin{aligned}
\|\mathbf X_t\|_F
&\stackrel{(\circ)}{\le}\|\mathbf X_1\|_F
+\tfrac54a\sqrt n\sum_{s=1}^{t-1}\eta_s
\stackrel{(\star)}{\le}\|\mathbf X_1\|_F+5a\sqrt n\,t^{1/4},
\end{aligned}
\]
where $(\circ)$ uses Eq.~\eqref{eq:practical-model-update} and the NS norm
bound in Lemma~\ref{lem:practical-ns-geometry}, and $(\star)$ uses
$\eta\le1$ and $\sum_{s=1}^{t-1}s^{-3/4}\le4t^{1/4}$.
Multiplying by the prescribed upper bound on $\lambda_t$ gives
$\lambda_t\|\mathbf X_t\|_F\le\rho$. The update therefore satisfies
$\|\mathbf X_{t+1}-\mathbf X_t\|_F\le d_\rho\eta_t$.

Set $\epsilon_t:=t^{-1/2}$ and
$C_\sigma:=2\sqrt2L^2d_\rho^2\eta^2+\sigma^2$, so that
$\eta_t=\eta\epsilon_t^{3/2}$.
For $\gamma_t=0$, let
$A_t=B_t:=\mathbb E\|\mathbf M_t-\nabla f(\mathbf X_t)\|_F^2$.
For $\gamma_t=t^{-1/2}$, use the buffer $\mathbf C_t$ in
Eq.~\eqref{eq:practical-mvr1-buffer} and let
$A_t:=\mathbb E\|\mathbf C_t-\nabla f(\mathbf X_t)\|_F^2$ and
$B_t:=A_t':=\mathbb E\|\mathbf M_t-\nabla f(\mathbf X_t)\|_F^2$.
The tracking estimates in Appendix~\ref{app:pl-corrected} give
$A_{t+1}\le(1-\epsilon_{t+1})A_t+C_\sigma\epsilon_{t+1}^2$
and $A_1\le\sigma^2$ in both cases. In the second case, they also give
$B_t=A_t'\le2A_t+2\sigma^2\epsilon_{t+1}^2$.

Let $a_t:=\kappa\eta_t$ and
$\Gamma_t:=\eta_t^{2/3}B_t/L+\mathcal R_\rho\eta_t^{4/3}$.
Applying the descent and PL estimates yields
\begin{equation}
\begin{aligned}
\Delta_{t+1}
&\stackrel{(\circ)}{\le}\Delta_t-(2a/3-\rho)\eta_t
\mathbb E\|\nabla f(\mathbf X_t)\|_F
+\frac{\eta_t^{2/3}}L B_t+\mathcal R_\rho\eta_t^{4/3}\\
&\stackrel{(\star)}{\le}(1-a_t)\Delta_t+\Gamma_t,
\end{aligned}
\label{eq:mvr1-descent}
\end{equation}
where $(\circ)$ follows from Lemma~\ref{lem:practical-mvr1-descent}
with the finite-NS bounds, and $(\star)$ uses
$\mathbb E\|\nabla f(\mathbf X_t)\|_F\ge(2\mu/G)\Delta_t$
from Lemma~\ref{lemma:pl_linearize} and
$\kappa=(2\mu/G)(2a/3-\rho)$.
Since $0<a_t\le1/4$, Lemma~\ref{lemma:unroll} now gives
\begin{equation}
\begin{aligned}
\Delta_T
&\le R_0^{(T)}\Delta_1+\sum_{t=1}^{T-1}\Gamma_tR_t^{(T)}\\
&=R_0^{(T)}\Delta_1+\frac1L\sum_{t=1}^{T-1}\eta_t^{2/3}B_tR_t^{(T)}
+\mathcal R_\rho\sum_{t=1}^{T-1}\eta_t^{4/3}R_t^{(T)}.
\end{aligned}
\label{eq:mvr1-unroll-explicit}
\end{equation}

To bound the products in this expression, recall $c=4\kappa\eta$.
For $j\in\{0,1\}$, we have
\begin{equation}
\begin{aligned}
R_j^{(T)}
&\stackrel{(\circ)}{\le}
\exp\left(-\kappa\eta\sum_{s=j+1}^{T-1}s^{-3/4}\right)
\stackrel{(\star)}{\le}
\exp\left(-4\kappa\eta\bigl(T^{1/4}-(j+1)^{1/4}\bigr)\right)\\
&=e^{(j+1)^{1/4}c}e^{-cT^{1/4}},
\end{aligned}
\label{eq:mvr1-product-bound}
\end{equation}
where $(\circ)$ uses $1-a_s\le e^{-a_s}$, and $(\star)$ uses
$\sum_{s=j+1}^{T-1}s^{-3/4}\ge\int_{j+1}^{T}s^{-3/4}\,ds$.
For the last sum in Eq.~\eqref{eq:mvr1-unroll-explicit},
Lemma~\ref{lemma:tail-unified} with $q=3/4$ and $p=1/3$ gives
$\sum_{t=1}^{T-1}\eta_t^{4/3}R_t^{(T)}\le C_{1/3}\eta_T^{1/3}$,
where $C_{1/3}=\kappa^{-1}(2^{4/3}+4/(e\kappa\eta))$.
The same estimate holds with $\eta_{t+1}^{4/3}$ in place of
$\eta_t^{4/3}$. It remains to substitute the tracking-error bound for
each MVR1 parameterization.

\textbf{Case 1: $\gamma_t=0$.} Here $B_t=A_t$.
Since $\eta_T^{1/3}=\eta^{1/3}T^{-1/4}$, the coefficient of the
polynomial term can be written as
\[
\begin{aligned}
\eta^{1/3}\left(\frac{4C_\sigma}{L\eta^{2/3}}+\mathcal R_\rho\right)
&\stackrel{(\circ)}{=}8\sqrt2Ld_\rho^2\eta^{5/3}
+4L^{-1}\sigma^2\eta^{-1/3}+\mathcal R_\rho\eta^{1/3}
=C_1^{(2)},
\end{aligned}
\]
where $(\circ)$ substitutes $C_\sigma=2\sqrt2L^2d_\rho^2\eta^2+\sigma^2$;
the last equality uses the definitions of $d_\rho$ and
$\mathcal R_\rho$. Substituting the weighted tracking-error estimate
into Eq.~\eqref{eq:mvr1-unroll-explicit} therefore gives
\[
\begin{aligned}
\Delta_T
&\stackrel{(\circ)}{\le}R_0^{(T)}\Delta_1
+\frac{4\eta^{2/3}\sigma^2}{L}R_1^{(T)}
+\left(\frac{4C_\sigma}{L\eta^{2/3}}+\mathcal R_\rho\right)
 C_{1/3}\eta_T^{1/3}\\
&\stackrel{(\star)}{\le}
\left(e^c\Delta_1+\frac{4\eta^{2/3}\sigma^2}{L}e^{2^{1/4}c}\right)
e^{-cT^{1/4}}+C_1^{(2)}C_{1/3}T^{-1/4}\\
&=C_1^{(1)}e^{-cT^{1/4}}+C_1^{(2)}C_{1/3}T^{-1/4},
\end{aligned}
\]
where $(\circ)$ uses Lemma~\ref{lemma:weighted_A_MVR1}, $A_1\le\sigma^2$,
and the tail estimates above, and $(\star)$ uses
Eq.~\eqref{eq:mvr1-product-bound} and the expression for $C_1^{(2)}$.
This proves the bound for $i=1$.

\textbf{Case 2: $\gamma_t=t^{-1/2}$.} Here
$B_t=A_t'\le2A_t+2\sigma^2\epsilon_{t+1}^2$.
Since $\epsilon_t\le1$ and
$\eta_{t+1}^{4/3}=\eta^{4/3}\epsilon_{t+1}^2$, we have
$\eta_t^{2/3}\epsilon_{t+1}^2\le\eta^{-2/3}\eta_{t+1}^{4/3}$.
Thus the weighted error sum satisfies
\[
\begin{aligned}
\sum_{t=1}^{T-1}\eta_t^{2/3}B_tR_t^{(T)}
&\stackrel{(\circ)}{\le}2\sum_{t=1}^{T-1}\eta_t^{2/3}A_tR_t^{(T)}
+2\sigma^2\eta^{-2/3}\sum_{t=1}^{T-1}\eta_{t+1}^{4/3}R_t^{(T)}\\
&\stackrel{(\star)}{\le}8\eta^{2/3}\sigma^2R_1^{(T)}
+(8C_\sigma+2\sigma^2)\eta^{-2/3}C_{1/3}\eta_T^{1/3},
\end{aligned}
\]
where $(\circ)$ uses the preceding two inequalities, and $(\star)$ uses
Lemma~\ref{lemma:weighted_A_MVR1}, $A_1\le\sigma^2$, and the tail estimate
with $\eta_{t+1}^{4/3}$. Expanding its contribution to the polynomial
term, we obtain
\[
\begin{aligned}
\eta^{1/3}\left(\frac{8C_\sigma+2\sigma^2}{L\eta^{2/3}}+\mathcal R_\rho\right)
&\stackrel{(\circ)}{=}16\sqrt2Ld_\rho^2\eta^{5/3}
+10L^{-1}\sigma^2\eta^{-1/3}+\mathcal R_\rho\eta^{1/3}
=C_2^{(2)},
\end{aligned}
\]
where $(\circ)$ uses
$8C_\sigma+2\sigma^2=16\sqrt2L^2d_\rho^2\eta^2+10\sigma^2$.
Substitution into Eq.~\eqref{eq:mvr1-unroll-explicit} now yields
\[
\begin{aligned}
\Delta_T
&\le R_0^{(T)}\Delta_1+\frac{8\eta^{2/3}\sigma^2}{L}R_1^{(T)}
+\left(\frac{8C_\sigma+2\sigma^2}{L\eta^{2/3}}+\mathcal R_\rho\right)
 C_{1/3}\eta_T^{1/3}\\
&\stackrel{(\circ)}{\le}
\left(e^c\Delta_1+\frac{8\eta^{2/3}\sigma^2}{L}e^{2^{1/4}c}\right)
e^{-cT^{1/4}}+C_2^{(2)}C_{1/3}T^{-1/4}\\
&=C_2^{(1)}e^{-cT^{1/4}}+C_2^{(2)}C_{1/3}T^{-1/4},
\end{aligned}
\]
where $(\circ)$ uses Eq.~\eqref{eq:mvr1-product-bound},
$\eta_T^{1/3}=\eta^{1/3}T^{-1/4}$, and the expression for $C_2^{(2)}$.
This proves the bound for $i=2$.

For exact matrix-sign normalization with $a=1$ and $\lambda_t=0$,
the same argument applies with $\kappa=2\mu/G$, $d_\rho^2=n$, and
$\mathcal R_\rho=L/2+Ln$. The definitions of $c$, $C_{1/3}$, and
$C_i^{(1)}$ retain their form. Substituting these values gives
\[
\begin{aligned}
C_1^{(2)}&=8\sqrt2Ln\eta^{5/3}+4L^{-1}\sigma^2\eta^{-1/3}
+(L/2+Ln)\eta^{1/3},\\
C_2^{(2)}&=16\sqrt2Ln\eta^{5/3}+10L^{-1}\sigma^2\eta^{-1/3}
+(L/2+Ln)\eta^{1/3}.
\end{aligned}
\]
\end{proof}

\section{Proof of Theorem \ref{th-last-mvr2}}
\label{proof:th-last-mvr2}

\begin{proof}
Write $d_\rho:=\tfrac54a\sqrt n+\rho$ and
$\mathcal R_\rho:=\frac L2[a^2(\tfrac49+\tfrac{25}{16}n)+d_\rho^2]$.
As in the preceding proof, $0\le\eta_t\lambda_t<1$, and iterating the
norm bound for the update gives
\[
\begin{aligned}
\|\mathbf X_t\|_F
&\stackrel{(\circ)}{\le}\|\mathbf X_1\|_F
+\tfrac54a\sqrt n\sum_{s=1}^{t-1}\eta_s
\stackrel{(\star)}{\le}\|\mathbf X_1\|_F+\tfrac{15}{4}a\sqrt n\,t^{1/3},
\end{aligned}
\]
where $(\circ)$ uses Eq.~\eqref{eq:practical-model-update} and the NS norm
bound, and $(\star)$ uses $\eta\le1$ and
$\sum_{s=1}^{t-1}s^{-2/3}\le3t^{1/3}$.
The prescribed bound on $\lambda_t$ therefore gives
$\lambda_t\|\mathbf X_t\|_F\le\rho$, and hence
$\|\mathbf X_{t+1}-\mathbf X_t\|_F\le d_\rho\eta_t$.

Set $\epsilon_t:=t^{-2/3}$,
$\mathbf S_t:=\mathbf M_t-\nabla f(\mathbf X_t)$,
$A_t:=\mathbb E\|\mathbf S_t\|_F^2$, and
$C_\sigma:=8L^2d_\rho^2+2\sigma^2\eta^{-2}$.
Then $\eta_t=\eta\epsilon_t$, and the tracking estimate in
Appendix~\ref{app:pl-corrected} gives
$A_{t+1}\le(1-\epsilon_{t+1})A_t+C_\sigma\eta_{t+1}^2$
with $A_1\le\sigma^2$.
Let $a_t:=\kappa\eta_t$. The descent estimate follows by applying
smoothness to the update:
\begin{equation}
\begin{aligned}
\Delta_{t+1}
&\stackrel{(\circ)}{\le}\Delta_t-(2a/3-\rho)\eta_t
\mathbb E\|\nabla f(\mathbf X_t)\|_F
+a(\tfrac23+\tfrac54\sqrt n)\eta_t\mathbb E\|\mathbf S_t\|_F
+\frac{Ld_\rho^2}{2}\eta_t^2\\
&\stackrel{(\star)}{\le}\Delta_t-(2a/3-\rho)\eta_t
\mathbb E\|\nabla f(\mathbf X_t)\|_F
+\frac{\sqrt{\eta_t}}L A_t+\mathcal R_\rho\eta_t^{3/2}\\
&\stackrel{(\bullet)}{\le}(1-a_t)\Delta_t
+\frac{\sqrt{\eta_t}}L A_t+\mathcal R_\rho\eta_t^{3/2},
\end{aligned}
\label{eq:mvr2-descent}
\end{equation}
where $(\circ)$ uses the NS alignment and norm bounds,
$\lambda_t\|\mathbf X_t\|_F\le\rho$, and the triangle inequality,
as in Eq.~\eqref{eq:practical-mvr1-descent-start};
$(\star)$ applies Young's inequality separately to the two terms in
$a(\tfrac23+\tfrac54\sqrt n)\eta_t\mathbb E\|\mathbf S_t\|_F$
and uses $\eta_t^2\le\eta_t^{3/2}$; and $(\bullet)$ uses
Lemma~\ref{lemma:pl_linearize} with $\kappa_\rho=2a/3-\rho$.
Since $0<a_t\le1/8$, unrolling this recursion yields
\begin{equation}
\begin{aligned}
\Delta_T
&\le R_0^{(T)}\Delta_1
+\frac1L\sum_{t=1}^{T-1}\sqrt{\eta_t}A_tR_t^{(T)}
+\mathcal R_\rho\sum_{t=1}^{T-1}\eta_t^{3/2}R_t^{(T)}.
\end{aligned}
\label{eq:mvr2-unroll}
\end{equation}

We first estimate the products using $c=3\kappa\eta$.
For $j\in\{0,1\}$, the same integral comparison gives
\begin{equation}
\begin{aligned}
R_j^{(T)}
&\stackrel{(\circ)}{\le}
\exp\left(-\kappa\eta\sum_{s=j+1}^{T-1}s^{-2/3}\right)
\stackrel{(\star)}{\le}
\exp\left(-3\kappa\eta\bigl(T^{1/3}-(j+1)^{1/3}\bigr)\right)\\
&=e^{(j+1)^{1/3}c}e^{-cT^{1/3}},
\end{aligned}
\label{eq:mvr2-product-bound}
\end{equation}
where $(\circ)$ uses $1-a_s\le e^{-a_s}$, and $(\star)$ uses
$\sum_{s=j+1}^{T-1}s^{-2/3}\ge\int_{j+1}^{T}s^{-2/3}\,ds$.
For the two sums in Eq.~\eqref{eq:mvr2-unroll}, applying the weighted
tracking-error and tail estimates gives
\[
\begin{aligned}
\sum_{t=1}^{T-1}\sqrt{\eta_t}A_tR_t^{(T)}
&\stackrel{(\circ)}{\le}8\sqrt\eta\,\sigma^2R_1^{(T)}
+8\eta C_\sigma\sum_{t=1}^{T-1}\eta_{t+1}^{3/2}R_t^{(T)}\\
&\stackrel{(\star)}{\le}8\sqrt\eta\,\sigma^2R_1^{(T)}
+8\eta C_\sigma C_{1/2}\eta_T^{1/2},\\
\sum_{t=1}^{T-1}\eta_t^{3/2}R_t^{(T)}
&\stackrel{(\star)}{\le}C_{1/2}\eta_T^{1/2},
\end{aligned}
\]
where $(\circ)$ uses Lemma~\ref{lemma:weighted_A} and $A_1\le\sigma^2$,
and $(\star)$ uses Lemma~\ref{lemma:tail-unified} with $q=2/3$ and $p=1/2$.
Here $C_{1/2}=\kappa^{-1}(2^{3/2}+4/(e\kappa\eta))$ and
$\eta_T^{1/2}=\eta^{1/2}T^{-1/3}$.
Thus, expanding the coefficient of the polynomial term gives
\[
\begin{aligned}
\eta^{1/2}\left(\frac{8\eta C_\sigma}{L}+\mathcal R_\rho\right)
&\stackrel{(\circ)}{=}\eta^{1/2}
\left(64Ld_\rho^2\eta+16L^{-1}\sigma^2\eta^{-1}+\mathcal R_\rho\right)\\
&=64Ld_\rho^2\eta^{3/2}+16L^{-1}\sigma^2\eta^{-1/2}
+\mathcal R_\rho\eta^{1/2}=C_3^{(2)},
\end{aligned}
\]
where $(\circ)$ substitutes $C_\sigma=8L^2d_\rho^2+2\sigma^2\eta^{-2}$.
Combining these estimates in Eq.~\eqref{eq:mvr2-unroll}, we obtain
\[
\begin{aligned}
\Delta_T
&\le R_0^{(T)}\Delta_1+\frac{8\sqrt\eta\,\sigma^2}{L}R_1^{(T)}
+\left(\frac{8\eta C_\sigma}{L}+\mathcal R_\rho\right)
 C_{1/2}\eta_T^{1/2}\\
&\stackrel{(\circ)}{\le}
\left(e^c\Delta_1+\frac{8\sqrt\eta\,\sigma^2}{L}e^{2^{1/3}c}\right)
e^{-cT^{1/3}}+C_3^{(2)}C_{1/2}T^{-1/3}\\
&=C_3^{(1)}e^{-cT^{1/3}}+C_3^{(2)}C_{1/2}T^{-1/3},
\end{aligned}
\]
where $(\circ)$ uses Eq.~\eqref{eq:mvr2-product-bound} and the expression
for $C_3^{(2)}$. This is the asserted bound.

For exact matrix-sign normalization with $a=1$ and $\lambda_t=0$,
the same argument uses $\kappa=2\mu/G$, $d_\rho^2=n$, and
$\mathcal R_\rho=L/2+Ln$. The definitions of $c$, $C_{1/2}$, and
$C_3^{(1)}$ retain their form, and the remaining coefficient is
$C_3^{(2)}=64Ln\eta^{3/2}+16L^{-1}\sigma^2\eta^{-1/2}
+(L/2+Ln)\eta^{1/2}$.
\par\nobreak\noindent
\end{proof}

\section{Experimental Details}
\label{appendix:exp}
The curve and hyperparameter-sweep experiments use 8 Ascend 910C (64GB) NPUs and 4 NVIDIA RTX 4090 (24GB) GPUs. The additional evaluations in Appendix~\ref{appendix:additional-evaluations} are reported separately with their configuration differences.
\paragraph{MVR1 parameterization.}
MVR1 treats the correction coefficient $\gamma$ as independently tunable: $\gamma=0$ recovers EMA, while $\gamma=1-\beta$ gives the Nesterov recurrence under matched initialization.
The selected MVR1 configurations in Tables~\ref{tab:resnet_hyperparameters} and~\ref{tab:llama_c4_hyperparameters} use $(\beta,\gamma)=(0.9,0.1)$ and $(0.95,0.05)$, respectively, and thus belong to this Nesterov-equivalent special case of the general parameterization.

\subsection{Training on CIFAR10}
The ResNet18 model~\cite{he2016deep} undergoes pretraining on the CIFAR-10 dataset~\cite{krizhevsky2009learning} with comprehensive hyperparameter specifications provided in Table~\ref{tab:resnet_hyperparameters}. For each optimizer, the learning rate is selected from
\[
\left\{10^{-4},,5\times10^{-4},,10^{-3},,5\times10^{-3},,10^{-2},,5\times10^{-2},,10^{-1}\right\}.
\]
To ensure a robust comparison, all experiments are repeated over five different random seeds, and we report the mean results with one standard deviation shaded. For the ResNet-18 model, we reshape each convolutional kernel into a 2D matrix and apply a Muon-type optimizer to these parameters, while the remaining 1D vector parameters are optimized with AdamW.

\begin{table}[h!]
\centering
\caption{Hyperparameters used for training ResNet18 on CIFAR10}
\begin{tabular}{lccccc}
\hline
 & SGD & Adam & Muon & Muon-MVR1 & Muon-MVR2 \\
\hline
Model Size & \multicolumn{5}{c}{42.7MB} \\
Training Epochs & \multicolumn{5}{c}{100} \\
Batch Size & \multicolumn{5}{c}{128} \\
Learning Rate & 0.1 & 0.01  & \multicolumn{3}{c}{0.05} \\
Learning Rate Scheduling & \multicolumn{5}{c}{cosine to 10\%} \\
Numerical precision & \multicolumn{5}{c}{float32} \\
Weight Decay & \multicolumn{5}{c}{0.01}\\
$(\beta_1,\beta_2)$ & \XSolidBrush & (0.9,0.999) & \multicolumn{3}{c}{\XSolidBrush} \\
Muon-Momentum &  \XSolidBrush & \XSolidBrush & \multicolumn{3}{c}{0.9}\\
Gamma & \XSolidBrush & \XSolidBrush &  \XSolidBrush & \multicolumn{2}{c}{0.1} \\
\hline
\end{tabular}

\label{tab:resnet_hyperparameters}
\end{table}

\subsection{Pretraining on C4}


\begin{table}[h!]
\centering
\caption{Hyperparameters used for training LLaMA2-130M on C4}
\begin{tabular}{lcccccc}
\hline
Hyper-parameter & AdamW &MARS-AdamW & Muon & Muon-MVR1 & Muon-MVR2 \\
\hline

Max Learning Rate & 8e-4 & 1e-3 & 8e-4 & 2e-3 & 2e-3 \\
Warmup Ratio  & \multicolumn{5}{c}{0.1}\\
Batch Size & \multicolumn{5}{c}{128} \\
Maximum Length & \multicolumn{5}{c}{4096} \\
Weight Decay  & \multicolumn{5}{c}{0.1} \\
$(\beta_1,\beta_2)$  & \multicolumn{5}{c}{(0.9,0.98)} \\
Muon-Momentum  & \XSolidBrush  & \XSolidBrush  & \multicolumn{3}{c}{0.95} \\
Gamma  & \XSolidBrush  & 0.025 & \XSolidBrush & \multicolumn{2}{c}{0.05} \\
\hline
\end{tabular}

\label{tab:llama_c4_hyperparameters}
\end{table}

$\blacktriangleright$ \textbf{Experimental setup.} We use 8 Ascend 910C (64GB) NPUs for the C4 curves and hyperparameter sweeps. We conduct these sweeps for LLaMA2-130M~\cite{Touvron2023Llama2O} trained for 12B tokens on the C4 (Colossal Clean Crawled Corpus) dataset \cite{raffel2020exploring}. For all optimizers (AdamW, MARS-AdamW, Muon, Muon-MVR1, and Muon-MVR2), we keep the model architecture and training data fixed. For LLaMA2-130M, we use a global batch size of 128 and a maximum sequence length of 4096. 
For each optimizer we train all configurations on the $4\times$ Chinchilla data, about 12B tokens, for 22,000 steps and select the hyperparameters that achieve the best validation performance.

$\blacktriangleright$ \textbf{Hyperparameter search.} Comprehensive experimental specifications are tabulated in Table \ref{tab:llama_c4_hyperparameters}. For AdamW, we set $(\beta_1, \beta_2) = (0.9, 0.98)$, $\epsilon = 10^{-8}$, and a weight decay of 0.1. The learning rate $\eta$ is selected from the set $\{3\mathrm{e}{-4}, 5\mathrm{e}{-4}, 8\mathrm{e}{-4}, 1\mathrm{e}{-3}, 2\mathrm{e}{-3}, 4\mathrm{e}{-3}, 6\mathrm{e}{-3}, 8\mathrm{e}{-3}\}$. For MARS-AdamW, we use the same $(\beta_1, \beta_2) = (0.9, 0.98)$, $\epsilon = 10^{-8}$, and weight decay of 0.1, and we search over the same learning-rate set for $\eta$. In addition, we sweep over the algorithmic parameter $\gamma \in \{0.01, 0.025, 0.05\}$.

For Muon, we set $\beta = 0.95$ and a weight decay of 0.1, and we again choose the learning rate $\eta$ from $\{3\mathrm{e}{-4}, 5\mathrm{e}{-4}, 8\mathrm{e}{-4}, 1\mathrm{e}{-3}, 2\mathrm{e}{-3}, 4\mathrm{e}{-3}, 6\mathrm{e}{-3}, 8\mathrm{e}{-3}\}$. Muon-MVR1 and Muon-MVR2 use the same settings $\beta = 0.95$ and weight decay of 0.1, and share the same learning-rate search space as Muon. For both Muon-MVR1 and Muon-MVR2, we additionally perform a sweep over $\gamma \in \{0.01, 0.025, 0.05\}$.
We use the Muon implementation from Moonlight\footnote{\url{https://github.com/MoonshotAI/Moonlight}}. For LLaMA model, we optimize all 2D matrix parameters (except the embedding layers) using a Muon-type optimizer, while the remaining 1D vector parameters (including the embedding layers) are optimized with AdamW. 

\subsection{Additional Evaluations}
\label{appendix:additional-evaluations}
\paragraph{Practical schedules and repeated seeds.}
Table~\ref{tab:additional-seed-results} reports a separate CIFAR-10 evaluation over five independent seeds with $\gamma=0.05$ for both MVR1 and MVR2; the configurations used for the curves in Figure~\ref{fig:training_dynamics} have $\gamma=0.1$ (Table~\ref{tab:resnet_hyperparameters}). The additional evaluation uses AdamW as its adaptive baseline. Muon (EMA) denotes the $\gamma=0$ configuration. For C4, Table~\ref{tab:additional-seed-results} uses three matched seeds on 32 Ascend 910C devices. Hyperparameters are fixed before these reruns, and all optimizers share the same initialization and data order within each seed. The correction coefficients are $\gamma=0.05$ for MVR1 and MVR2, $0.025$ for MARS-AdamW, and $0$ for Muon (EMA). The reported accuracy and perplexity summaries are means and standard deviations across seeds; epoch and step times are averages within each evaluation.

\paragraph{Scaled diminishing schedules on CIFAR-10.}
Table~\ref{tab:theory-schedule-cifar} evaluates the scaled diminishing schedules with stepsize prefactor $c=0.05$ under an equal budget of 78,199 stochastic-gradient evaluations. The stepsizes are $\eta_t=ct^{-3/4}$ for MVR1 and $\eta_t=ct^{-2/3}$ for MVR2; the latter uses $\beta_t=1-t^{-2/3}$ and $\gamma_t=1$. Each MVR2 step after initialization counts two evaluations. These runs examine the finite-budget behavior of the diminishing schedules on ResNet-18, separately from the cosine schedules above.
\begin{table}[!htbp]
\centering
\caption{CIFAR-10 with scaled diminishing schedules and equal gradient-evaluation budgets.}
\label{tab:theory-schedule-cifar}
\begin{tabular}{lcc}
\toprule
Optimizer & Test accuracy (\%) & Training time (min) \\
\midrule
Muon-MVR1 & $89.54\pm0.35$ & $30.44\pm0.75$ \\
Muon-MVR2 & $70.59\pm0.27$ & $25.97\pm2.02$ \\
\bottomrule
\end{tabular}
\end{table}

\paragraph{Memory and implementation cost.}
In a separate 32-device measurement with local batch size 4 and global batch size 128, MVR1 and MVR2 have peak allocated memory of 12,993.5 and 13,549.2 MB, respectively (an increase of 555.7 MB, or $4.3\%$). Peak reserved memory is 28,156 and 29,156 MB (an increase of 1,000 MB, or $3.6\%$), and resident allocated memory after each step is 1,068.7 and approximately 1,612 MB. MVR2 needs an additional forward/backward pass and a model-level buffer for the previous parameters or update. Performing the passes sequentially allows the first computation graph to be released before the second pass. In data-parallel implementations, the same-sample gradient terms can be combined locally before a single all-reduce; synchronizing each pass separately instead incurs two gradient synchronizations.

\subsection{NS-step and Weight-Decay Ablation}
\label{appendix:ns-wd-settings}
Figure~\ref{fig:ns-wd-ablation} uses the $12\times6$ stochastic matrix quadratic
\[
f(\mathbf{X};\xi)=\tfrac12\langle\mathbf{X}-\mathbf{X}^*,
\mathbf{H}_\xi\odot(\mathbf{X}-\mathbf{X}^*)\rangle_F
+\langle\mathbf{Z}_\xi,\mathbf{X}\rangle_F,
\]
where the entries of $\mathbf{H}_\xi$ are uniform on $[0.5,1.5]$ and
$\|\mathbf{Z}_\xi\|_F=0.15$. Muon-MVR2 is run for 1,500 iterations with
$\eta_t=t^{-2/3}$, $\beta_t=1-\eta_t$, $\gamma_t=1$, $a=0.2$, and
$\lambda_t=\lambda_0t^{-1/3}$, with $\lambda_0$ chosen from the budget
$\rho$ as in Eq.~\eqref{eq:practical-wd-schedule}. The NS map uses
the coefficients from Section~\ref{section:practical-muon-mvr} and exact
Frobenius normalization. Entries average 10 seeds.

\section{Future work}
\label{appendix:lim}
Future directions include conducting a large-scale, unified evaluation of Muon variants with thorough tuning and improving the Muon-MVR1 rate beyond $\widetilde{\mathcal{O}}(T^{-1/4})$. The PL last-iterate guarantees require a uniform bound on the gradient norms along the generated iterates. This assumption is used only for the non-ergodic PL analysis and may not hold automatically for all deep-learning objectives; relaxing it is an important direction for future work.
\bibliography{manuscript}

\end{document}